\documentclass{article}
\usepackage[top=1in, bottom=1in, left=1in, right=1in]{geometry}

\newcommand{\fairAUC}{fairAUC}
\newcommand{\maxAUC}{maxAUC}
\newcommand{\minBias}{minBias}
\newcommand{\random}{random}

\newcommand{\dm}{manager}
\newcommand{\dms}{managers}
\newcommand{\algo}{procedure}
\newcommand{\algos}{procedures}

\usepackage{fancyhdr,graphicx,amsmath,amssymb}
\usepackage[ruled,vlined]{algorithm2e}
\usepackage{amsthm}
\usepackage{mathtools}
\usepackage{bm}
\renewcommand{\epsilon}{\varepsilon}
\usepackage{tablefootnote}

\usepackage[hidelinks]{hyperref}
\usepackage{cleveref}
\usepackage{xcolor}
\usepackage{soul}
\usepackage{caption}
\usepackage{nicefrac}

\newtheorem{assumption}{Assumption}
\newtheorem{theorem}{Theorem}[section]
\newtheorem{lemma}[theorem]{Lemma}
\newtheorem{definition}[theorem]{Definition}

\newtheorem{corollary}[theorem]{Corollary}

\newtheorem{claim}[theorem]{Claim}
\newtheorem{observation}[theorem]{Observation}
\newtheorem{proposition}[theorem]{Proposition}

\newtheorem{remark}[theorem]{Remark}

\crefname{equation}{Equation}{Equations}
\crefname{figure}{Figure}{Figures}
\crefname{table}{Table}{Tables}
\crefname{section}{Section}{Sections}
\crefname{appendix}{Section}{Sections}
\crefname{algorithm}{Algorithm}{Algorithms}
\crefname{assumption}{Assumption}{Assumptions}
\crefname{theorem}{Theorem}{Theorems}
\crefname{lemma}{Lemma}{Lemmas}
\crefname{definition}{Definition}{Definitions}
\crefname{conjecture}{Conjecture}{Conjectures}
\crefname{corollary}{Corollary}{Corollaries}
\crefname{construction}{Construction}{Constructions}
\crefname{claim}{Claim}{Claims}
\crefname{observation}{Observation}{Observations}
\crefname{proposition}{Proposition}{Propositions}
\crefname{fact}{Fact}{Facts}
\crefname{question}{Question}{Questions}
\crefname{problem}{Problem}{Problems}
\crefname{remark}{Remark}{Remarks}
\crefname{example}{Example}{Examples}
\crefname{appendix}{Appendix}{Appendices}

\usepackage{mathrsfs}

\newcommand{\yesnum}{\addtocounter{equation}{1}\tag{\theequation}}
\newcommand{\tagnum}[1]{\addtocounter{equation}{1}{\tag{#1) \ (\theequation}}}
\makeatletter
\newcommand{\customlabel}[2]{%
\protected@write \@auxout {}{\string \newlabel {#1}{{#2}{\thepage}{#2}{#1}{}} }%
\hypertarget{#1}{}
}
\makeatother

\newcommand{\red}[1]{\textcolor{black}{#1}}

\newcommand{\N}{\mathbb{N}}

\newcommand{\R}{\mathbb{R}}

\newcommand{\cD}{\mathcal{D}}

\newcommand{\sfrac}[2]{#1/#2}

\renewcommand{\epsilon}{\varepsilon}

\newcommand{\Ex}{\operatornamewithlimits{\mathbb{E}}}

\newcommand{\argmax}{\operatornamewithlimits{argmax}}

\def\abs#1{\left| #1 \right|}
\def\sabs#1{| #1 |}

\newcommand{\sinparen}[1]{(#1)}
\newcommand{\sinbrace}[1]{\{#1\}}

\newcommand{\sinangle}[1]{\langle#1\rangle}
\newcommand{\inparen}[1]{\left(#1\right)}
\newcommand{\inbrace}[1]{\left\{#1\right\}}
\newcommand{\insquare}[1]{\left[#1\right]}
\newcommand{\inangle}[1]{\left\langle#1\right\rangle}

\newcommand{\Stackrel}[2]{\stackrel{\mathmakebox[\widthof{\ensuremath{#2}}]{#1}}{#2}}

\newcommand{\zo}{\{0,1\}}

\newcommand{\sexp}[1]{{{\hbox{\tiny$($}}#1{\hbox{\tiny$)$}}}}

\newcommand{\Var}{\operatorname{Var}}
\newcommand{\Cov}{\operatorname{Cov}}

\newcommand{\cN}{\mathcal{N}}

\newcommand{\AUC}{\ensuremath{\mathrm{AUC}}}
\newcommand{\ab}{\ensuremath{\inbrace{a,b}}}
\newcommand{\bias}{\mathsf{Bias}}

\usepackage{natbib}
\bibpunct[, ]{(}{)}{,}{a}{}{,}%

\title{\bf Fairness for AUC via Feature Augmentation}

\author{Hortense Fong \\ Yale University \and Vineet Kumar \\ Yale University \and Anay Mehrotra \\ Yale University \and Nisheeth K. Vishnoi \\ Yale University}

\begin{document}

\maketitle

\begin{abstract}
  We study fairness in the context of classification where the performance is measured by the area under the
  curve (AUC) of the receiver operating characteristic. AUC is commonly used to measure the performance
  of prediction models. The same classifier can have significantly varying AUCs for different protected groups
  and, in real-world applications, it is often desirable to reduce such cross-group differences. We address the
  problem of how to acquire additional features to most greatly improve AUC for the disadvantaged group.
  We develop a novel approach, fairAUC, based on feature augmentation (adding features) to mitigate bias
  between identifiable groups. The approach requires only a few summary statistics to offer provable guarantees on AUC improvement, and allows managers flexibility in determining where in the fairness-accuracy tradeoff they would like to be. We evaluate fairAUC
  on synthetic and real-world datasets and find that it significantly improves AUC for the disadvantaged group
  relative to benchmarks maximizing overall AUC and minimizing bias between groups.
\end{abstract}

\newpage
\setcounter{tocdepth}{2}
\tableofcontents
\addtocontents{toc}{\protect\setcounter{tocdepth}{2}}

\newpage

\section{Introduction}\label{section:introduction}
Algorithms are often the basis of many important decisions in today's business world and society. There are a wide range of applications, including hiring \citep{liem2018screening,de2019bias,lambrecht2019algorithmic},  mortgage lending \citep{fuster2020mortgage}, criminal justice  \citep{berk2018fairness}, and healthcare \citep{obermeyer2019dissecting}, in which algorithms are used to make predictions, which are then used to make decisions, either with or without human supervision. Many such algorithms have been found to be unfair or discriminatory on the basis of legally and socially salient characteristics like race, gender, and age. %

Given the importance of achieving fairness across individuals and groups, a wide range of fair algorithms have been proposed. Most of the fairness interventions assume that data is already collected and fixed, and focus on how to design algorithms that are fair. However, if the original data features are collected without recognizing fairness issues, focusing on only the algorithm might not be sufficient.
Consider a scenario in which features are selected to maximize accuracy in a population with two groups. Then, it is possible that the features are perfectly predictive for the majority group but entirely uninformative for the minority group.
In this case, an algorithmic solution cannot improve the classification accuracy of the minority group.
Indeed, a survey of industry practitioners finds that it is at the \textit{data collection} step that practitioners seek guidance \citep{holstein2019improving}.

This phenomenon of additional feature acquisition is an everyday business practice where firms obtain more information about their current or prospective customers, e.g. through a credit report. To quote American Express:\footnote{\url{https://www.americanexpress.com/en-us/credit-cards/credit-intel/hard-inquiry-vs-soft-inquiry/}}

\begin{quote}
  \textit{A major fact of life is that whoever you do business with will likely evaluate your financial status – and do it regularly. Whenever you apply for a loan or sign up for a new cell phone plan, lenders and businesses will make a ``hard” pull of your credit report to get a sense of how well or poorly you handle spending and debt payments. Businesses you already have an account with, or others that may want to offer you a ``preapproval'' deal, can also take a peek at your credit report. %
  }\end{quote}

  \noindent Motivated by this problem, we propose a procedure that uses \textit{feature acquisition or augmentation} (additional feature collection) to improve the predictive performance of  disadvantaged groups.\footnote{A distinct but related literature focuses on feature selection rather than acquisition. With feature selection, the manager already has the data corresponding to all the features and chooses a subset of those to use in an algorithm \citep{guyon2003introduction,li2017feature}, whereas with feature acquisition the manager is obtaining auxiliary or additional features from a third-party like a data vendor. The data vendor might reveal either a sample of the records with each of the additional features or provides the manager with summary statistics of the features.}
  By focusing on the disadvantaged group(s), our approach aims to reduce bias, characterized by the ratio of the area under the receiver operating characteristic curves (AUCs) between protected groups. Our approach, which we call \fairAUC, is applicable to a wide variety of classification algorithms and requires only a few data distribution moments of the additional (auxiliary) features. The method is flexible enough to allow decision-makers or managers to determine where in the fairness-accuracy tradeoff they would like to be.

  AUC is a non-parametric performance measure that has long been used in binary classification problems, across a wide range of fields, including diagnostic systems, medicine, and in machine learning \citep{thompson1989statistical,bertsimas2016analytics,ahsen2019cancer}.
  AUC is derived from the receiver operating characteristic (ROC) curve, which captures classifier performance in two dimensions by plotting the true positive rate against the false positive rate, by varying the classification threshold. Integrating the area under the ROC curve summarizes the true positive and false positive rates into a single metric, the AUC, which falls between zero and one. AUC is also related to the Mann-Whitney $U$-statistic, and represents the probability that a classifier will rank a randomly chosen positive instance higher than a randomly chosen negative instance. %

  When should a \dm\ use AUC as a model performance criterion? First, classification algorithms require the \dm\ to set a threshold on scores output by a model to separate the classes. AUC provides a threshold-invariant way to obtain model performance without human judgment regarding appropriate thresholds. AUC integrates across thresholds, and is especially useful in environments where there may be multiple \dms, who have different thresholds. Second, AUC is invariant to the base rate, or the proportion of individuals in each class. For data with significant class imbalance, an algorithm would achieve high accuracy by simply always predicting the majority class. However, AUC would not assign this algorithm a high performance measure because the algorithm fails to discriminate between the positive and negative classes. Unlike accuracy, F1, and the area under the precision-recall curve, AUC is robust to changes in the base rate,\footnote{When base rates vary, accuracy, F1, and the area under the precision-recall curve will change even if the fundamental characteristics of the classes remain the same (i.e., $\Pr_{\rm train}[X|Y]=\Pr_{\rm test}[X|Y]$ but $\Pr_{\rm train}[Y] \neq \Pr_{\rm test}[Y]$).} which may vary significantly over time and place \citep{fawcett2006introduction}. Third, AUC serves as a measure of rank-ordering, which is particularly useful when there are different intensities of intervention available \citep{kallus2019fairness}. For example, a radiologist may set different thresholds for different treatment recommendations based on the outcomes of some tests. Similarly, a bank may set different interest rates depending on credit score and other factors. Thus, the \dm\ would be interested in multiple thresholds, not just one, and AUC can provide an overall characterization across all such thresholds.

  \subsection*{Our Contribution}

  We propose the \fairAUC\ procedure based on feature augmentation to maximally increase the AUC of the disadvantaged group during each round of feature acquisition. It allows the \dm\ to identify new (costly) features to acquire. We provide theoretical guarantees for how \fairAUC\ improves the AUC of each of the groups in a round of feature acquisition. Our theoretical guarantees are directly applicable to the following commonly used models, all of which belong to the class of Generalized Linear Models (GLMs): logistic and multinomial regression, linear SVM, Poisson or negative binomial regression \citep{nelder1972generalized}. We evaluate the performance of \fairAUC\ alongside benchmark procedures using synthetic data as well as in real empirical contexts (COMPAS and Diabetes datasets). We find that \fairAUC\ achieves low bias between groups, while obtaining relatively high levels of AUC. Moreover, our approach permits flexibility in determining how many features to acquire, and suggests which ones, based on AUC, fairness, or a weighted combination.

  \subsubsection*{Feature Acquisition using fairAUC}
  We use a  binormal framework to characterize the distribution of a feature and show how Fisher's linear discriminant (FLD) can be used to acquire an additional feature to maximally increase the AUC of the disadvantaged group each feature acquisition round. FLD produces the linear projection which maximizes AUC within this framework \citep{su1993linear}. %
  While other papers have also suggested searching for additional features to increase fairness \citep{hardt2016equality,chen2018why}, we provide specific recommendations on \textit{which features to acquire}.
  The binormal framework provides us easily interpretable summary statistics and theoretical guarantees.

  \begin{figure}
    \begin{center}
      \includegraphics[height=11cm]{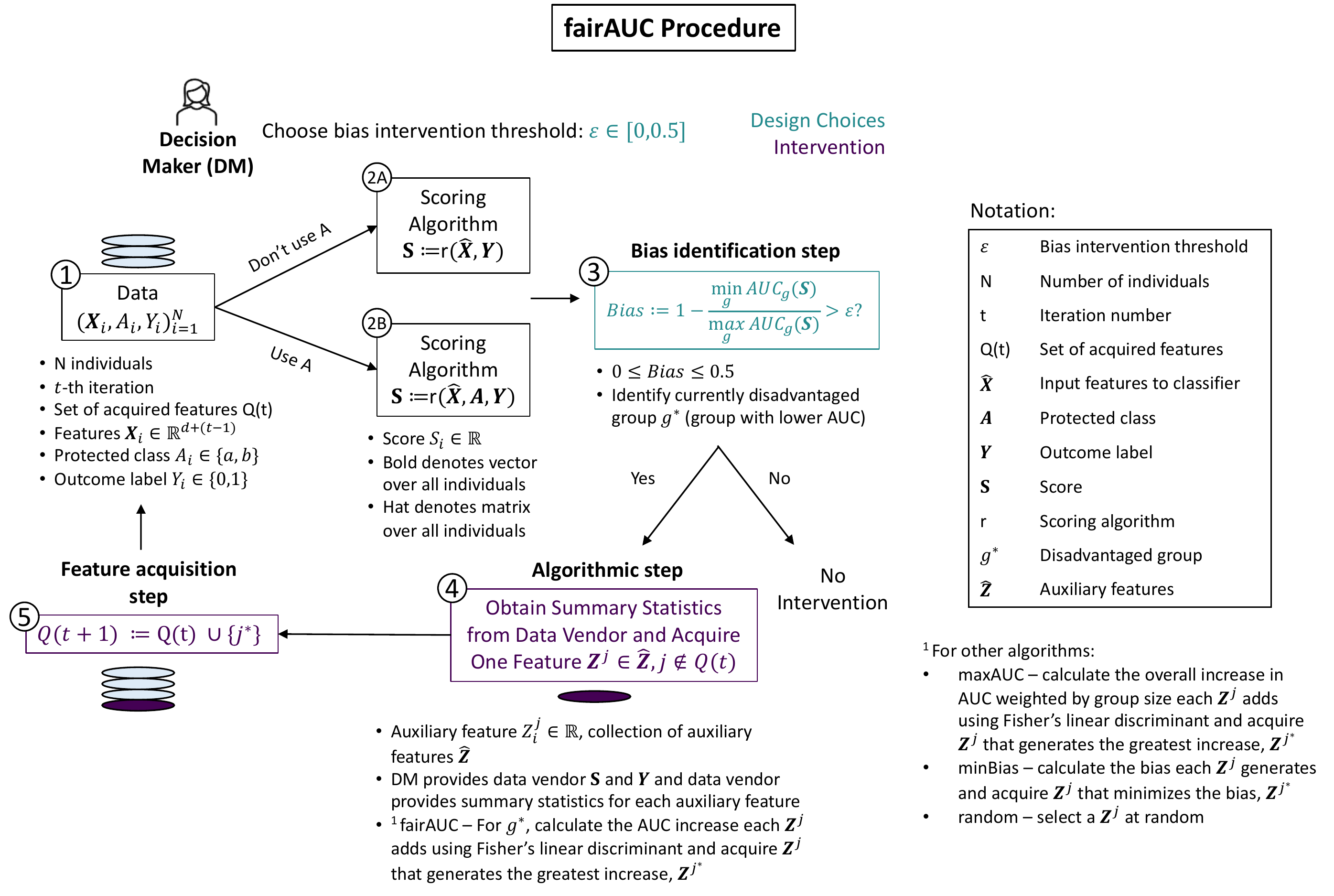}
    \end{center}
    \caption{Schematic of Proposed fairAUC Feature Acquisition Procedure}\label{figure:schematic}
  \end{figure}

  Figure \ref{figure:schematic} overviews our proposed \fairAUC\ \algo. \fairAUC\ seeks to improve the AUC of the lower-AUC group, rather than explicitly minimizing bias because the latter does not encourage learning in the long run. Our approach is greedy, focusing on feature acquisition, whereby we start with a set of initial features and then obtain additional features over rounds. The set of features available is summarized by: (a) moments of the data distribution of the auxiliary features, and (b) correlation with the data already collected. The \fairAUC\ \algo\ chooses the feature that most increases the AUC of the lower-AUC group each round, and proceeds through multiple rounds until a threshold condition is satisfied.

  Our research focuses on feature acquisition, which is related to, but distinct from the problem of feature selection.
  In a feature acquisition problem, the decision maker (manager) has a set of data corresponding to specific individuals. For example, firms have data on customers, and universities have data on applicants. These data correspond to specific features/columns. University applicants' features might include date of birth, gender, citizenship, high school GPA, and test scores. Similarly, a bank offering loans might have features including name, age, income, credit score, number of accounts, and account balance.

  Suppose a manager at a bank wants to predict the probability of repayment for a new loan applicant. The manager is also concerned about fairness across groups and finds that these features are very predictive (in the sense of achieving a high AUC) for group $a$ but not for group $b$. Group, here, could represent race, gender, or any other group of interest. The manager might want to acquire additional features from a data vendor to obtain more fair predictions. To compute the AUC improvement corresponding to each potential feature, the manager provides the data vendor with a set of individual records including an identifier $i$ for each individual, a score $S$ that is the output of a classifier predicting loan repayment, \red{and class $Y$ representing previous loan repayment behavior}. The data vendor computes the summary statistics of each available feature, specifically the class-conditional means, variances, and covariances of the features with respect to $S$. The manager then uses fairAUC to determine which feature to acquire.\footnote{A second method to compute the AUC improvement for each feature would be for the manager to acquire all the features available with the data vendor for a (small) subset of the firm's customers. The manager can then compute the class-conditional means, variances and covariances with respect to the score $S$, and then decide which feature to acquire next using fairAUC. The downside to this strategy is that it could end up being fairly expensive. For example, purchasing all the features from Aspire North, a data vendor, would cost more than \$15,000/1,000 individuals.}

  \subsubsection*{Sources of Auxiliary Features}

  The \fairAUC\ \algo\ assumes that there are auxiliary features available for acquisition. Some concrete potential sources of additional features include data vendors, like Axciom and Experian, and social media data vendors, like Brandwatch and Grepsr. Third-party data vendors acquire data by buying it, licensing it, or scraping it from public records and can build databases of thousands of features. \red{These features can include personal data, financial data, behavioral data, etc.}
  A number of startups that provide loans (e.g., Earnest, Kreditech) already use such alternative sources of data (e.g., social network data, e-commerce shopping behavior, LinkedIn profiles) to improve their loan provision decisions.\footnote{\url{https://www.upturn.org/static/files/Knowing_the_Score_Oct_2014_v1_1.pdf}} \fairAUC\ enables managers to determine which features to acquire from data vendors based on a small set of summary statistics. Note that \fairAUC\ applies to individual-level data. Table \ref{table:potential_aux_features} suggests potential auxiliary features and feature sources for several classification problems.

  \begin{table}[htbp]
    \small
    \centering
    \caption{Potential Auxiliary Features for Various Classification Problems}
    \vspace{1mm}
    \begin{tabular}{ccccc}
      \hline \hline
      Problem\tablefootnote{Loan provision \citep{kamiran2009classifying,hardt2016equality}, bail decision \citep{larson2016compas}, hiring \citep{schumann2020we}, extra medical attention \citep{obermeyer2019dissecting} are problems that have received attention in the fairness literature because of the large impact these decisions have on different groups.}
      & Prediction & First-party Data & Auxiliary Features & Source \\
      & Outcome $\hat{Y}$ & Examples $\mathbf{\hat{X}}$ & Examples $\mathbf{\hat{Z}}$ & \\
      \hline
      Loan provision & Default & Name, address, SSN, & Work history, college & Data vendor, \\
      & & credit history\tablefootnote{\url{https://apnews.com/article/47693e8cf0ee8328e82cd7ecd06d3df3}}
      & major, spending and & social data \\
      & & & saving behavior, & vendor \\
      & & & social network data\tablefootnote{\url{https://www.upturn.org/static/files/Knowing_the_Score_Oct_2014_v1_1.pdf}} & \\
      \hline
      Bail decision & Recidivism & Criminal history, & Spending  and saving & Data vendor, \\
      & & questionnaire responses\tablefootnote{\url{https://www.propublica.org/article/machine-bias-risk-assessments-in-criminal-sentencing}}
      & behavior, credit history, & social data \\
      & & & social network data & vendor \\
      \hline
      Hiring & Promotion & Resume, referral, & Social network data & Social data  \\
      & & interview & engagement & vendor \\
      \hline
      Extra Medical & Hospital & Biomarker values, & Wearables, social & Devices\tablefootnote{e.g., Apple Watch, Fitbit}, social  \\
      Attention & Readmission & comorbidities & network data & data vendor \\
      \hline \hline
    \end{tabular}
    \label{table:potential_aux_features}
  \end{table}

  \noindent More broadly, \fairAUC\ can also be thought about as a framework for determining which features should be collected by the \dm\ or decision maker going forward. For example, consider a government agency that wants to decrease homelessness in the long run across groups. The agency needs to determine whom it should allocate housing to and so the prediction problem is whether someone is going to be homeless two years after receiving housing. Suppose they know a potential set of features (e.g., family support, interests, education) from working with homeless individuals and have some prior distribution over the class-conditional means and variances but have not yet collected data corresponding to those features. Collecting all these features is costly and so the agency seeks to prioritize a single feature to collect. The decision maker can use intuition behind \fairAUC\ to guide which feature to prioritize.

  \subsubsection*{Advantages of fairAUC}

  The \fairAUC\ procedure has several appealing aspects. First, it can be used with a variety of classification algorithms. While the firm that is doing the prediction and making the decision (e.g., the bank in the case of offering a loan) has access to its classification algorithm, this algorithm does not need to be shared with the data provider.%

  In fairAUC, the classification model is retrained during each round by the manager and a score is obtained in the process. This score \red{and class (e.g., prior loan repayment)} is then shared with the data provider. The data provider uses the AUC formula to compute the AUC improvement corresponding to each of the available features. The formula only needs the scores, \red{classes,} and the new feature vector.

  Second, it uses minimal summary statistics of the auxiliary features, rather than requiring full access to the feature matrix. Third, \fairAUC\ does not treat either of the groups as permanently disadvantaged (or advantaged), unlike most research in the fairness literature. Rather, as we proceed with feature augmentation, the \textit{currently higher-AUC group can become the disadvantaged group after the addition of a new feature}. Thus, our goal in each round is to equalize the AUCs by improving the AUC of the \textit{currently disadvantaged} group, preventing reverse discrimination.

  \subsubsection*{Performance of fairAUC}

  We first characterize how the \fairAUC\ procedure improves the AUC of the lower-AUC group by a minimum threshold amount, thus providing theoretical performance guarantees \red{(under the binormal assumption)}. Next, we evaluate the performance of \fairAUC\ with synthetic data generated using a systematic data generation procedure proposed by \cite{guyon2003design}. We consider three natural benchmarks:  \minBias, which aims to directly minimize the bias in AUC across groups each round, \maxAUC, which ignores fairness to maximize the overall AUC weighted by group size, and a random feature acquisition approach, which may reflect the acquisition of easy-to-acquire data. Each \algo\ can be used with or without access to the protected class attribute during classification. This is important since certain laws prohibit the use of protected attributes.

  Compared to \maxAUC, \fairAUC\ achieves significantly greater levels of fairness (in terms of equalizing AUC), with fairly low tradeoffs in AUC. We characterize the accuracy-fairness tradeoff that is achievable using a weighted combination of fairness and AUC objectives, and find that \fairAUC\ obtains low levels of bias without significant sacrifice of overall AUC. Compared to \minBias, \fairAUC\ obtains far higher AUCs since \minBias\ fails to incentivize learning.

  We also evaluate the four procedures using the COMPAS dataset, a commonly used dataset in fairness studies, auxiliary data purchased from a data vendor, and the Diabetes dataset. We find similar to the synthetic data that fairAUC reduces bias compared to \maxAUC\ with a relatively low tradeoff in AUC. Relative to \minBias, \fairAUC\ achieves comparable levels of fairness but with far greater predictive accuracy.
  Finally, we confirm the robustness of our results using other data generating procedures and classifiers.

  \section{Related Literature}

  This paper touches on a few different streams of the fairness in algorithmic systems literature, which addresses questions around bias identification as well as bias reduction.

  \subsection{Sources of Bias}
  Researchers have documented a number of causes of bias \citep{barocas2016big} and have documented both human \citep{mejia2021transparency,benson2021potential} and algorithmic discrimination \citep{fu2020artificial}. It is critically important to understand the source of bias in order to provide guidance to firms and policymakers on how to address bias, since the recommended intervention would depend on the cause. For example, \cite{lambrecht2019algorithmic} find that advertising on Facebook with the objective of maximizing cost effectiveness inadvertently shows STEM career ads less frequently to women than men, and they report that the source of this bias is that the market bids up the advertising rates to reach women higher than that for men. Thus, in this case market forces are potentially the cause rather than an algorithm. Our study specifically considers that bias can arise due to the nature of data collected, rather than the algorithm. We focus on feature acquisition and its impact on classification performance for members of different protected groups.

  \subsection{Fairness Criterion}
  To quantify bias, a measure relevant to the problem must be used. Several fairness criteria have been proposed \citep{le2022survey}. In general, the various measures aim to achieve specific criteria, namely independence \citep{dwork2012fairness,kamiran2012preprocessing,feldman2015certifying}, sufficiency \citep{chouldechova2017fair}, and separation \citep{hardt2016equality,zafar2017fairness,kallus2019fairness}.
  It has been shown under mild assumptions that no measure of fairness can simultaneously achieve two of the three criteria \citep{kleinberg2016inherent,chouldechova2017fair,barocas2019fairness}. Therefore, the appropriate fairness criterion depends on the problem of interest (see \cref{appendix:fairness_measures} for a comparison of measures).
  Other measures distinguish between individual versus group fairness, and intertemporal ideas of fairness \citep{gupta2019individual}. We study group fairness and the criterion we focus on is separation, which recognizes that the protected attribute may be correlated with the target variable. For example, the base rates of loan repayment may differ among groups so a bank may be justified in having different lending rates for different groups \citep{barocas2019fairness}. The fairness measure we use is related to equalized odds, which achieves separation, in that it is also derived from the ROC curve. Our focus, however, is equalized AUCs, also known as accuracy equity in the literature. \red{As discussed earlier, AUC is not only a metric commonly used to measure classification performance in machine learning but also has many desirable qualities, such as being base rate invariant and threshold invariant.}

  \subsection{Bias Reduction Strategies}\label{section:lit_bias_reduction}
  Bias reduction strategies can occur prior to (pre-processing), during (in-processing), and after (post-processing) model training. Pre-processing strategies alter the feature space to be uncorrelated with the protected attribute \citep{kamiran2012preprocessing,zemel2013learning,feldman2015certifying,celis2020preprocessing,shimao2019strategic}. In-processing strategies directly incorporate the fairness constraint into the optimization problem \citep{dwork2012fairness,zafar2017fairness,woodworth2017learning,celis2019meta}. Using AUC as a fairness metric with in-processing has proven challenging and remains an open problem \citep{celis2019meta}. Post-processing strategies occur after classifier training and manipulate the classifier to be uncorrelated with the protected attribute \citep{hardt2016equality}. \cite{noriega2019active} demonstrate that post-processing strategies that rely on randomization to achieve equalized odds are inefficient and Pareto sub-optimal.
  Most proposed strategies assume the dataset to be fixed and take an algorithmic approach to reducing bias but practitioners have voiced a need for data collection guidance \citep{holstein2019improving}.
  We take a different approach by developing a \algo\ for additional feature acquisition, which occurs during the data collection stage. Our solution provides guidance on which additional features should be acquired to improve the AUC of the lower-AUC group and ultimately equalize AUCs across groups.

  \paragraph{Fair Feature Acquisition.} Our main contribution is to the feature acquisition with fairness considerations literature. This literature addresses \textit{which features} should be acquired and in some cases \textit{which individuals} \dms\ should acquire additional features for. For example, \cite{cai2020fair} develop an algorithm to jointly determine which individuals a firm should collect additional features for and then allocate resources to. \cite{noriega2019active}, \cite{bakker2021beyond}, and our work complement this stream of research by determining \textit{which} features to acquire to achieve various measures of fairness. While \cite{noriega2019active} and \cite{bakker2021beyond} require the full feature matrix to be known for all individuals in order to determine which feature to acquire, \fairAUC\ requires only a set of summary statistics. %

  \section{Preliminaries and Assumptions} \label{model}
  \subsection{Preliminaries}

  We consider a standard binary classification problem with two groups.
  The dataset consists of $N$ i.i.d. data points $(X_i,A_i,Y_i)_{i=1}^N$ sampled from a distribution $\mathcal{D}$.
  Here the input feature $X_i \in \mathbb{R}$, the group $A_i \in \{a,b\}$, and the class label $Y_i \in \{0,1\}$. Note that here we specify $X_i$ as a scalar ``score'' for notational simplicity, whereas our \fairAUC\ procedure accommodates general vectors $X_i \in \mathbb{R}^d$.\footnote{For example, with a logistic regression specified as $\Pr[Y=1|X] = \dfrac{\exp(X^\top\theta)}{1+\exp(X^\top\theta)}$, the score would be $S=X^\top\theta$, and new data that is above a score threshold would be classified as 1.}
  We discuss the general case in \cref{section:fair_auc_over_multiple_iterations}.
  Here, and subsequently, we drop the subscript from  $(X_i,A_i,Y_i)$ when we do not want to refer to a specific individual.
  Let $p_{g0}(x)\coloneqq \Pr_{(X,A,Y) \sim \mathcal{D}}[X=x|A=g,Y=0]$ denote the distribution of the input feature belonging to the negative class for each group.
  Similarly, let $p_{g1}(x)\coloneqq \Pr_{(X,A,Y) \sim \mathcal{D}}[X=x|A=g,Y=1]$ denote the distribution of the input feature belonging to the positive class for each group.

  For a data point $(X,A,Y),$ consider a  binary classifier $r$ that, for  a threshold $\tau \in \mathbb{R}$ is defined as:
  \begin{equation}
    r(X,A,Y)\coloneqq
    \begin{cases}
      1, & \text{if}\ X \geq \tau, \notag\\
      0, & \text{if}\ X < \tau.
    \end{cases}
  \end{equation}
  \noindent
  The true positive rate (TPR) measures the proportion of individuals in the positive class being correctly classified as positive.
  The false positive rate (FPR) measures the proportion of individuals in the negative class being incorrectly classified as positive.
  Thus, the TPR and FPR are bounded below by $0$ and bounded above by $1$. The TPR and FPR of group $g \in \{a,b\}$ can be written as functions of the threshold $\tau$:
  \begin{equation}\label{eq:TPR_general}
    \mathrm{TPR}_g(\tau) \coloneqq \int_{\tau }^{\infty}p_{g1}(x)dx \qquad \text{and} \qquad     \mathrm{FPR}_g(\tau) \coloneqq \int_{\tau}^{\infty}p_{g0}(x)dx.
  \end{equation}
  These two give rise to the ROC curve as follows: the TPR maps the threshold $\tau$ to the $y$-axis and the FPR maps $\tau$ to the $x$-axis.
  Formally, for a group $g$, ROC is defined as $\mathrm{ROC}_g(\tau)\coloneqq (\mathrm{FPR}_g(\tau), \mathrm{TPR}_g(\tau))$. 
  The area under the two-dimensional ROC curve (AUC) aggregates the information captured in the TPR and FPR and is defined for a group $g$ as follows:
  \begin{definition}[\bf Area under the ROC curve (AUC)]\label{AUC_def}
    \begin{equation}\label{eq:AUC_general}
      \mathrm{AUC}_g \coloneqq \int_{0}^{1}\mathrm{TPR}_g(\mathrm{FPR}_g^{-1}(x))dx.
    \end{equation}
  \end{definition}
  \noindent
  AUC ranges from $0$, which occurs when the classifier predicts the opposite of the class label, to $1$, which occurs when the classifier can perfectly classify the two classes. An AUC of $0.5$ means the classifier cannot distinguish between the two classes and is represented by the diagonal line on the ROC plane.\footnote{In reality, a classifier is unlikely to obtain an AUC less than $0.5$ since making the opposite prediction would improve the AUC. There are cases in which an AUC less than $0.5$ can occur. Consider a dataset comprised of two groups where one group makes up the vast majority of the data. Suppose the features of the majority group are negatively correlated with the same features of the minority group. If only a single set of weights is used in classification, the majority group will dominate the determination of the weights and the resulting AUC for the minority group will be worse than random.}
  {$\AUC$ depends on the specific scoring rule $r$. For now, we consider a general scoring rule and later consider specific scoring rules such as those derived from a linear classifier, a logistic regression classifier, or other generalized linear models.}

  We measure bias by comparing the AUCs obtained from the groups $g\in\{a,b\}$:
  \begin{definition}[\bf Bias]\label{bias_def}
    \begin{equation}
      \mathrm{Bias} \coloneqq 1 - \frac{\min_g(\mathrm{AUC}_g)}{\max_g(\mathrm{AUC}_g)}.
    \end{equation}
  \end{definition}
  \noindent
  Bias ranges from $0$ to $1$, with larger values representing greater inequality between groups.

  \subsection{Class-conditional Means, Variances, and the Binormal Assumption}
  The class-conditional means and variances of $X$ for each group $g$ are defined as
  $\mu_{gy} \coloneqq \mathbb{E}[X|A=g,Y=y]$
  and
  $\sigma^2_{gy} \coloneqq \mathrm{Var}[X|A=g,Y=y],$ respectively. The unconditional (class-independent) variance of $X$ for each group is $\mathrm{Var}[X|A=g]$.

  To obtain an analytical relationship between moments of the data and AUC, we assume that for each group the input feature follows a binormal distribution \citep{pesce2007reliable}, since the ROC  (and therefore its AUC) is known to be robust to departures from this assumption \citep{hanley1996binormal}.
  This binormal assumption produces four Gaussian distributions across the two classes and two groups: $\mathcal{N}(\mu_{a0},\sigma_{a0}^2)$, $\mathcal{N}(\mu_{a1},\sigma_{a1}^2)$, $\mathcal{N}(\mu_{b0},\sigma_{b0}^2)$, $\mathcal{N}(\mu_{b1},\sigma_{b1}^2)$, where $0$ and $1$ represent the classes, and $a$ and $b$ represent the groups. We assume the means of the positive classes are greater than the means of the negative classes for each group $(\mu_{a1} \geq \mu_{a0}, \;\mu_{b1} \geq \mu_{b0})$ and that the conditional variances are positive. Figure \ref{figure:intro_figures} (Left) displays a density plot of two binormal distributions for which the class-conditional variances within each group are equal but the class-conditional variances for group $a$ are smaller than those for group $b$. %

  \begin{figure*}
    \begin{minipage}{0.5\textwidth}
      \begin{center}
        \includegraphics[height=4cm]{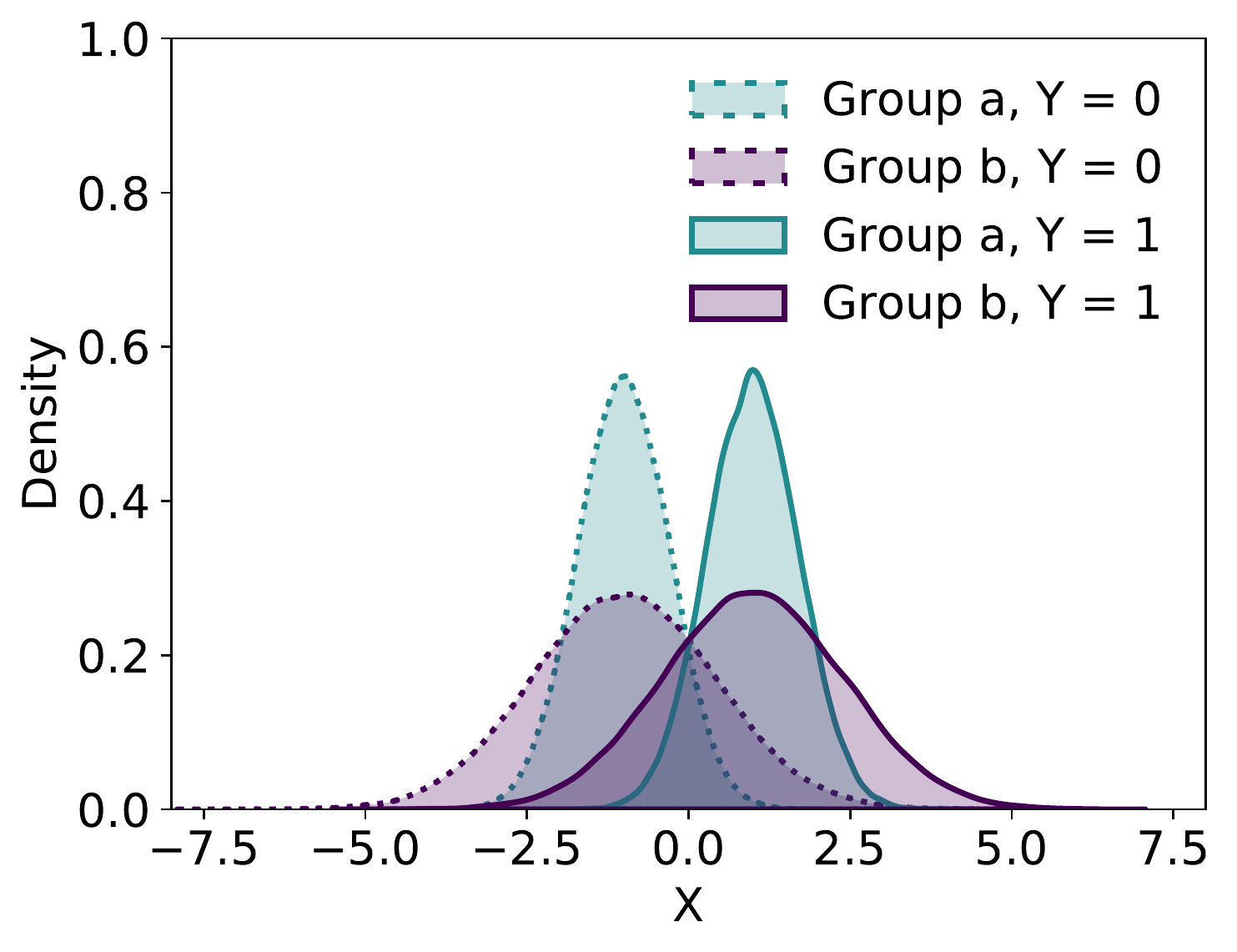}
      \end{center}
    \end{minipage}
    \begin{minipage}{0.5\textwidth}
      \begin{center}
        \includegraphics[height=4cm]{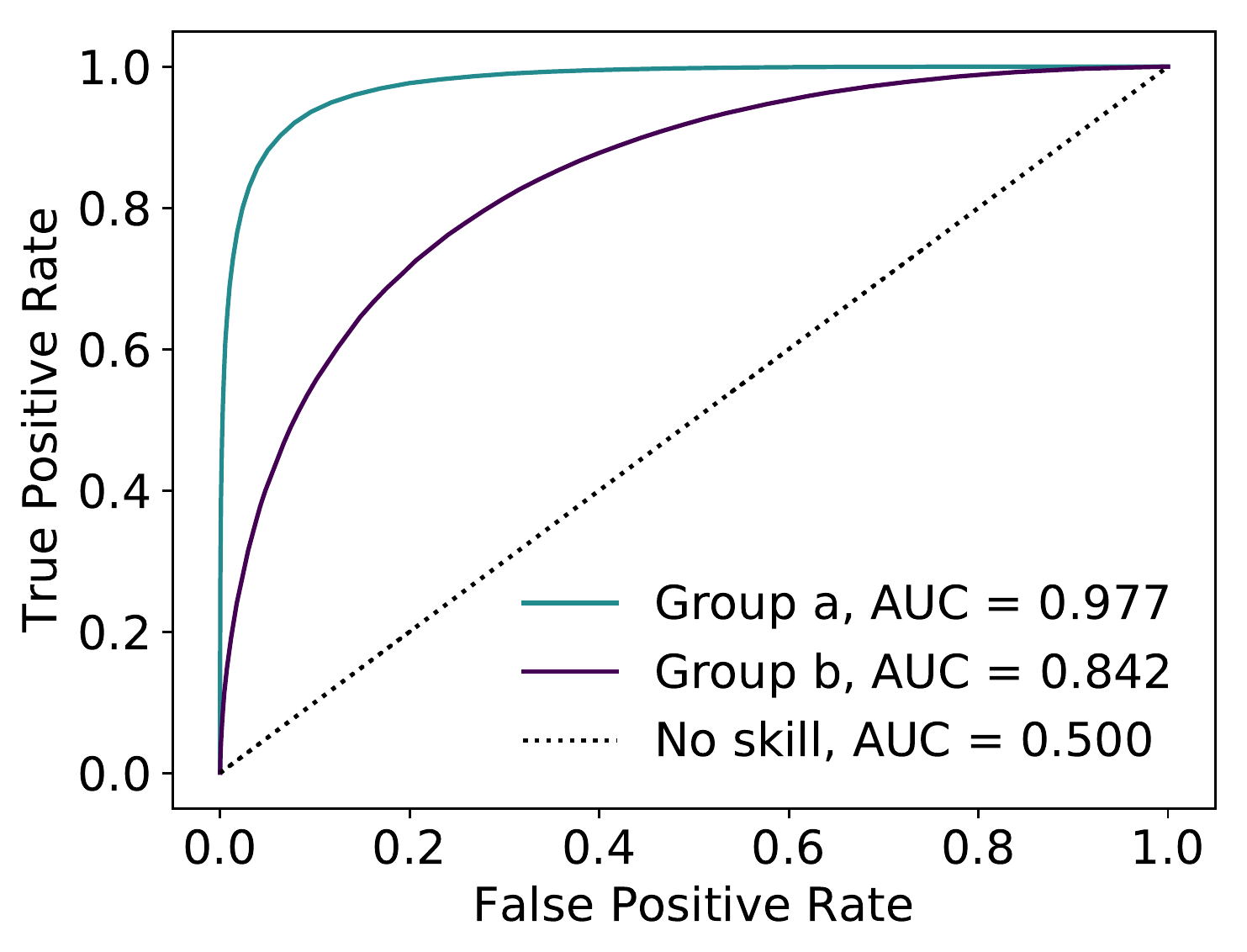}
      \end{center}
    \end{minipage}
    \caption{\footnotesize
    (Left) Binormal density plots where the class means of groups a and b are equal but the class-conditional variances of b are greater than those of a.
    (Right) ROC curves and AUC by group.}\label{figure:intro_figures}
  \end{figure*}

  Incorporating the binormal assumption, TPR and FPR from Equation \eqref{eq:TPR_general}  for group $g$ can be written as:
  \begin{align}
    \mathrm{TPR}_g(\tau) = 1 - \Phi \left(\frac{\tau-\mu_{g1}}{\sigma_{g1}} \right) \qquad \text{and} \qquad
    \mathrm{FPR}_g(\tau) = 1 - \Phi \left(\frac{\tau-\mu_{g0}}{\sigma_{g0}} \right).
  \end{align}
  Where $\Phi(\cdot)$ represents the standard normal cumulative distribution function.
  We can now express the AUC defined in Equation \eqref{eq:AUC_general} as a function of the class-conditional means and variances of each group:
  \begin{equation} \label{eq:AUC_binormal}
    \mathrm{AUC}_g = \Phi \left(\frac{\mu_{g1}-\mu_{g0}}{\sqrt{\sigma_{g0}^2+\sigma_{g1}^2}} \right).
  \end{equation}
  \noindent
  Figure \ref{figure:intro_figures} (Right) displays the ROC curves and their associated AUCs from the binormal distributions shown in Figure \ref{figure:intro_figures} (Left). The diagonal line represents random guessing.

  \section{Methodology}

  Many papers consider the unconditional distributions of the input feature for each group \citep{corbett2018measure,chen2018why,emelianov2020fair}. It may be expected that higher variance, flatter unconditional distributions generate higher AUCs since greater spread provides more information. Suppose $X$ takes just one value (is deterministic). Then the unconditional variance is zero and the classifier learns nothing from this data. Given this example, one may believe that a higher unconditional variance corresponds to better classification. However, this thought experiment conflates the separation of means with variance. Analyzing the unconditional distribution mixes together base rates, class-conditional means, and class-conditional variances, obscuring the relationship between the data and AUC. \cref{section:uncond_var} formalizes the previous ideas.

  Here we highlight which features of the data distributions pin down AUC. Equation \eqref{eq:AUC_binormal} informs us that increasing the difference in class means and decreasing the class-conditional variances increase the AUC. Figure \ref{figure:intro_figures} (Left) and (Right) visualize an example in which the differences in class means are equal between the two groups but the conditional variances of group $b$ are larger than those of group $a$. Because of the larger conditional variances, the AUC of $b$ is lower than the AUC of $a$. Finally, the base rate $\pi_g$ does not appear in the AUC formula, reinforcing the idea that differences in base rates between the two groups will not contribute to bias with respect to AUC. %

  \subsection{Strategies for Additional Feature Acquisition}
  We consider strategies for selecting additional costly features for classification with fairness considerations. One natural strategy would be to select features that minimize the bias across groups.
  We consider a greedy strategy, \minBias, that chooses the feature minimizing the difference between AUCs across groups  in each round. However, such a strategy fails to incentivize additional learning; when features that are relatively uninformative for both groups minimize bias, it will select those features rather than features that improve AUC.

  Another strategy is to select features that improve the AUC of the disadvantaged group (group with lower AUC). We develop a \algo\ which does exactly this, noting that with a greedy feature acquisition procedure over multiple rounds, the group that is (dis)advantaged may change across rounds. Note that such a \algo\ does not guarantee a reduction in bias because it is possible that the additional feature is even more predictive for the higher-AUC group or dramatically increases the AUC of the disadvantaged group much more than for the advantaged group. %

  Thus, the natural question is: given the data we have, what additional features(s) should we acquire to \textit{most improve the AUC of the currently disadvantaged group}?
  Because AUC is calculated using a scalar score, we must select a dimensionality reduction method which aggregates our existing data with the additional feature(s) into one dimension.
  Here we use Fisher's linear discriminant (FLD) as our dimensionality reduction method because it generates the linear projection which maximizes AUC \citep{su1993linear}. \red{Later, we will broaden the set of classification models to extend beyond FLD to the class of Generalized Linear Models (GLMs).}

  \subsubsection{Fisher's Linear Discriminant (FLD)}\label{section:FLD}
  We make a few simplifying assumptions.
  First, we assume only {\em one} additional feature can be acquired in each round and develop a greedy strategy. We relax this assumption in \cref{appendix:multiple_variables}.
  Second, we assume the existing input feature the \dm\ has can be represented in one dimension (i.e., the output of a score function).
  Third, we assume that a binormal distribution provides a reasonable approximation for the features.
  Using FLD, we determine the benefit of a new feature to the AUC of the disadvantaged group under consideration (denoted $g^\star$).

  Let $n$ denote the number of individuals in the disadvantaged group $g^\star$.
  Consider a dataset $(X_i,A_i,Y_i)_{i=1}^n$ where we have already collected one (non-protected) feature $X_i \in \mathbb{R}$ for each individual $i$.
  Since our focus is only on a single group, we drop the group subscript notation used in the previous sections.
  Moreover, assume we have access to an auxiliary feature $(Z_i)_{i=1}^n$, which we could choose to acquire.

  We seek to determine the benefit to $g^\star$ of acquiring $(Z_i)_{i=1}^n$.
  For each outcome class $y\in\{0,1\}$, the class-conditional mean vector and covariance matrix of $(X_i,Z_i)_{i=1}^n$ are:
  \begin{align*}
    \bm{\mu_y} \coloneqq
    \begin{bmatrix}
      \mu_{X,y}\\
      \mu_{Z,y}
      \end{bmatrix} =
      \begin{bmatrix}
        \mathbb{E}[X|Y=y]\\
        \mathbb{E}[Z|Y=y]
      \end{bmatrix}
      \quad\text{and}\quad
      \bm{\Sigma_y} \coloneqq
      \begin{bmatrix}
        \sigma_{X,y}^2 & \rho_y \sigma_{X,y} \sigma_{Z,y} \\
        \rho_y \sigma_{X,y} \sigma_{Z,y} & \sigma_{Z,y}^2
        \end{bmatrix}.
      \end{align*}
      \noindent
      Where $\sigma_{X,y}^2=\mathrm{Var}[X|Y=y]$, $\sigma_{Z,y}^2=\mathrm{Var}[Z|Y=y]$, and $\rho_y=\frac{\mathrm{Cov}[X,Z|Y=y]}{\sigma_{X,y}\sigma_{Z,y}}$. Note that $\rho_y$ represents the class-conditional correlation of $X$ and $Z$ and not the unconditional correlation.

      Let $\bm{w}$ represent a potential projection direction that projects $(X_i,Z_i)$ for each individual $i \in [n]$ to $\mathbb{R}$, combining the two features into a single value.
      Then the projected class mean $\Tilde{\mu}_y \in \mathbb{R}$ is defined as $\Tilde{\mu}_y \coloneqq \bm{w}^\top\bm{\mu_y}$ and the projected class-conditional variance $\Tilde{\sigma}^2_y \in \mathbb{R}$ is defined as $\Tilde{\sigma}^2_y \coloneqq \bm{w}^\top \bm{\Sigma_y} \bm{w}$.
      \noindent
      The FLD objective function is known to maximize AUC \citep{su1993linear}. In terms of the projected means and variances, the FLD objective is:
      $   J(\bm{w}) \coloneqq \frac{(\Tilde{\mu}_1-\Tilde{\mu}_0)^2}{\Tilde{\sigma}_0^2+\Tilde{\sigma}_1^2}.
      $
      \noindent In terms of the pre-projection mean vectors and covariance matrices, it is:
      $$J(\bm{w}) = \frac{\bm{w}^\top(\bm{\mu_1}-\bm{\mu_0})(\bm{\mu_1}-\bm{\mu_0})^\top \bm{w}}{\bm{w}^\top(\bm{\Sigma_0}+\bm{\Sigma_1}) \bm{w}}.$$
      \noindent The projection direction $\bm{w}$ which maximizes $J(\bm{w})$ can be found by solving a generalized eigenvalue problem \citep{duda2006pattern}.
      The optimal linear projection direction (when $\bm{\Sigma_0}+\bm{\Sigma_1}$ is invertible) is given by:
      \begin{equation}\label{eq:fld_projection}
        \bm{w}^\star = (\bm{\Sigma_0}+\bm{\Sigma_1})^{-1} (\bm{\mu_1}-\bm{\mu_0}).
      \end{equation}
      \noindent
      Plugging $\Tilde{\mu}_y = \bm{w}^{\star\top} \bm{\mu_y}$ and $\Tilde{\sigma}^2_y = \bm{w}^{\star \top} \bm{\Sigma_y} \bm{w}^\star$ into Equation \eqref{eq:AUC_binormal} yields the AUC of the optimal linear combination of input features $(X,Z)$:
      \begin{equation}\label{eq:fld_auc}
        \mathrm{AUC}(X,Z) = \Phi \left(\sqrt{(\bm{\mu_{1}}-\bm{\mu_{0}})^\top (\bm{\Sigma_0}+\bm{\Sigma_1})^{-1}(\bm{\mu_{1}}-\bm{\mu_{0}})} \right).
      \end{equation}
      \noindent
      The benefit to the disadvantaged group of acquiring $Z$ is the difference between this new value of AUC and the previous value of AUC that used only $X$.

      {More generally, one can consider other scoring rules such as those derived from Generalized Linear Models (GLMs).
      A GLM is specified by a vector of weights $\theta\in \R^2$ and an increasing and invertible link function $\psi\colon \R\to \R$.
      It generates scores $\psi^{-1}\inparen{\theta_1 X+\theta_2 Z},$ and given a threshold $\tau$, it labels a sample $X$ as 1 if and only if $$\psi^{-1}(\theta_1 X+\theta_2 Z)>\tau.$$
      For instance, if $\psi$ is the logit function, then the GLM is a logistic regression classifier which predicts 1 if and only if $$\frac{\exp\inparen{\theta_1 X+\theta_2 Z}}{1+\exp\inparen{\theta_1 X+\theta_2 Z}}>\tau.$$
      It turns out that the ROC curve of the GLM $G$ specified by $\theta$ and $\psi$ is the same as \red{that for} the linear classifier $L$ which uses the linear projection $\theta_1 X+\theta_2 Z$ and, hence, the AUC of $G$ and $L$ is the same.\footnote{This is because $G$ at threshold $\tau$ makes the same prediction as $L$ at threshold $\psi(\tau)$ --  see \cref{claim:auc_indep_of_link}.}
      Thus, \cref{eq:fld_auc} is also the highest AUC achievable by a GLM that uses $X$ and $Z$ to predict $Y$.}

      So far, $Z$ has represented an arbitrary feature available for acquisition.
      When given a choice over many possible features, which feature $Z$ maximizes Equation \eqref{eq:fld_auc} for a given $X$?

      \subsection{fairAUC Procedure}\label{section:fair_auc_over_multiple_iterations}
      We now present our \fairAUC\ procedure, for which Equation \eqref{eq:fld_auc} serves as the backbone.
      It relies on knowing only a few summary statistics of the data. It can be used with data vendors who provide costly features. Alternatively, \dms\ may collect a small sample of additional features and estimate the benefit  of each using this strategy prior to collecting the features for all individuals.

      It begins by taking in the data the \dm\ has for $N$ individuals $(\bm{X}_i,A_i,Y_i)_{i=1}^N$, the \dm's scoring algorithm $r$, and the level of acceptable bias $\epsilon$.
      As before, we drop the subscript from  $(\bm{X}_i,A_i,Y_i)$ when we do not refer to a specific individual.
      The input data $\bm{X} \in \mathbb{R}^d$, the group $A \in \{a,b\}$, and the class $Y \in \{0,1\}$.

      The \dm\ aggregates the features in $\bm{X}$ into a single score $S \in \mathbb{R}$ using a fixed scoring algorithm $r$.
      Our framework allows for any scoring algorithm $r$. When the data is binormal, FLD generates the optimal AUC. Otherwise, FLD is meant to serve as a heuristic and one can replace it with a different scoring algorithm $r$. We show in \cref{section:nonlinear_classification} that different classifiers, including nonlinear ones, end up acquiring nearly all the same features, \red{although at times in different orders}.
      Furthermore, $r$ may or may not use the protected attribute $A$ depending on the context.
      For instance, FICO is prohibited from using characteristics like race, gender, and marital status in producing its credit score.
      We refer to using $A$ as using separate classifiers for each group and not using $A$ as using only a single classifier for both groups.

      In each round $t$ of the \fairAUC\ \algo, there is a bias identification step, an algorithmic step, and a feature acquisition step.
      In the bias identification step, we first calculate the AUC for each of the groups from the scores $S$.
      The bias of the model is calculated from the AUCs of the two groups.
      If the bias is smaller than a given tolerance level $\epsilon$, the \dm\ does not need to take any intervention to reduce bias.
      However, if the bias is larger than $\epsilon$, the \dm\ acquires one additional feature per round.
      The group with lower AUC is referred to as the currently disadvantaged group, which can vary over the rounds of feature acquisition, and is denoted $g^\star$.

      In the algorithmic step, \fairAUC\ aims to acquire the feature that most increases the AUC of $g^\star$ using the FLD heuristic explained in the previous section (i.e., Equation \eqref{eq:fld_auc}).
      As previously discussed, FLD generates a linear combination of features which maximizes AUC and requires only summary statistics to calculate.
      Let $(\bm{Z_i})_{i=1}^N$ where $\bm{Z_i} \in \mathbb{R}^{d'}$ represent the auxiliary features available for acquisition.
      Let $m = d + d'$ capture the total number of features that exist in the data the \dm\ owns and in the auxiliary data available for acquisition.
      Let $\{1,\ldots,d'\}$ denote the indices of all auxiliary features that are available for acquisition.
      In any given round $t$, we use $Q(t) \subset \{1,\ldots,d'\}$ to denote the set of auxiliary features acquired so far.
      Initially, we have that $Q(0)=\emptyset$.
      Hence, the set of features available is $[d']\setminus Q$.
      We assume that the cost of each feature is the same and that a feature is acquired for all $N$ individuals.

      For group $g^\star$, the \dm\ obtains the class-conditional means of each of the features available for acquisition in $[d']\setminus Q$ as well as the class-conditional covariance matrices of each of the available features and the score $S$, requiring the \dm\ to share $S$ with the data vendor each round and $Y$ the first round.\footnote{Note that in the case of working with a data vendor the \fairAUC\ \algo\ assumes that the data vendor also knows $A$. In practice, this information may need to be shared. A benefit of the \fairAUC\ \algo\ is that it does not require the \dm\ to share $X$ with the data vendor, only $S$. It also does not require the data vendor to share more than a few summary statistics.}  We relax this requirement in \cref{appendix:zero_correlation}, where we assume the score $S$ to be uncorrelated with the auxiliary features, eliminating the need to share $S$.
      The conditional means and covariances inform how valuable each of the additional features is to the \dm\ in terms of increasing the AUC of $g^\star$.
      The \dm\ acquires the feature which maximizes the AUC of group $g^\star$.

      In the feature acquisition step, $Q(t)$ is updated to include this new feature. The feature acquired is concatenated with the existing dataset and becomes the input to the next round.
      Procedure 1 formalizes each iteration of \fairAUC.
      To simplify the notation, let $\hat{\bm{X}}$ denote the collection of $(\bm{X}_i)_{i=1}^N$, and similarly $\hat{\bm{Z}}$ the collection of $(\bm{Z}_i)_{i=1}^N$.
      $\bm{Z}_i$ denotes a row vector of features for each individual, $\bm{Z}^j$ denotes a column feature vector across all individuals, and $\hat{\bm{Z}}$ denotes the collection of features for all individuals.

      \begin{algorithm}%
        \SetAlgorithmName{Procedure 1}{procedure}{}
        \caption{{\fairAUC} ($t$-th iteration)\label{Procedure}}
        \footnotesize
        \KwIn{data owned $(\bm{X_i},A_i,Y_i)_{i=1}^N$, scoring algorithm $r$, bias threshold $\epsilon$, set of acquired features $Q(t)$, data available for acquisition $(\bm{Z_i})_{i=1}^N$;}
        \KwOut{$Q(t+1)$;}
        \eIf{$\bm{A}$ cannot be used}{
        $\bm{S} \coloneqq r(\hat{\bm{X}},\bm{Y})$\;
        }{
        $\bm{S} \coloneqq r(\hat{\bm{X}},\bm{A},\bm{Y})$\;
        }
        \For {group $g \in \{a,b\}$}{
        compute $\mathrm{AUC}_g(S)$ (Definition \ref{AUC_def}) \;
        }
        $g^\star  \coloneqq \arg\min_g\mathrm{AUC}_g(S)$ (Disadvantaged group) \;
        $\mathrm{Bias} \coloneqq 1 - \frac{\min_g(\mathrm{AUC}_g(S))}{\max_g(\mathrm{AUC}_g(S))}$ (Definition \ref{bias_def}) \;
        \eIf{$\mathrm{Bias}>\epsilon$}{
        \For {feature $\bm{Z}^j \in \hat{\bm{Z}},j \notin Q(t)$}{
        for feature $\bm{Z}^j$, group $g^\star$, and score $\bm{S}$, obtain class-conditional means, $\bm{\mu_{0}},\bm{\mu_{1}}$, and covariance matrices, $\bm{\Sigma_{0}},\bm{\Sigma_{1}}$ (Summary Statistics by Group Subroutine)\;
        $h(\bm{S},\bm{Z}^j) \coloneqq \Phi \left(\sqrt{(\bm{\mu_{1}}-\bm{\mu_{0}})^\top (\bm{\Sigma_0}+\bm{\Sigma_1})^{-1}(\bm{\mu_{1}}-\bm{\mu_{0}})} \right)$\;
        }
        $j^\star \coloneqq \arg\max_j h(\bm{S},\bm{Z}^j)$\;
        acquire feature $\bm{Z}^{j^\star}$\;
        return $Q(t+1) \coloneqq Q(t) \cup \{j^\star\}$\;
        }{
        no intervention;
        }
      \end{algorithm}

      \begin{algorithm}%
        \SetAlgorithmName{Subroutine}{subroutine}{list of subroutines}
        \caption{{Summary Statistics by Group} \label{Subroutine}}
        \footnotesize
        \KwIn{feature available for acquisition $\bm{Z}$, group $g$, existing score $\bm{S}$;}
        \KwOut{class-conditional mean vectors $\bm{\mu_{0}},\bm{\mu_{1}}$, class-conditional covariance matrices $\bm{\Sigma_{0}},\bm{\Sigma_{1}}$;}
        \For {class $y \in \{0,1\}$}{
        $n \coloneqq n_{A=g,Y=y}$ \;
        $\bm{\mu_y} \coloneqq
        \begin{bmatrix}
          \Bar{S}_y\\
          \Bar{Z}_y
          \end{bmatrix} =
          \begin{bmatrix}
            \frac{1}{n}\sum_{i:
            \begin{smallmatrix}
              A_i=g\\
              Y_i=y
              \end{smallmatrix}}S_i \\
              \frac{1}{n}\sum_{i:
              \begin{smallmatrix}
                A_i=g\\
                Y_i=y
                \end{smallmatrix}}Z_i
                \end{bmatrix}$ \;
                $\bm{\Sigma_y} \coloneqq
                \begin{bmatrix}
                  \sigma_{S,y}^2 & \rho_y \sigma_{S,y} \sigma_{Z,y} \\
                  \rho_y \sigma_{S,y} \sigma_{Z,y} & \sigma_{Z,y}^2
                  \end{bmatrix}$ \\
                  \noindent where
                  $\sigma_{S,y}^2=
                  \frac{1}{n-1} \sum_{i:
                  \begin{smallmatrix}
                    A_i=g\\
                    Y_i=y
                    \end{smallmatrix}}(S_i-\Bar{S}_y)^2$,
                    $\sigma_{Z,y}^2=
                    \frac{1}{n-1}  \sum_{i:
                    \begin{smallmatrix}
                      A_i=g\\
                      Y_i=y
                      \end{smallmatrix}}(Z_i-\Bar{Z}_y)^2$,
                      and $\rho_y=
                      \frac{1}{(n-1)\sigma_{s,y}\sigma_{Z,y}} \sum_{i:
                      \begin{smallmatrix}
                        A_i=g\\
                        Y_i=y
                        \end{smallmatrix}}(S_i-\Bar{S}_y)(Z_i-\Bar{Z}_y)$\;
                        }
                        return $\bm{\mu_{0}},\bm{\mu_{1}},\bm{\Sigma_{0}},\bm{\Sigma_{1}}$
                      \end{algorithm}

                      \subsection{Theoretical Guarantees on Improvement of AUC by fairAUC}\label{section:theoretical_guarantees}
                      In this section, we present our theoretical result on the fairAUC procedure in the binormal framework.

                      {Different classifiers that use the same set of features can have different AUCs.
                      We define the AUC of group $g$ at the start of iteration $t$, as the highest AUC that a GLM can achieve for samples in group $g$ using features $Q(t)$.
                      For instance, if the link function $\psi$ is the logit function (corresponding to logistic regression), then AUC of group $g$ is equal to the highest AUC achieved by a logistic regression classifier using $Q(t)$.
                      The specific choice of $\psi$ does not affect the definition because it turns out that changing the link function $\psi$ does not change the ROC curve of a GLM and, hence, does not change its AUC (see \cref{claim:auc_indep_of_link}).}

                      {At a high level, our result uses the fact that if $S$ is an FLD-based score on the acquired features $Q(t)$,
                      then the AUC achieved by scores $S$ can be computed using \cref{eq:AUC_binormal} and it is equal to the AUC of group $g$ (as defined above).
                      To prove our result, we provide fairness guarantees for \fairAUC{} that uses FLD-based scores $S$.}
                      We begin by analyzing the improvement in AUC of the disadvantaged group $g^\star$.
                      In each iteration $t\in \N$ of \fairAUC{}, where the AUC for $g^\star$ is bounded away from 1, and there is at least one auxiliary feature present which has ``low'' class-conditional covariances with the current scores $S$  for $g^\star$ and has ``bounded'' class-conditional variances and means for $g^\star$,
                      \fairAUC{} is guaranteed to improve the AUC of $g^\star$ by at least a constant in iteration $t$.

                      \begin{proposition}{\bf (Theoretical guarantee on fairAUC; informal statement. See Theorem \ref{thm:main_result} for formal version).}\label{thm:main_result_main_body}
                        In the binormal framework, if the summary statistics subroutine outputs the unbiased means and covariances of the queried features
                        {and $S$ is the FLD-based score,} then
                        we can provably guarantee that in each iteration of the fairAUC procedure, the AUC value of the disadvantaged group $g^\star$ increases by at least
                        $$\max_{\ell}\frac{1}{4} \cdot \gamma^{\sfrac{3}{2}} \cdot \beta^2_\ell  \cdot (1-\delta_\ell)^2.$$
                        Here $\ell$ runs over unacquired features, $\gamma$ is the distance of the current AUC value of $g^\star$ from $1$, $\beta_\ell$ is the difference in the normalized class-conditional means of the $\ell$th unacquired feature on $g^\star$ (Equation~\eqref{eq:cond_on_aux_1}), and $\delta_\ell$ is the absolute value of the normalized class-conditional covariance between the {FLD-based} score $S$ and the $\ell$th unacquired feature on $g^\star$ (Equation~\eqref{eq:cond_on_aux_2}).
                      \end{proposition}

                      \noindent The theorem implies that the improvement is more when $\gamma$ and $\beta_\ell$ are large and $\delta_\ell$ is close to 0 for at least one unacquired feature, and less when $\gamma$ is close to 0 or for all features either $\beta_\ell$ is close to 0 or $\delta_\ell$ is close to 1.
                      To see why the improvement value may depend on these quantities, note that:
                      \begin{enumerate}
                        \item If $\gamma$ is close to 0, then the AUC of the disadvantaged group is close to 1, which is its maximum value, and hence, cannot increase significantly. In contrast, when we have data with features that are not as informative, then the potential improvement $\gamma$ and actual improvement are higher.
                        \item Intuitively, the difference between class-conditional means ($\beta_\ell$) helps create separation between the two classes, and including features with high separation improves AUC. If $\beta_\ell$ is close to 0, then Equation~\eqref{eq:AUC_binormal} tells us that the classifier using the $\ell$th unacquired feature to predict the outcome has a low AUC, and so the $\ell$th unacquired feature is not a good predictor and does not increase the AUC significantly.
                        \item The normalized covariance provides a measure of dependence between the new feature to be acquired and the score that summarizes existing features. This is related to mutual information, and since FLD is a linear discriminant, covariance provides a characterization of the mutual information. When $\delta_\ell$ is close to 1, then the $\ell$th unacquired feature is highly correlated with the score derived from the existing features, and so does not add additional information.
                      \end{enumerate}

                      \noindent  There are a few points to note.
                      First, while a specific feature $\ell$ might create separation for the disadvantaged group, it might also create such separation for the advantaged group. Recall that \fairAUC{} by design only focuses on improving the performance of the disadvantaged group. However, we will prove in Theorem \ref{thm:result_for_advantaged_group} in the Appendix similar bounds on the increase in AUC for the advantaged group.

                      Second, note that it is possible for the AUC of the currently disadvantaged group to exceed that of the advantaged group due to this feature acquisition. If that happens, the definition of disadvantaged group changes for the next round.%

                      {Third, the lower bound in Proposition 1 applies when \fairAUC{} uses FLD as the scoring rule. However, once the features are selected, the data manager can retrain a GLM of their choice, say logistic regression.
                      If the lower bound in \cref{thm:main_result_main_body} is $\Delta(t)$ at the $t$-th iteration, then the AUC of the logistic regression classifier trained by the data manager at the $t$-th iteration is at least addtiviely $\Delta(t)$ larger than the AUC of the logistic regression classifier from the previous iteration.\footnote{The same also holds for GLMs with link functions $\psi$ other than logit. This is because the link function does not affect the AUC of a GLM classifier – see \cref{claim:auc_indep_of_link}.}}

                      \subsection{fairAUC with Bias Guarantees: noisy fairAUC}\label{section:noisy_fairAUC_main}
                      The \fairAUC\ \algo\ can increase bias if the acquired feature is even more informative for the advantaged group or if the AUC of the disadvantaged group greatly overshoots the AUC of the advantaged group after the acquisition of a new feature.
                      More formally, suppose we have data $\mathbf{X}$ and we acquire feature $\mathbf{Z}$ so that our new data is $\mathbf{X}' \coloneqq (\mathbf{X},\mathbf{Z})$.
                      We can then have $\text{Bias}(\mathbf{X}')>\text{Bias}(\mathbf{X})$, where $\text{Bias}(\mathbf{X})$ and $\text{Bias}(\mathbf{X'})$ are the values of bias obtained using features $\mathbf{X}$ and $\mathbf{X'}$ respectively.
                      If our aim is to always decrease bias, one strategy is to alter \fairAUC\ by adding a noisy version of feature $\textbf{Z}$ instead. The noisy version of $\mathbf{Z}$ introduces noise to the group that would have higher AUC after adding $\textbf{Z}$. Specifically, for this group, we use $$\textbf{Z}' \coloneqq \lambda \textbf{Z} + (1-\lambda) N(0,1)$$ where $0\leq \lambda\leq 1$ indicates how much weight to place on $\textbf{Z}$.

                      The intuition for noisy fairAUC comes from the following idea: When a new feature is added, if we add pure noise to the feature (say a random normal variable of zero mean and some variance), then the AUC of the classifier with the data including the new feature is monotonically decreasing in the variance of the noise added. The intuition is more readily apparent if we consider an extreme case: consider adding pure noise as the new feature to the advantaged group $a$ (so $\lambda_a=0$), and therefore completely ignoring the newly acquired feature. In that case, the AUC of the advantaged group would remain the same. In contrast, for the disadvantaged group $b$, for any $0<\lambda_b<1$, the AUC would increase by the addition of the new feature. Given this disparate impact on the AUCs across the groups, bias would decrease during this round of feature acquisition.

                      More generally, we can apply this intuition to the less extreme case of adding noise to the new feature only for the advantaged group. By choosing an appropriate weight to place on the noise,
                      we obtain a provable guarantee that in noisy fairAUC, the bias between the two groups always (weakly) decreases during each round of feature acquisition (see \cref{section:noisy_auc_proof}). Moreover, the noisy fairAUC algorithm retains the desirable property of the original fairAUC algorithm that the AUC of both groups is also guaranteed to (weakly) increase in each round as features are acquired.
                      Thus, this approach of adding noise addresses the concern that bias could potentially increase in fairAUC. Note that this noisy fairAUC algorithm does not restrict the features in any way, or impose any additional assumptions beyond what is required for fairAUC.

                      \textit{What is the tradeoff then?}
                      Of course, the bias guarantee in noisy fairAUC will come at the cost of a reduction in the AUC improvement for the advantaged group (but it is important to note that it will never result in a \textit{decrease} in AUC for any group). By adding noise to the advantaged group’s acquired feature, we are reducing the AUC of the advantaged group relative to the fairAUC case, where no noise was added. We detail how the main theoretical result that all groups see an increase in AUC during a feature acquisition round continues to hold with noisy \fairAUC\ (see \cref{remark:noisyproof} in \cref{section:noisy_auc_proof}).
                      The above tradeoff positions noisy fairAUC at a different point on the fairness-accuracy curve relative to fairAUC. This positioning allows a manager to decide which algorithm to use based on how much they value the requirement that bias be non-increasing.

                      \section{Empirical Results}

                      We compare four feature acquisition strategies, namely \fairAUC, \maxAUC, \minBias, and \random. During each round of feature acquisition, the \fairAUC\ \algo\ selects the feature that most improves AUC for the group with lower AUC according to FLD. The \maxAUC\ \algo\ selects the feature that most improves the \textit{overall weighted AUC} using  FLD weighted by group size (see \cref{appendix:maxAUC} for the \maxAUC\ procedure). The \minBias\ \algo\ selects the feature that minimizes the bias between the two groups. The \random\ \algo\ selects a feature at random, and represents a baseline in which the \dm\ collects additional data in an uninformed manner.

                      \subsection{Synthetic Data}
                      We use the data generation strategy proposed by \cite{guyon2003design} for the controlled benchmarking of variable selection algorithms in binary classification problems. We generate $N=20,000$ individuals with 50 non-protected continuous binormally distributed features, one binary protected feature (group), and a binary outcome (class). Of the 50 features, half are \textit{informative} in that the class-conditional distributions of each of the features have means that are separated from each other. The remaining features are \textit{uninformative} random noise features. Group $a$ constitutes 70\% and group $b$ 30\%. The base rate of positive class labels in both groups is 25\%.

                      For the synthetic data, there is nothing fundamentally different between the two groups besides the number of individuals in each group so any difference in predictive performance stems from the feature acquisition procedure. %
                      We set the level of acceptable bias $\epsilon=10^{-6}$ to demonstrate the various \algos\ over many rounds. We acquire 10 additional features and use logistic regression for classification.
                      We randomly select one feature to represent the data the \dm\ begins with (Round 0) for classification. %

                      \subsubsection{Synthetic Data Results}
                      Figure \ref{figure:synthetic_results} (Top) compares the \fairAUC, \maxAUC, \minBias, and \random\ \algos\ in terms of group-wise AUCs when the protected attribute is used in the scoring function (i.e., each group has a separate classifier).
                      Under \fairAUC, the initially disadvantaged minority group $b$ quickly obtains predictive performance equal to majority group $a$. Under \maxAUC, group $b$'s AUC always trails group $a$'s AUC even though separate classifiers are trained for each group. The \minBias\ \algo\ quickly reduces bias and maintains low bias but fails to select informative features. \fairAUC\ and \maxAUC\ outperform the \random\ and \minBias\ \algos\ in terms of group-wise AUC.
                      Figure \ref{figure:rounds_sim_single_eqsep} in \cref{appendix:single_classifier} compares the \algos\ when the protected attribute is not used. The overall patterns among the \algos\ remain the same.

                      We graph the accuracy-fairness tradeoff (where accuracy is measured by AUC) in Figure \ref{figure:synthetic_results} (Center) that results from using \fairAUC\ rather than \maxAUC.
                      Ideally, a \algo\ generates points in the lower right of the graph, i.e., low bias and high AUC. The dotted lines connect the corresponding rounds between \fairAUC\ and \maxAUC. All of the lines have a positive slope, indicating that \fairAUC\ reduces bias but at the cost of overall AUC.
                      For \fairAUC, we observe that the bias does not monotonically decrease but rather jumps around. After all the rounds are complete, the \dm\ can evaluate the accuracy (AUC) versus bias tradeoff, according to their requirements. If the \dm\ requires a lower bias, they could choose the round that corresponds to the lowest level of bias (Round 7), whereas if they prefer to tradeoff a higher level of bias for a higher AUC, they might choose Round 10. The crucial aspect is that the feature acquisition procedure provides the \dm\ with a flexible set of options at various points on the accuracy-bias spectrum.
                      The \minBias\ \algo, as expected, produces low bias values but at the cost of significantly lower AUC. The \random\ \algo\ generates bias values between \fairAUC\ and \maxAUC\ but at far worse AUC values than either.
                      Figure \ref{figure:tradeoff_sim_single_eqsep} in \cref{appendix:single_classifier} graphs the tradeoff when the protected attribute A is not used and we observe the same patterns. %

                      Finally, we evaluate convex combinations of the fairness and maximum AUC objectives to generate a Pareto frontier for bias and overall AUC. Figure \ref{figure:synthetic_results} (Bottom) shows the intermediate bias and overall AUC values that can be achieved by altering the weight of the two objectives over different feature augmentation rounds when A is used.\footnote{For earlier rounds, many weight combinations select the same feature acquisition strategy, resulting in overlap.} Full weight on the fairness objective represents \fairAUC\ and full weight on the AUC objective represents \maxAUC.
                      The \dm\ therefore also has flexibility in determining how much weight to give to each of the objective functions. Figure \ref{figure:pareto_sim_single_eqsep} in \cref{appendix:single_classifier} graphs the Pareto frontier when A is not used and shows a similar pattern. \maxAUC\ generates far higher levels of bias when A is not used in classification.

                      \begin{figure*}
                        \begin{minipage}{\textwidth}
                          \begin{center}
                            \includegraphics[width=9.5cm]{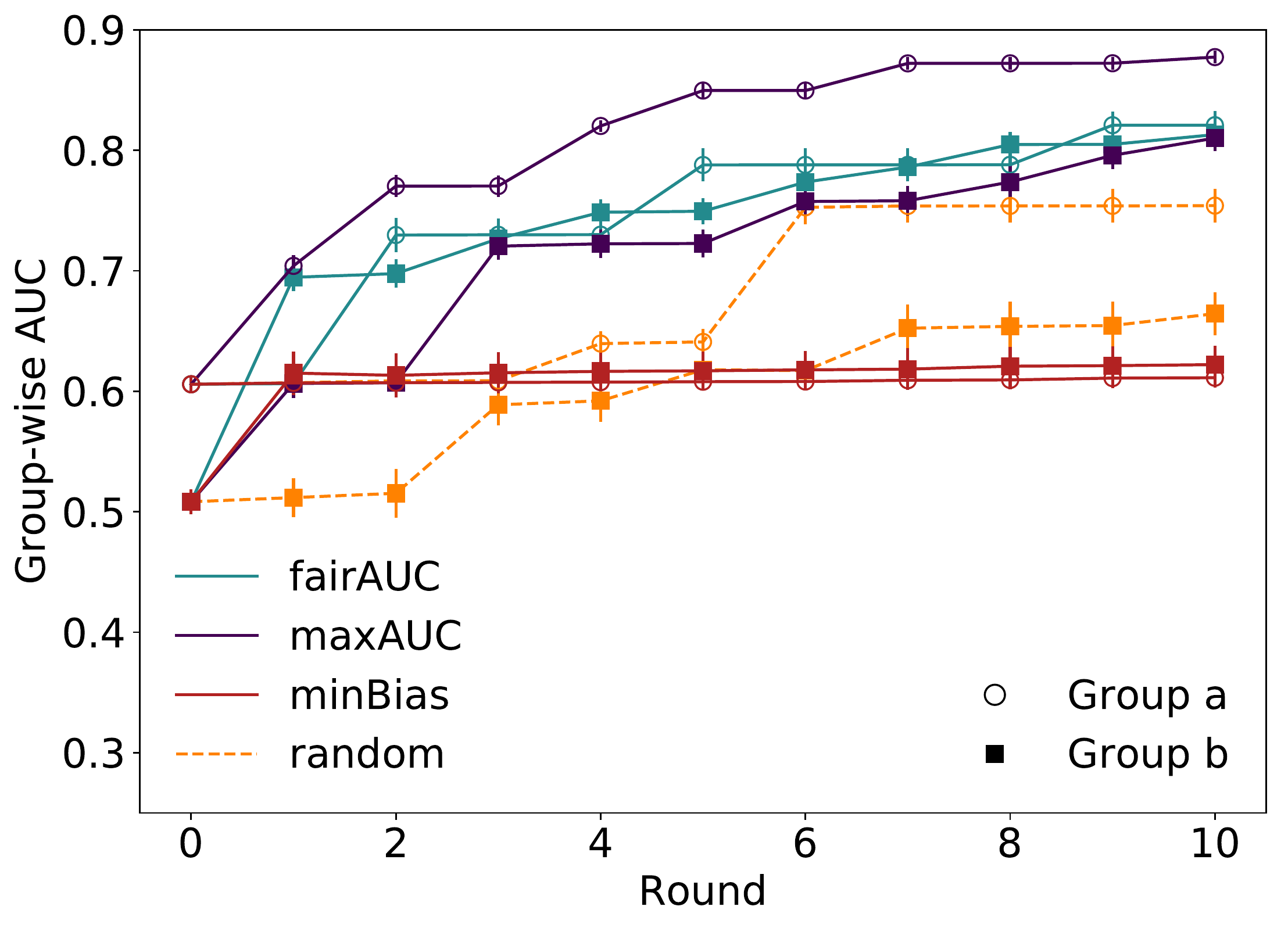}
                          \end{center}
                        \end{minipage}
                        \begin{minipage}{\textwidth}
                          \begin{center}
                            \includegraphics[width=9.5cm]{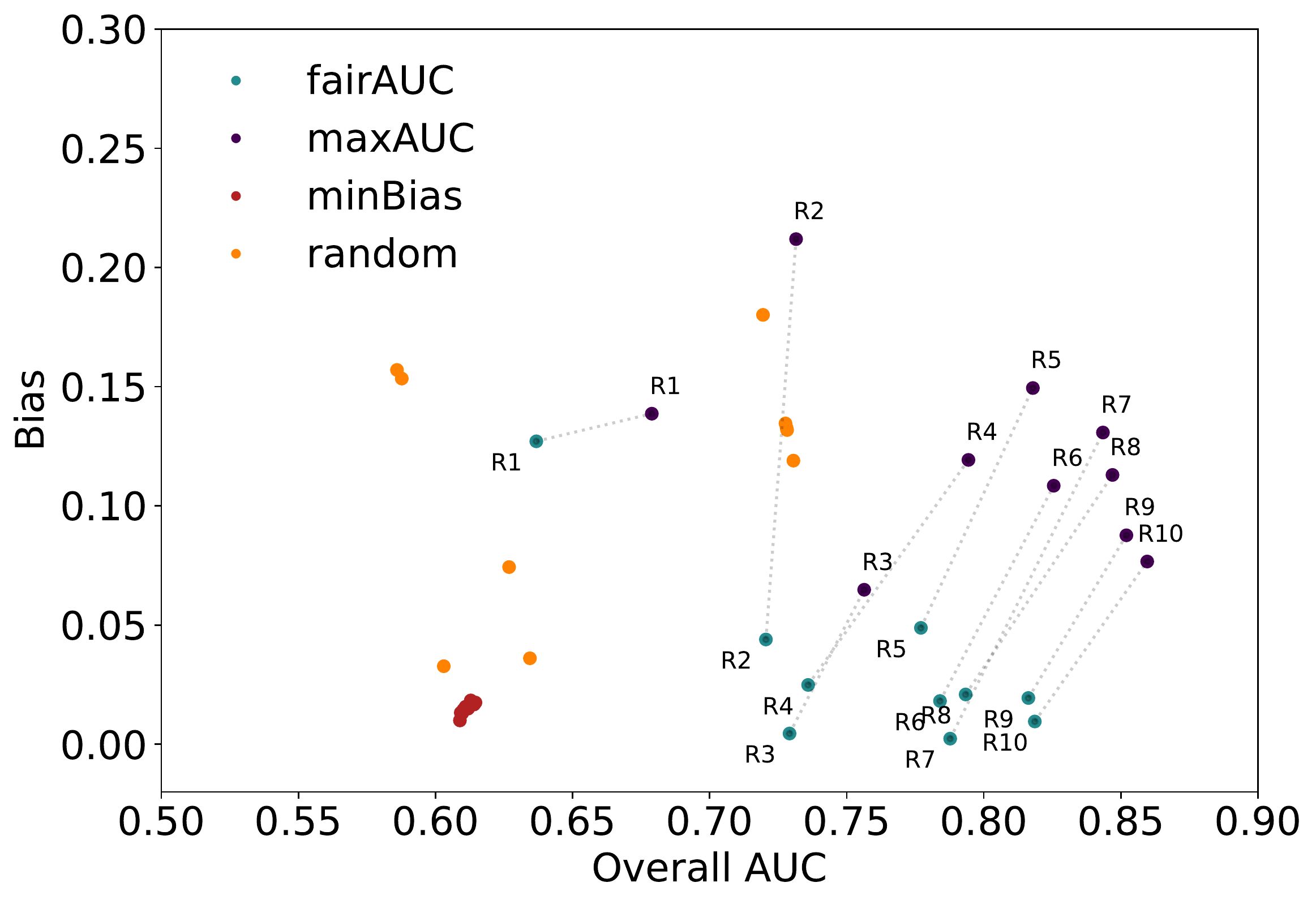}
                          \end{center}
                        \end{minipage}
                        \begin{minipage}{\textwidth}
                          \begin{center}
                            \includegraphics[width=9.5cm]{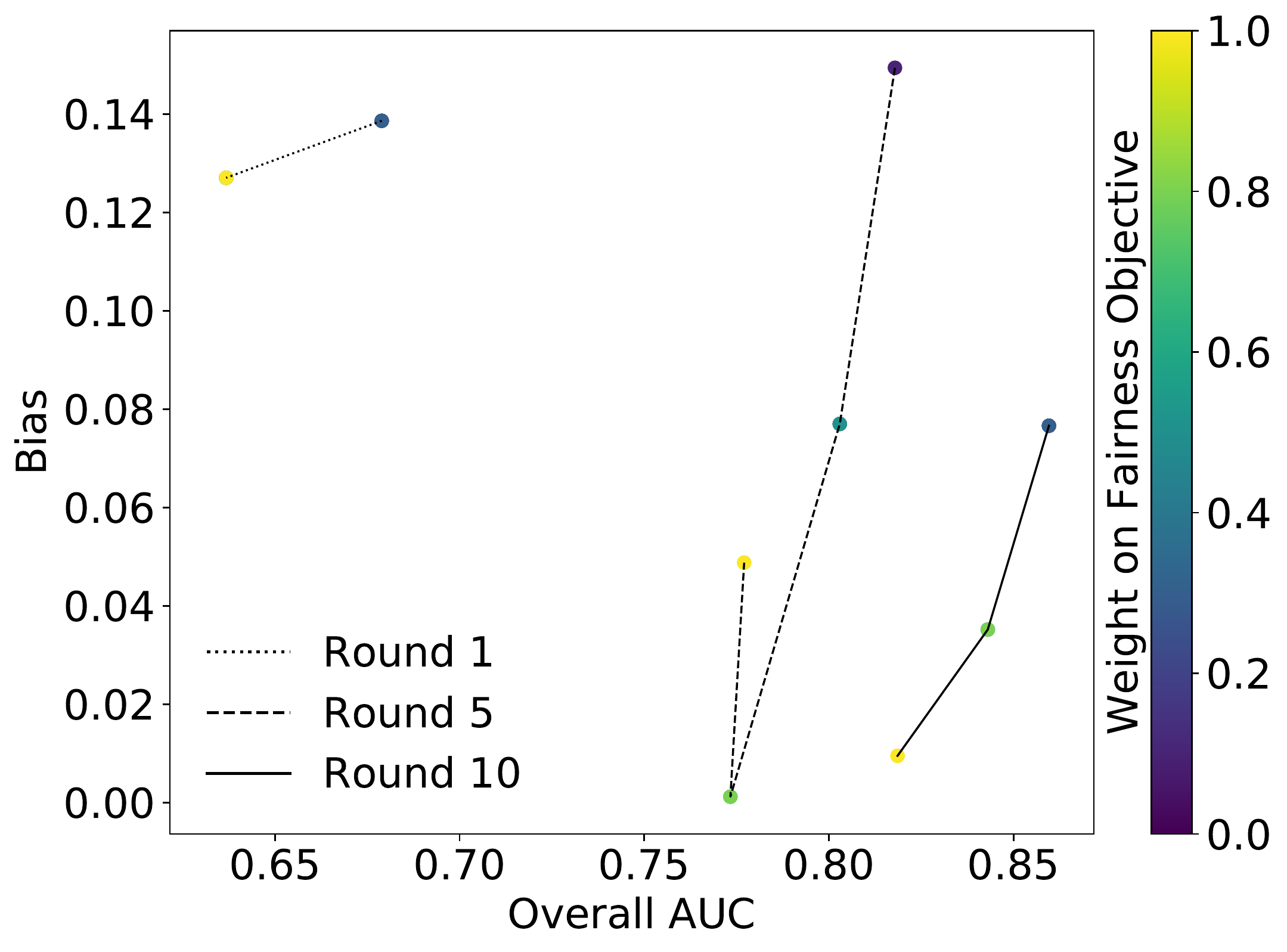}
                          \end{center}
                        \end{minipage}
                        \caption{ (Top) AUC by group over feature acquisition rounds using different feature acquisition strategies and using the protected attribute. The vertical error bars represent 95\% confidence intervals on the AUCs.
                        (Center) Comparison of accuracy-fairness tradeoff among feature acquisition strategies using the protected attribute.
                        (Bottom) Pareto frontier of convex combinations of the fairness and AUC objectives for several rounds of feature acquisition using the protected attribute.} \label{figure:synthetic_results}
                      \end{figure*}

                      The results highlight a number of pitfalls that can occur in data collection and prediction algorithm design. First, collecting data to maximize overall AUC or accuracy can inadvertently hurt the minority group. This can occur even when the two groups are equally separable and separate classifiers are trained for each group. %
                      Second, a strategy that aims to only minimize bias can result in the collection of features that are not predictive for either group.

                      \subsubsection{Robustness of fairAUC to Assumptions}\label{section:robustness_main_paper}

                      \fairAUC\ uses FLD as a heuristic to determine which feature to acquire. Recall that FLD is the linear classifier that maximizes AUC when the data is binormally distributed. We test the robustness of \fairAUC\ to other data generating processes, specifically using a different distribution which can accommodate a range of possible probabilities in \cref{appendix:other_data_distributions}. We find that the results are robust to using data that follow various gamma distributions. Next, we have specified the FLD framework and the associated theoretical results as applying to the class of Generalized Linear Models (GLMs). Therefore, we test the robustness of the results to using nonlinear classifiers like random forest and nonlinear SVM, in place of a linear classifier, logistic regression (\cref{section:nonlinear_classification}). Here too, we find similar results that our proposed \fairAUC\ \algo\ achieves much lower bias across groups, while obtaining overall AUC at a slightly lower to similar level as maxAUC. Finally, \fairAUC\ requires back and forth data exchange between the \dm\ and the data vendor to calculate covariances. One strategy to simplify the \algo\ even further is to assume independence between the auxiliary features and the firm's data. Table \ref{table:zero_corr} in \cref{appendix:zero_correlation} shows that out of the ten features acquired when correlations are accounted for, eight are acquired when the correlations are assumed to be zero.
                      Ignoring the class-conditional correlations generally results in the acquisition of the same features but in a less efficient order. We next apply the \fairAUC\ method to two real-world datasets from two distinct application areas, namely criminal justice and health care.

                      \subsection{Application: Predicting Violent Recidivism}\label{section:compas_data}

                      \subsubsection{COMPAS Dataset}
                      The dataset covers 6,172 criminal defendants from Broward County, Florida and contains information on their COMPAS score, demographics (gender, race, age), criminal history, and whether they actually recidivated within a two-year period after release. Our target variable of interest is violent recidivism and the protected attribute is age\footnote{Note that we define age as age at charge, which is different from the age recorded in the ProPublica dataset. ProPublica records defendants' age in 2016, the year the data was collected, rather than the age at charge. We calculate age at charge by subtracting date of birth from the date the defendant went to jail.} (under 25 vs. 25+). Those under 25 represent 33\% of the data and have a 14\% violent recidivism rate while those over 25 have a 10\% violent recidivism rate.

                      \subsubsection{Data Pre-processing}
                      We take log of the numerical variables in the dataset (e.g., number of priors) to reduce the impact of outliers. We also convert the categorical variables (e.g., race) into binary variables. We do not use the risk assessment levels or decile scores generated by COMPAS as inputs since they are the outcome variables, i.e. essentially what we seek to predict.

                      Suppose a judge has data on gender (initial independent variable), age group (protected attribute), and whether each defendant violently reoffended within two years of being released (outcome variable). Defendants under 25 years of age are the initially disadvantaged (lower-AUC) group based on the data the judge has. Given that features are costly to acquire, our focus is on which additional feature a judge should collect to better predict the likelihood of violent recidivism for defendants under 25.

                      \subsubsection{COMPAS Results}
                      We use logistic regression as the classifier and compare the \fairAUC, \maxAUC, \minBias, and \random\ procedures over ten rounds. Figure \ref{figure:compas_separate} plots the performance of the four \algos. The AUCs have fairly large error bars but the means follow the pattern seen in the synthetic data. \fairAUC\ improves the AUC for the group of defendants under 25 and decreases bias while \maxAUC\ does not close the gap between the two groups.

                      \begin{figure}
                        \centering
                        \caption{  \footnotesize Predicting Violent Recidivism using Protected Attribute (Age)}
                        \includegraphics[width=9.5cm]{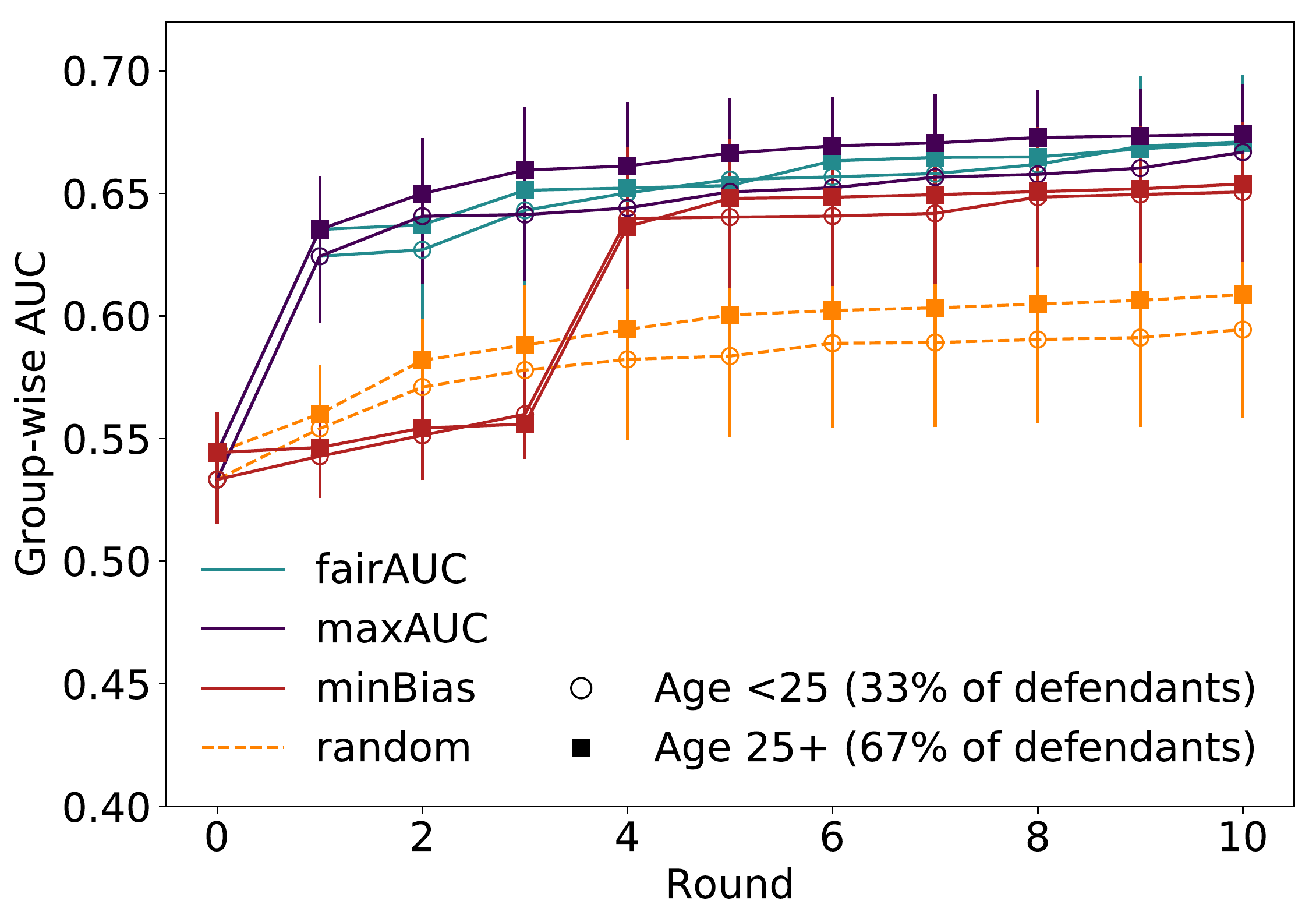}
                        \label{figure:compas_separate}
                      \end{figure}

                      \subsubsection{Acquiring Additional Features from a Data Vendor}\label{section:aspire_north}

                      In the previous section, we treat the features in the COMPAS dataset as features that can be acquired. In this section, we treat the COMPAS data as given and purchase additional features from a data vendor. %
                      Because the COMPAS dataset includes names and dates of birth for the defendants, it is possible to purchase additional features for these individuals. We note that the data available for purchase is current (2022) and not dated back to when the defendant was arrested (2013/2014), making it inappropriate for the actual prediction task at hand. The goal here is to demonstrate the practical feasibility of obtaining data from another source and combining it with first-party data.

                      We acquire data from Aspire North, a data vendor that works with small- to medium-sized businesses. The vendor sells a core set of roughly 550 features for \$80/1,000 individuals and charges more for specific features, such as ethnicity, net worth, and spending in different categories. Purchasing all available features exceeds \$15,000/1,000 individuals. See Table \ref{table:aspire_north_features} in \cref{appendix:data_vendor}  for additional details about the features.

                      We provide the data vendor the following information: name, birthday, zip codes in Broward county for the first pass of matching, and zip codes in Florida for the second pass of matching. Of the 6,172 defendants in the COMPAS dataset, the data vendor successfully matched 1,679 individuals. Some individuals may have changed their name, moved out of state, or it could be that the data vendor does not have data on all individuals.\footnote{When given name and address, Experian claims a match rate of  85\%.} Defendants under 25 represent 30\% of the data and have a 10.3\% violent recidivism rate while those over 25 have a 9.6\% violent recidivism rate. In addition to the core set of features, we purchase nine additional-fee features.

                      \subsubsection{Data Vendor Feature Acquisition Results}
                      To compare the performance of the four \algos, we begin with the COMPAS features as the initial data and acquire features from the purchased dataset. Where there are missing values in the acquired data, we replace the values with the means from each group. We use logistic regression as the classifier. Figure \ref{figure:compas_aspire_separate} plots the performance of the four algorithms. For this dataset, the initial data is more predictive of violent recidivism for the defendants under 25 years of age. Compared to \maxAUC, \fairAUC\ greatly reduces bias between the two groups. %

                      \begin{figure}[ht]
                        \centering
                        \caption{Predicting Violent Recidivism using Protected Attribute (Age) with Data Vendor Features}
                        \includegraphics[width=9.5cm]{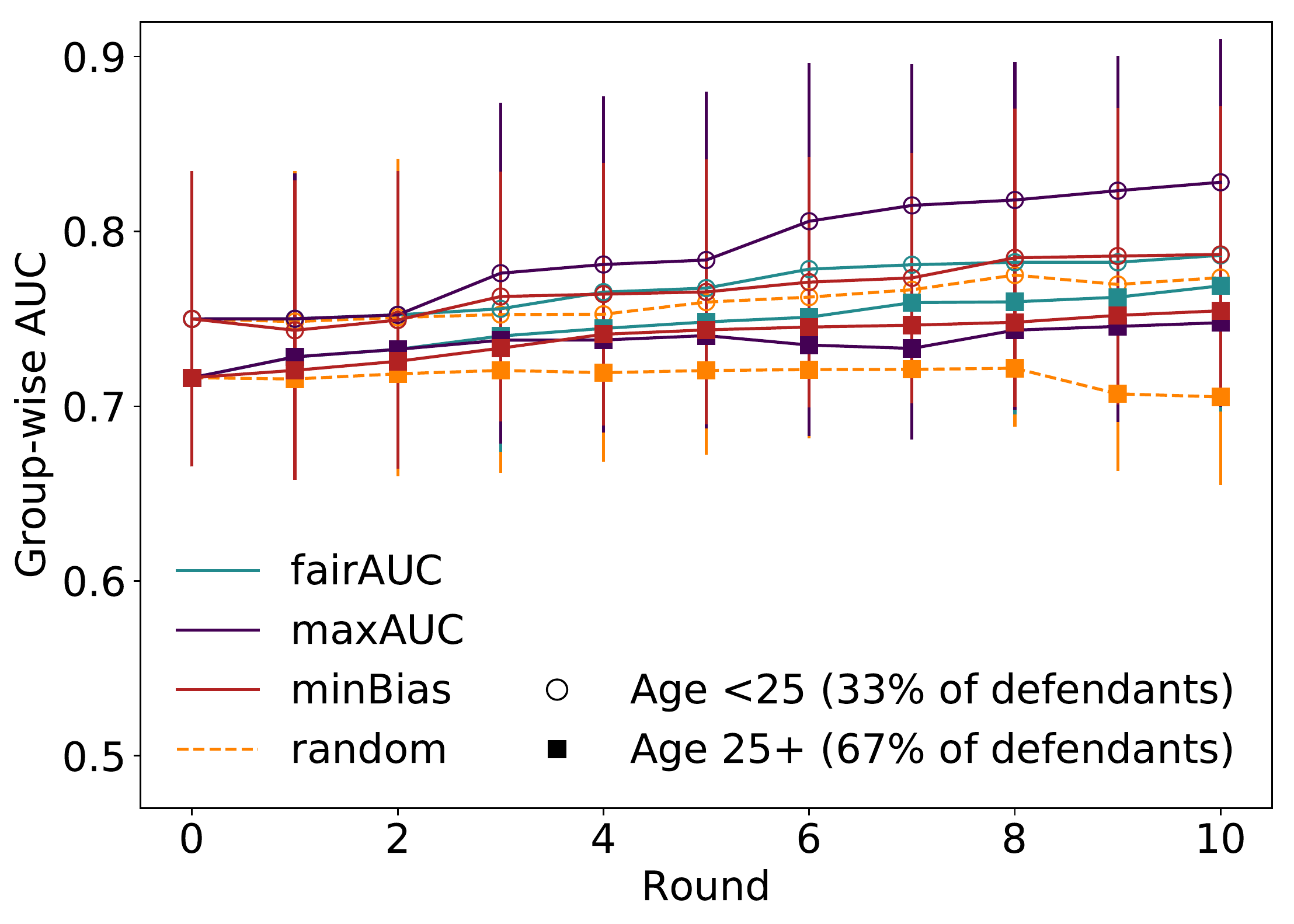}
                        \label{figure:compas_aspire_separate}
                      \end{figure}

                      \subsection{Application: Predicting Hospital Readmission}\label{section:diabetes_data}

                      Finally, we test the various procedures on a second completely distinct dataset in a healthcare application using the publicly available Diabetes dataset \citep{strackdiabetes}.\footnote{\url{https://archive.ics.uci.edu/ml/datasets/diabetes+130-us+hospitals+for+years+1999-2008}} The dependent variable examined is hospital readmission within 30 days, and the explanatory (predictor) variables include reason for admission, time in hospital, intervention with drugs, number of medications, and other health-related data, as well as demographic data on gender, race, and age. After cleaning the data, there are 45,715 observations. We take log of the numeric variables to reduce the impact of outliers and convert the categorical variables to indicator variables.

                      We find that when examining race as a protected group (77\% Caucasian, 23\% Non-Caucasian), our fairAUC procedure reduces bias, while obtaining a high AUC for both Caucasian and Non-Caucasian groups, compared to the maxAUC procedure, which obtains much higher AUC for the group of Caucasian individuals.

                      \begin{figure}[ht]
                        \centering
                        \caption{Predicting Hospital Readmission using Protected Attribute (Race)}
                        \includegraphics[width=9.5cm]{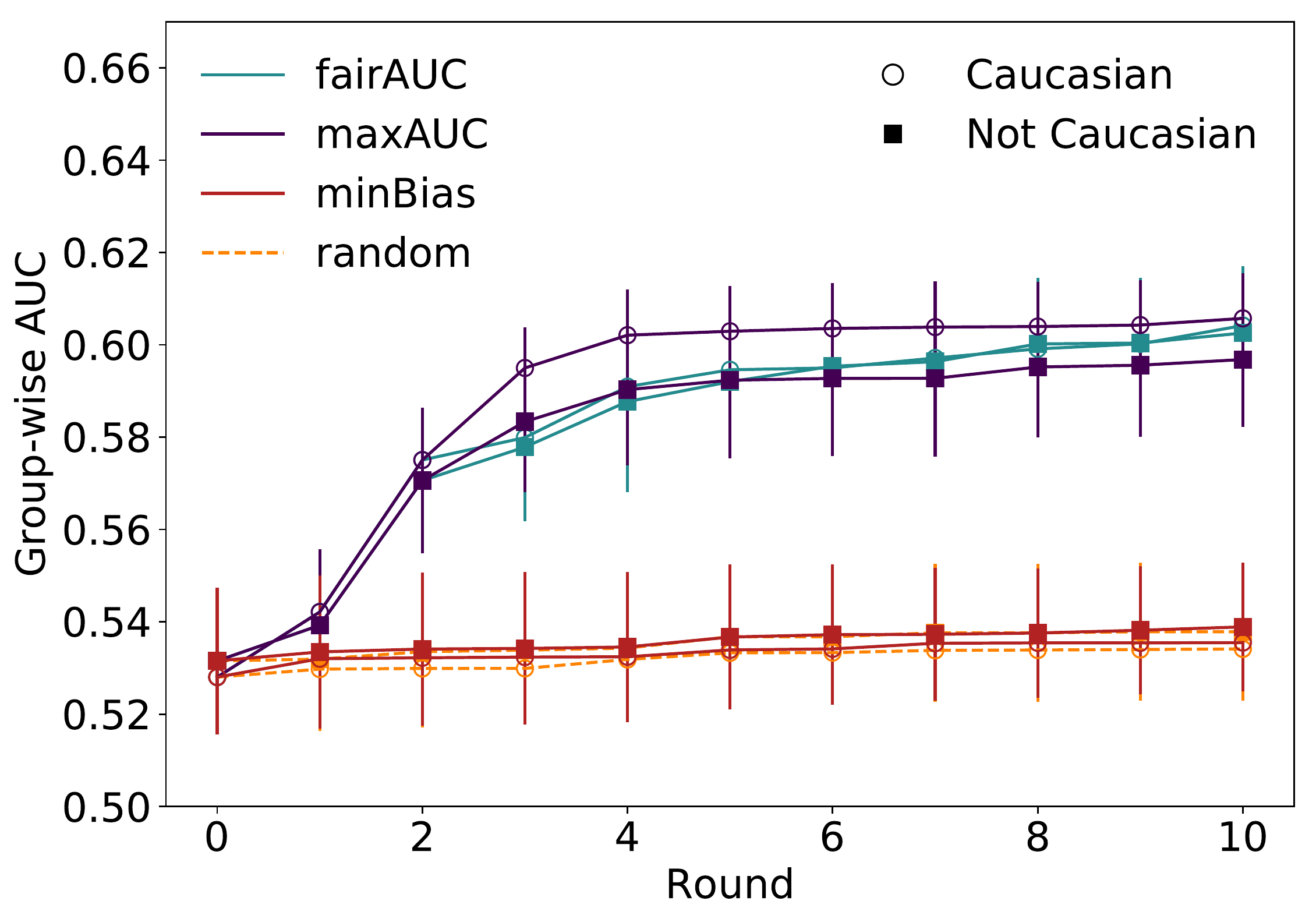}
                        \label{figure:diabetes}
                      \end{figure}

                      \noindent  Our real-world applications span across multiple domains, lending external validity to the broad range of applications where fairAUC can be useful.

                      \section{Conclusion}
                      We propose \fairAUC, an approach to feature acquisition that helps achieve fairness in the AUC measure. %
                      Our approach, which can incorporate a wide variety of classification algorithms, aims to improve the performance of the lower-AUC group. %
                      First, using a theoretical analysis we show provable AUC improvements for the disadvantaged group. Second, we test our approach using synthetic data as well as in real-world contexts and find that our approach performs well in reducing bias, while also increasing AUC for both the disadvantaged and advantaged groups.

                      While our method has many advantages, it is not without limitations. First, our method applies to cases with binary outcome labels, although in principle it could be extended to a multiclass classification problem. Second, if two ROC curves cross,  then one classifier performs better in one region of ROC space and the other classifier performs better in the other region of ROC space. Our approach would only consider the overall AUC. In practical situations, we might want to weight false positives and true positives differently or consider a notion of partial AUC.
                      In practice, the algorithm can be altered to account for such asymmetric weights.
                      Third, our \algo\ assumes the underlying data distributions are approximately binormal. While \fairAUC\ is meant to provide guidance as a heuristic, large deviations from normality may undermine its effectiveness. However, we do find that \fairAUC\ continues to perform well with different gamma distributions and in real-world data.

                      We trust that this paper is a first step in identifying and directly addressing fairness as it relates to the data collection process and AUC. This work complements other work that focuses on data collection through rows \citep{cai2022adaptive}, rather than features. We expect that more broadly these areas and their interaction will be further investigated in future research.

                      \paragraph{Acknowledgments.} This project is supported in part by an NSF Award (CCF-2112665).

                      \section{Proofs}

                      \subsection{Proof That AUC Can Be Non-monotone in Unconditional Variance}\label{section:uncond_var}

                      In this section, we formalize the idea that the unconditional variance does not inform AUC by writing the unconditional variance as a function of the conditional variances, where $\pi_g\coloneqq\Pr[Y=1|A=g]$ represents the proportion of observations from the positive class for group $g$.
                      The conditional variances, $\sigma_{g1}^2$ and $\sigma_{g0}^2$, and the unconditional variance, $\mathrm{Var}[X|A=g]$, can be written as:
                      \begin{eqnarray}\label{eq:cond_var_pos}
                        \sigma_{g1}^2 &=& \mathrm{Var}[X|A=g,Y=1] \nonumber \\
                        &=& \mathbb{E}[X^2|A=g,Y=1] - \mathbb{E}[X|A=g,Y=1]^2 \nonumber \\
                        &=& \int x^2p_{g1}(x)dx - \mu_{g1}^2,
                      \end{eqnarray}

                      \begin{eqnarray}\label{eq:cond_var_neg}
                        \sigma_{g0}^2 &=& \mathrm{Var}[X|A=g,Y=0] \nonumber \\
                        &=& \mathbb{E}[X^2|A=g,Y=0] - \mathbb{E}[X|A=g,Y=0]^2 \nonumber \\
                        &=& \int x^2p_{g0}(x)dx - \mu_{g0}^2,
                      \end{eqnarray}
                      and
                      \begin{align*}
                        \yesnum\label{eq:uncond_var}
                        \mathrm{Var}[X|A=g] &= \mathbb{E}[X^2|A=g] - \mathbb{E}[X|A=g]^2\\
                        &=\pi_g \int x^2p_{g1}(x)dx  +  (1-\pi_g)\int x^2p_{g0}(x)dx - (\pi_g\mu_{g1}+(1-\pi_g)\mu_{g0})^2.
                      \end{align*}

                      \noindent It follows from Equations \eqref{eq:cond_var_pos}, \eqref{eq:cond_var_neg}, and \eqref{eq:uncond_var} that:
                      \begin{eqnarray}\label{eq:uncond_var_result}
                        \mathrm{Var}[X|A=g] &=& \pi_g(1-\pi_g)(\mu_{g1} - \mu_{g0})^2  + \pi_g \sigma_{g1}^2 + (1-\pi_g) \sigma_{g0}^2.
                      \end{eqnarray}

                      \noindent When we hold the difference in class means and base rate constant, different combinations of $\sigma_{g0}^2$ and $\sigma_{g1}^2$ can produce the same unconditional variance in Equation \eqref{eq:uncond_var_result}. According to the binormal AUC formula (Equation \eqref{eq:AUC_binormal}), these combinations of $\sigma_{g0}^2$ and $\sigma_{g1}^2$ do not all map to the same AUC for group $g$. Indeed, the same unconditional variance can be mapped to multiple AUCs.
                      The left half of Table \ref{table:var_auc} (Constant $\mathrm{Var}[X|A=g]$) shows a numerical example of a single unconditional variance mapping to multiple AUCs for different conditional variances.
                      The right half of Table \ref{table:var_auc} (Increasing $\mathrm{Var}[X|A=g]$) shows that there is not a monotonic relationship between the unconditional variance and AUC.

                      \begin{table}[htbp]
                        \small
                        \centering
                        \caption{Unconditional Variance and AUC}
                        \vspace{1mm}
                        \begin{tabular}{rrrrrrrrr}
                          \hline \hline
                          \multicolumn{4}{c}{Constant $\mathrm{Var}[X|A=g]$}& & \multicolumn{4}{c}{Increasing $\mathrm{Var}[X|A=g]$} \\
                          \cline{1-4}
                          \cline{6-9}
                          $\sigma_{g0}^2$ & $\sigma_{g1}^2$ & Var & AUC & & $\sigma_{g0}^2$ & $\sigma_{g1}^2$ & Var & AUC \\
                          \hline
                          10 & 1 & 18.80 & 0.82 &  & 2 & 4 & 19.60 & 0.95 \\
                          4 & 2.5 & 18.80 & 0.94 & & 12 & 3 & 20.80 & 0.75 \\
                          2 & 3 & 18.80 & 0.98 & & 4 & 8 & 23.20 & 0.80 \\
                          \hline \hline
                          \multicolumn{9}{l}{\textit{Note}: $\pi_g = 0.8$ and $\mu_{g1} - \mu_{g0} = 10$.}
                        \end{tabular}
                        \label{table:var_auc}
                      \end{table}

                      \begin{observation}[Non-informativeness of Unconditional Variance]\label{obs:1}
                        The ranking of the unconditional variance between groups is not informative of the ranking of AUC between groups. For groups $a$ and $b$, if $\mathrm{Var}[X|A=a] > \mathrm{Var}[X|A=b]$, AUC$_a$ can be greater than, equal to, or less than AUC$_b$.
                      \end{observation}

                      \begin{proof}[Proof of Observation 1]
                        Set $\pi_a=\pi_b = \pi$, $\mu_{a1}=\mu_{b1}=\mu_{1}$, and $\mu_{a0}=\mu_{b0}=\mu_{0}$.
                        Then, using Equation \eqref{eq:uncond_var_result},
                        $\mathrm{Var}[X|A=a] > \mathrm{Var}[X|A=b]$ implies:
                        \begin{equation}\label{eq:proof_var_inequality}
                          \pi \sigma_{a1}^2 + (1-\pi) \sigma_{a0}^2 > \pi \sigma_{b1}^2 + (1-\pi) \sigma_{b0}^2.
                        \end{equation}
                        \noindent
                        Further, suppose that $\pi<0.5$ and $\mu_1\neq \mu_0$.
                        Consider the following two cases that demonstrate that $\mathrm{AUC}_a$ can be greater than or less than $\mathrm{AUC}_b$.

                        \noindent
                        \begin{enumerate}
                          \item Let $\sigma^2_{a0}=\sigma^2_{a1}=\sigma^2_{a}$ and $\sigma^2_{b0}=\sigma^2_{b1}=\sigma^2_{b}$. It follows from Equation \eqref{eq:proof_var_inequality} that $\sigma^2_{a}>\sigma^2_{b}$ so $\mathrm{AUC}_a = \Phi \left(\frac{\mu_{1}-\mu_{0}}{\sqrt{2\sigma_{a}^2}} \right) <  \Phi \left(\frac{\mu_{1}-\mu_{0}}{\sqrt{2\sigma_{b}^2}} \right) = \mathrm{AUC}_b$.
                          Here, we also use the fact that  $\mu_1\neq \mu_0$ and that $\Phi(\cdot)$ is a monotonically increasing function.
                          \item Let $\sigma^2_{a0}=\sigma^2_{a1}=\sigma^2_{a}$. It follows from Equation \eqref{eq:proof_var_inequality} that:
                          \begin{equation*}
                            \sigma^2_{a}>\pi \sigma^2_{b1}+(1-\pi)\sigma^2_{b0}.
                          \end{equation*}
                          \noindent
                          Let
                          \begin{equation}\label{eq:prop_case2}
                            \sigma^2_{a}=\pi \sigma^2_{b1}+(1-\pi)\sigma^2_{b0}+\epsilon
                          \end{equation}
                          where $\epsilon>0$. We want to find conditions under which $\mathrm{AUC}_a \geq \mathrm{AUC}_b$. It follows from Equation \eqref{eq:AUC_binormal} that $\mathrm{AUC}_a \geq \mathrm{AUC}_b$ when $2\sigma^2_{a} \leq \sigma^2_{b1}+\sigma^2_{b0}$ (since $\Phi(\cdot)$ is a monotonically increasing function).
                          Incorporating Equation \eqref{eq:prop_case2}, the AUC condition requires:
                          \begin{equation*}
                            \sigma^2_{b1}+\sigma^2_{b0} \geq 2 \pi \sigma^2_{b1} +2(1-\pi) \sigma^2_{b0} +2\epsilon,
                          \end{equation*}
                          \noindent
                          which simplifies to:
                          \begin{equation}\label{eq:prop_boundary}
                            \sigma^2_{b1} \geq \sigma^2_{b0} + \frac{2\epsilon }{1-2\pi}
                          \end{equation}
                          \noindent
                          when $\pi<0.5$.
                        \end{enumerate}

                        \noindent
                        Note that the smaller class needs to have higher variance for Equation \eqref{eq:prop_boundary} to hold. Class-conditional variances are weighted in the expected overall unconditional variance but not weighted in the AUC formula. The closer we are to class balance (i.e., $\pi=0.5$) the greater the difference in class-conditional variances we need for $\mathrm{AUC}_a \geq \mathrm{AUC}_b$.
                      \end{proof}

                      \subsection{Proof of \cref{thm:main_result_main_body}} \label{section:theoretical_results}
                      In this section, we prove \cref{thm:main_result_main_body}, which lower bounds the AUC-improvement for the disadvantaged group in each iteration of \fairAUC{}.
                      We also prove an analogous result for the advantage group (\cref{thm:result_for_advantaged_group}).

                      Toward this, we analyze the fairAUC procedure in the binormal framework for features~\citep{su1993linear} (where the features follow a normal distribution conditioned on the class and the protected group).
                      We show that if \fairAUC{} uses FLD-based scores $S$ (Equation~\eqref{eq:def_s_star}), then
                      in each iteration $t\in \N$, where the AUC for the disadvantaged group is bounded away from 1 and there is at least one auxiliary feature which has ``low'' class-conditional covariances with the current scores $S$ and has ``bounded'' class-conditional variances and means,
                      \fairAUC{} improves the AUC of the disadvantaged group by at least a constant in iteration $t$ (\cref{thm:main_result}).

                      From \cref{section:fair_auc_over_multiple_iterations}, recall that there are a total of $m$ features, out of which, the decision-maker initially has access to $d$ {\em acquired features}:
                      \begin{align*}
                        X\coloneqq (X^1,X^2,\dots,X^d)\in \R^d,\tag{Acquired features}
                      \end{align*}
                      and has the option to augment $d^\prime\coloneqq m-d$ {\em auxiliary features}:
                      \begin{align*}
                        \qquad\qquad\qquad\qquad\ \ \hspace{0.5mm}
                        Z \coloneqq (Z^1,Z^2,\dots,Z^{d^\prime})\in \R^{d^\prime}. \tag{Auxiliary features}
                      \end{align*}
                      The $m$ features $X\cup Z$, together with class label $Y$ and group $A$, are assumed to follow the following binormal framework in this section.
                      \begin{definition}[\bf Binormal framework]\label{def:binormal_framework}
                        The $m$ features $X\cup Z$, the class label $Y$, and the group label $A$ are distributed according to a distribution $\cD$ over $\R^d\times \R^{d^\prime} \times \zo\times \ab$, such that
                        for each $y\in\zo$ and $g\in\ab$, conditioned on $A=g$ and $Y=y$, the $m$ features, follow a $m$-variate normal distribution with an invertible covariance matrix.
                        {(Note that conditioned on $A=g$ and $Y=y$ different features can be correlated with each other.)}
                      \end{definition}
                      \noindent From the distribution $\cD$ (in \cref{def:binormal_framework}), $N\in \N$ independent samples are drawn to construct a dataset $D$ before starting the \fairAUC{} procedure.
                      The summary statistics subroutines (Subroutines SSR and SSR2) use $D$, every time they are queried, to compute the approximations to first and second moments of the distribution $\cD$;
                      we assume that these approximations have a negligible error (\cref{asmp:ssr}). %
                      \begin{assumption}\label{asmp:ssr}
                        Assume that the sample means and covariances computed by two summary statistics subroutines (Subroutines SSR and SSR2) are equal to the corresponding true means and covariances of draws from $\cD$.
                      \end{assumption}
                      \noindent Since the subroutines use independent samples from $\cD$, where the features follow a normal distribution, from the concentration inequalities of the normal distribution~\citep{tropp2015introduction}, we expect the samples means and covariances of the features on $D$ to be ``good approximations'' of the true means and covariances of the features on $\cD$ for large $N$.

                      At each iteration, \fairAUC{} acquires one auxiliary feature.
                      For each $t\in [d^\prime]$, let $Q(t)\subseteq [d^\prime]$ denote the set of all auxiliary features acquired before the start of the $t$-th iteration; %
                      where we have $Q(1)\coloneqq \emptyset$.
                      Further, let $X(t)$ denote the tuple of all features in $Q(t)$ and the $d$ features $(X^1,X^2,\dots,X^d)$, i.e.,
                      \begin{align*}
                        X(t)\coloneqq (X^1,X^2,\dots,X^d) \  \cup \ \inparen{Z^\ell}_{\ell\in Q(t)}.
                      \end{align*}
                      Note that because $Q(1)=\emptyset$, $X(1)=X$.

                      \begin{algorithm}[!t]
                        \SetAlgorithmName{Procedure}{procedure}{}
                        \caption{{\fairAUC{} ($t$-th iteration)}}
                        \small
                        \vspace{1mm}
                        \KwIn{Data owned $({X_i},A_i,Y_i)_{i=1}^N$, $[d^\prime]$, indices acquired $Q(t)\subseteq [d^\prime]$, and data acquired $\sinparen{Z^\ell_i}_{i\in [N],\ \ell\in Q(t)}$}
                        \vspace{1mm}
                        \KwOut{Set $Q(t+1)\subseteq [d^\prime]$ of the auxiliary features augmented}
                        \vspace{1mm}
                        \hrule
                        \vspace{2mm}

                        \For {group $g \in \{a,b\}$}{

                        {\bf Query} $\Sigma^{\sexp{g}}_0,\Sigma^{\sexp{g}}_1,\mu^{\sexp{g}}_0,\mu^{\sexp{g}}_1 = $ \texttt{SSR}($X\cup \sinparen{Z^\ell}_{\ell\in Q(t)},A,g$)\\
                        {\bf Compute} $\Delta\mu^\sexp{g} \coloneqq \sinparen{\sabs{\mu^{\sexp{g}}_{11}-\mu^{\sexp{g}}_{01}}, \dots, \sabs{\mu^{\sexp{g}}_{1d}-\mu^{\sexp{g}}_{0d}}}$\\
                        {\bf Compute} $\Sigma^\sexp{g}\coloneqq \Sigma^{\sexp{g}}_0+\Sigma^{\sexp{g}}_1$\\
                        {\bf Initialize} ${S} \coloneqq \inparen{0}_{i=1}^{N}$\\
                        \For {$i \in [N]$}{
                        \uIf{$A_i = a$}{
                        {\bf Set} ${S}_i \coloneqq {(\Delta\mu^\sexp{a})^\top (\Sigma^\sexp{a})^{-1} X_i}$ \hfill \textcolor{gray}{\footnotesize // Compute FLD scores}\\
                        }
                        \Else{
                        {\bf Set} ${S}_i \coloneqq {(\Delta\mu^\sexp{b})^\top (\Sigma^\sexp{b})^{-1} X_i}$ \hfill \textcolor{gray}{\footnotesize // Compute FLD scores}\\
                        }
                        }
                        }\vspace{3mm}

                        \For {group $g \in \{a,b\}$}{
                        {\bf Compute}
                        $\AUC_g(X)\coloneqq \Phi\inparen{\sqrt{ (\mu^{\sexp{g}}_1-\mu^{\sexp{g}}_0)^\top (\Sigma^{\sexp{g}}_0+\Sigma^{\sexp{g}}_1)^{-1} (\mu^{\sexp{g}}_1-\mu^{\sexp{g}}_0)}}$
                        }
                        $g(t)  \coloneqq \arg\min_{g\in \ab}(\AUC_g(X))$ \hfill \textcolor{gray}{\footnotesize // Find disadvantaged group} \vspace{3mm}

                        \For {auxiliary feature $\ell \in [d^\prime]$}{
                        \vspace{2mm}
                        \textcolor{gray}{// For group $g(t)$ query: class-conditional means ${\mu_{0}},{\mu_{1}}\in \R^{2}$, and }\\
                        \textcolor{gray}{// covariance matrices ${\Sigma_{0}},{\Sigma_{1}}\in \R^{2\times 2}$  between score $S$ and auxiliary feature $Z^\ell$.}\\
                        \vspace{2mm}
                        {\bf Query} $\Sigma_0, \Sigma_1, \mu_0, \mu_1=$ \texttt{SSR2}($\ell,g(t),S$)\\
                        {\bf Compute}
                        $\AUC_{g(t)}({S},{Z}^\ell) \coloneqq \Phi \left(\sqrt{({\mu_{1}}-{\mu_{0}})^\top ({\Sigma_0}+{\Sigma_1})^{-1}({\mu_{1}}-{\mu_{0}})} \right)$
                        }\vspace{3mm}

                        $i \coloneqq \arg\max_{\ell\in [d^\prime]} \AUC_{g(t)}({S},{Z}^\ell)$
                        $Q(t+1) = Q(t)\cup \inbrace{i}$\\\vspace{3mm}
                        {\bf return} $Q(t+1)$.
                        \vspace{3mm}
                      \end{algorithm}

                      \begin{algorithm}[!b]\label{subroutine:1}
                        \SetAlgorithmName{Subroutine}{subroutine}{list of subroutines}
                        \caption{{\texttt{SSR} (summary statistic subroutine)}}
                        \small
                        \vspace{1mm}
                        \KwIn{Acquired features $\sinbrace{X_i^j}_{i\in [N], j\in [d+t]}$, protected attributes $\inbrace{A_i}_{i=1}^N$, group $g\in \ab$}
                        \vspace{1mm}
                        \KwOut{class-conditional mean vectors ${\mu_{0}},{\mu_{1}}\in \R^d$, class-conditional covariance matrices ${\Sigma_{0}},{\Sigma_{1}}\in \R^{d\times d}$}
                        \vspace{1mm}\hrule
                        \vspace{2mm}
                        \For {class $y \in \{0,1\}$}{
                        \vspace{2mm}
                        {\bf Compute} $n \coloneqq \sum_{i}\mathbb{I}[A_i=g, Y_i=y]$ \hfill \textcolor{gray}{\footnotesize // Total elements with $A_i=g$ and $Y_i=y$}\\\vspace{1mm}
                        {\bf Compute} ${\mu_y} \coloneqq
                        \frac{1}{n}\begin{bmatrix}
                        \sum_{i:
                        \begin{smallmatrix}
                          A_i=g, Y_i=y
                          \end{smallmatrix}}X_i^1,
                          \dots,
                          \sum_{i:
                          \begin{smallmatrix}
                            A_i=g, Y_i=y
                            \end{smallmatrix}}X_i^d
                          \end{bmatrix}$  \hfill \textcolor{gray}{\footnotesize // Empirical mean of $X$ when $(A_i,Y_i)=(g,y)$}\\\vspace{1mm}
                          {\bf Compute} matrix ${\Sigma_y}\in \R^{d\times d}$, where for all $\ell,k\in [d]$ \hfill \textcolor{gray}{\footnotesize // Empirical covariance of $X$ when $(A_i,Y_i)=(g,y)$}
                          $$(\Sigma_y)_{\ell,k}\coloneqq \frac{1}{n-1} \sum_{i:\begin{smallmatrix}
                          A_i=g, Y_i=y
                          \end{smallmatrix}}(X_i^\ell-(\mu_y)_\ell)(X_i^k-(\mu_y)_k)$$
                          }

                          \vspace{3mm}
                          {\bf return} ${\mu_{0}},{\mu_{1}},{\Sigma_{0}},{\Sigma_{1}}$
                          \vspace{3mm}
                        \end{algorithm}

                        \begin{algorithm}[!t]
                          \SetAlgorithmName{Subroutine}{subroutine}{list of subroutines}
                          \caption{{\texttt{SSR2} (Summary statistic subroutine - 2) } \label{subroutine:2}}
                          \small
                          \vspace{1mm}
                          \KwIn{auxiliary feature index $\ell\in [d^\prime]$, group $g$, score $\inbrace{S_i}_{i=1}^N$ (Also, has access to all auxiliary features $\inbrace{Z_i}_{i=1}^N$)}
                          \vspace{1mm}
                          \KwOut{Class-conditional mean vectors ${\mu_{0}},{\mu_{1}}\in \R^2$, class-conditional covariance matrices ${\Sigma_{0}},{\Sigma_{1}}\in \R^{2\times 2}$}
                          \vspace{1mm}\hrule
                          \vspace{2mm}
                          \For {class $y \in \{0,1\}$}{
                          {\bf Compute} $n \coloneqq \sum_{i}\mathbb{I}[A_i=g, Y_i=y]$ \hfill \textcolor{gray}{\footnotesize // Total elements with $A_i=g$ and $Y_i=y$}\\\vspace{1mm}
                          {\bf Compute} ${\mu_{S,y}} \coloneqq \frac{1}{n} \sum_{i:
                          \begin{smallmatrix}
                            A_i=g, Y_i=y
                          \end{smallmatrix}}S_i$\\\vspace{1mm}
                          {\bf Compute} ${\mu_{Z,y}} \coloneqq \frac{1}{n} \sum_{i:
                          \begin{smallmatrix}
                            A_i=g, Y_i=y
                          \end{smallmatrix}}Z_i^\ell$\\\vspace{1mm}

                          {\bf Compute} ${\Sigma_y} \coloneqq
                          \begin{bmatrix}
                            \sigma_{S,y}^2 & \rho_y \\
                            \rho_y & \sigma_{Z,y}^2
                            \end{bmatrix}$  where
                            \noindent
                            \begin{align*}
                              \textstyle  \sigma_{S,y}^2 &\coloneqq
                              \frac{1}{n-1} \sum\nolimits_{i:
                              \begin{smallmatrix}
                                A_i=g, Y_i=y
                                \end{smallmatrix}}(S_i-\mu_{S,y})^2,\\
                                \textstyle  \sigma_{Z,y}^2 &\coloneqq
                                \frac{1}{n-1}  \sum\nolimits_{i:
                                \begin{smallmatrix}
                                  A_i=g,  Y_i=y
                                  \end{smallmatrix}}(Z_i^\ell - \mu_{Z,y})^2, \text{ and}\\
                                  \textstyle          \rho_y &\coloneqq
                                  \frac{1}{n-1} \sum\nolimits_{i:
                                  \begin{smallmatrix}
                                    A_i=g, Y_i=y
                                    \end{smallmatrix}}(S_i-\mu_{S,y})\cdot (Z_i^\ell -\mu_{Z,y})
                                  \end{align*}
                                  }
                                  \vspace{2mm}
                                  {\bf return} $\mu_{0},\mu_{1},\Sigma_{0},\Sigma_{1}$
                                  \vspace{2mm}
                                \end{algorithm}

                                We need to define the AUC of $X(t)$ for a group $g\in \ab$ (\cref{def:auc_2}) before stating our results.
                                Note that
                                $X(t)$ is always of the form $X\cup \sinbrace{Z}_{\ell\in Q(t)}$, i.e., $k\geq 0$ auxiliary features augmented to $X$.  %
                                We restrict our definition of the AUC (\cref{def:auc_2}) to such sets of features.
                                \begin{definition}[\bf AUC of \red{generalized linear models} on group $g$]\label{def:auc_1}
                                  Given $k\geq 0$ auxiliary features, say $Z^1,Z^2,\dots,Z^k$, acquired features $X\in \R^{d}$, a vector $w\in \R^{d+k}$ \red{and an increasing and invertible link function $\psi\colon \R\to \R$},
                                  consider a classifier $C$ that given threshold $\tau\in \R$, predicts
                                  $$\mathbb{I}\insquare{\red{\psi^{-1}}\inparen{\sum_{i=1}^{d} w_i X^i + \sum_{i=1}^k w_{d+i} Z^i} > \tau}.$$
                                  {Then the AUC of $C$ for group $g\in \ab$, denoted by} $$\AUC_{g}(w,\red{\psi,} X,Z^1,\dots,Z^k),$$ {is the area under the ROC curve of $C$ when samples $\sinparen{(X,Z),Y,A}$ are drawn from $\cD$ conditioned on $A=g$.}
                                \end{definition}
                                \begin{definition}[\bf AUC for group $g$]\label{def:auc_2}
                                  Given $k\geq 0$  auxiliary features, say $Z^1,Z^2,\dots,Z^k$, acquired features $X\in \R^d$, and a group $g\in \ab$,
                                  define the AUC of $(X,Z^1,\dots,Z^k)$ for group $g$ as
                                  \begin{align*}
                                    \AUC_{g}(X,Z^1,\dots,Z^k) \ \ &\coloneqq\ \
                                    \red{\max_{
                                    \substack{w,\psi}
                                    }}  \
                                    \AUC_{g}(\red{w,\psi},X,Z^1,\dots,Z^k),
                                  \end{align*}
                                  \red{where the maximum is over all $w\in \R^{d+k}$ and all increasing and invertible functions $\psi\colon \R\to \R$.}
                                \end{definition}

                                \noindent
                                Using \cref{def:auc_2}, we can formalize the ``FLD-based'' score that \fairAUC{} uses in this section.
                                For all samples in group $g\in \ab$ (i.e., $i\in[N]$, with $A=g$), define the scores $S(t)\in \R$ used in \fairAUC{} as the projection of $X(t)$ that maximizes the AUC of \red{the best generalized linear model} on group $g$ (see \cref{def:auc_1}):
                                \begin{align*}
                                  S(t)\coloneqq \red{\inparen{\psi^\star}^{-1}\inparen{\sinangle{w^\star, X(t)}}},
                                  \ \text{where } \red{w^\star,\psi^\star \coloneqq \argmax\nolimits_{w,\psi}}\  \AUC_{g(t)}(w,\red{\psi,} X(t)).\yesnum\label{eq:def_s_star}
                                \end{align*}
                                One can show that $S(t)$ is equivalent to the projection obtained using FLD on each group; this uses the fact that the data follows the binormal framework (\cref{def:binormal_framework}; see \cite{su1993linear}).

                                \noindent Now, we restrict out attention to a particular iteration $t\in \N$. %
                                We define certain quantities that show up in our results (\cref{thm:main_result}).
                                Suppose $g(t)\in \ab$ is the disadvantaged group in the $t$-th iteration.
                                For each auxiliary feature $\ell\in [d^\prime]\backslash Q(t)$, let $\Delta v_\ell^\sexp{t}$ {be the absolute difference of its class conditional means (on $g(t)$), i.e.,}
                                \begin{align*}
                                  \Delta v_\ell^\sexp{t} &\coloneqq \abs{\Ex[Z^\ell\mid Y=1, A=g(t)] - \Ex[Z^\ell\mid Y=0, A=g(t)]}.\yesnum\label{eq:delta_v_ell}
                                \end{align*}
                                Next, using $\Delta v_\ell^\sexp{t}$ define the following quantities for each auxiliary feature $\ell\in [d^\prime]\backslash Q(t)$
                                \begin{align*}
                                  \beta_\ell^\sexp{t} &\coloneqq \frac{\Delta v_\ell^\sexp{t} }{\sqrt{\sum\nolimits_{y\in \zo}\Var[Z^\ell\mid Y=y,A=g(t)]}},
                                  \yesnum\label{eq:cond_on_aux_1}\\
                                  \delta_\ell^\sexp{t} &\coloneqq \frac{1}{\Delta v_\ell^\sexp{t} }\cdot\abs{ {\textstyle\sum\nolimits_{y\in \zo}}\Cov[S(t), Z^\ell \mid Y=y, A=g(t)]}.
                                  \label{eq:cond_on_aux_2}\yesnum
                                \end{align*}
                                Finally, define $\gamma^\sexp{t}$ as
                                \begin{align*}
                                  \gamma^\sexp{t} \coloneqq 1 - \AUC_{g(t)}(X(t)).
                                \end{align*}
                                We prove \cref{thm:main_result}.

                                \begin{theorem}[\bf Effect of \fairAUC{} on the AUC of the disadvantaged group]\label{thm:main_result}
                                  Suppose that the $m$ features $X\cup Z$, class label $Y$, and protected group $A$ follow the binormal framework (\cref{def:binormal_framework}).
                                  Further, assume that two summary statistics subroutines satisfy \cref{asmp:ssr}.
                                  Then, for all iterations $t\in[d^\prime]$ and all auxiliary features $\ell\in [d^\prime]\backslash Q(t)$, it holds that
                                  \begin{align*}
                                    \AUC_{g(t)}(S(t),Z^{\ell}) - \AUC_{g(t)}(X(t))
                                    &\ \ >\ \ %
                                    {\frac{1}{4} \cdot \inparen{\gamma^\sexp{t}}^{\sfrac{3}{2}}} \cdot \inparen{{\beta_\ell}^\sexp{t}\cdot(1-\delta_\ell)^\sexp{t}}^2.
                                    \tag{AUC increment on selecting $Z^\ell$}
                                  \end{align*}
                                  Further, the auxiliary feature $i \in [d^\prime]\backslash Q(t)$ selected by \fairAUC{} in the $t$-th iteration satisfies
                                  \begin{align*}
                                    \AUC_{g(t)}(X(t),Z^{i}) - \AUC_{g(t)}(X(t))
                                    & \geq   \max_{\ell\in [d^\prime]\backslash Q(t)}
                                    {\frac{1}{4} \cdot \inparen{\gamma^\sexp{t}}^{\sfrac{3}{2}}} \cdot \inparen{{\beta_\ell}^\sexp{t}\cdot(1-\delta_\ell)^\sexp{t}}^2.
                                    \tagnum{AUC increment by \fairAUC{}}\customlabel{eq:guarantee_2}{\theequation}
                                  \end{align*}
                                \end{theorem}
                                \noindent Some remarks are in order:
                                \begin{enumerate}
                                  \item {\bf (Dependence on $\gamma^\sexp{t}$).} As $\AUC_{g(t)}(X(t))$ approaches 1 (i.e., $\gamma^\sexp{t}$ approaches 0), the lower bound in Equation~\eqref{eq:guarantee_2} approaches 0.
                                  This is expected because when $\AUC_{g(t)}(X(t))$ is close to 1, which is its maximum value, each auxiliary feature can only increment the AUC for $g(t)$ by a small amount.
                                  \item {\bf (Dependence on $\beta_\ell^\sexp{t}$).} If $\Delta\sabs{v_{\ell}^\sexp{t}}$ is small or $\sum\nolimits_{y\in \zo}\Var[Z^\ell\mid Y=y,A=g(t)]$ is large,
                                  then Equation~\eqref{eq:AUC_binormal} tells us the classifier which uses $Z^\ell$ to predict the class $Y$ has a low AUC, i.e., $Z^\ell$ is not a ``good predictor'' of $Y$.
                                  This is captured by $\beta_\ell^\sexp{t}$ in Equation~\eqref{eq:guarantee_2}.
                                  To see this, observe that when $\Delta\sabs{v_{\ell}^\sexp{t}}$ is small or $\sum\nolimits_{y\in \zo}\Var[Z^\ell\mid Y=y,A=g(t)]$ is large, $\beta_\ell^\sexp{t}$ is small.
                                  Thus, the increment in the AUC is also small.
                                  \item {\bf (Dependence on $\delta_\ell^\sexp{t}$).}
                                  To gain some intuition about the dependence on $\delta_\ell^\sexp{t}$, consider the extreme case, where $Z^\ell$ is identical to $S(t)$.
                                  This maximizes the class-conditional covariances of $Z^\ell$ and $S(t)$ (on $A=g(t)$) subject to a fixed value of variance of $Z^\ell$.
                                  Thus, it also maximizes $\delta_\ell^\sexp{t}$.
                                  However, in this case, any linear combination of $X(t)$ and $Z^
                                  \ell$ is identical to some linear combination of $X(t)$ (and vice-versa).\footnote{This uses the fact that $S(t)$ is a linear combination of $X(t)$. Since $Z^\ell$ is identical to $S(t)$, $Z^\ell$ is also  linear combination of $X(t)$.}
                                  Thus, $\AUC_{g(t)}(X(t),Z^\ell)=\AUC_{g(t)}(X(t)).$
                                  Intuitively, $Z^\ell$ does not provide any new information.
                                \end{enumerate}

                                \noindent \cref{thm:main_result} shows the effect of \fairAUC{} on the AUC of the disadvantaged group. %
                                Our next result (\cref{thm:result_for_advantaged_group}), captures the effect of \fairAUC{} on the AUC of the advantaged group.
                                Suppose $\hat{g}(t)\in \ab$ is the advantaged group at the $t$-th iteration.
                                \cref{thm:result_for_advantaged_group} provides a lower bound in the improvement on the AUC of the advantaged group in the $t$-th iteration in terms of quantities
                                $\Delta \hat{v}_\ell^\sexp{t}$, $\hat{\beta}_\ell^\sexp{t}$, $\hat{\delta}_\ell^\sexp{t}$, and $\hat{\gamma}^\sexp{t}$
                                (\cref{eq:hat_v,eq:hat_beta,eq:hat_delta,eq:hat_gamma});
                                these are equivalent to $\Delta {v}_\ell^\sexp{t}$, ${\beta}_\ell^\sexp{t}$, ${\delta}_\ell^\sexp{t}$, and ${\gamma}^\sexp{t}$ in \cref{thm:main_result}, except the disadvantaged group $g(t)$ in the definitions changes to the advantaged group $\hat{g}(t)$.

                                Formally, we define $\Delta \hat{v}_\ell^\sexp{t}$, $\hat{\beta}_\ell^\sexp{t}$, $\hat{\delta}_\ell^\sexp{t}$, and $\hat{\gamma}^\sexp{t}$ as follows.
                                \begin{align*}
                                  \Delta \hat{v}_\ell^\sexp{t} &\coloneqq \abs{\Ex[Z^\ell\mid Y=1, A=\hat{g}(t)] - \Ex[Z^\ell\mid Y=0, A=\hat{g}(t)]},\yesnum \label{eq:hat_v}\\
                                  \hat{\beta}_\ell^\sexp{t} &\coloneqq \frac{\Delta v_\ell^\sexp{t} }{\sqrt{\sum\nolimits_{y\in \zo}\Var[Z^\ell\mid Y=y,A=\hat{g}(t)]}},
                                  \yesnum\label{eq:hat_beta}\\
                                  \hat{\delta}_\ell^\sexp{t} &\coloneqq \frac{1}{\Delta v_\ell^\sexp{t} }\cdot\abs{ {\textstyle\sum\nolimits_{y\in \zo}}\Cov[S(t), Z^\ell \mid Y=y, A=\hat{g}(t)]},
                                  \label{eq:hat_delta}\yesnum\\
                                  \hat{\gamma}^\sexp{t} &\coloneqq 1 - \AUC_{\hat{g}(t)}(X(t)).\yesnum\label{eq:hat_gamma}
                                \end{align*}
                                \begin{theorem}[\bf Effect of \fairAUC{} on the AUC of the advantaged group]\label{thm:result_for_advantaged_group}
                                  Suppose that the $m$ features $X\cup Z$, class label $Y$, and protected group $A$ follow the binormal framework (\cref{def:binormal_framework}).
                                  Further, assume that two summary statistics subroutines satisfy \cref{asmp:ssr}.
                                  Then, for all iterations $t\in[d^\prime]$ %
                                  the auxiliary feature $i \in [d^\prime]\backslash Q(t)$ selected by \fairAUC{} in the $t$-th iteration satisfies
                                  \begin{align*}
                                    \AUC_{\hat{g}(t)}(X(t),Z^{i}) - \AUC_{\hat{g}(t)}(X(t))
                                    &\ \geq \
                                    {\frac{1}{4} \cdot \inparen{\hat{\gamma}^\sexp{t}}^{\sfrac{3}{2}}} \cdot \inparen{{\hat{\beta}_\ell}^\sexp{t}\cdot(1-\hat{\delta}_\ell)^\sexp{t}}^2.
                                    \tagnum{AUC increment by \fairAUC{}}\customlabel{eq:guarantee_thm2}{\theequation}
                                  \end{align*}
                                \end{theorem}
                                \noindent At a first glance, the lower bound in \cref{thm:result_for_advantaged_group} may appear to be equivalent to \cref{thm:main_result}.
                                The difference is that \fairAUC{} is guaranteed to pick the best feature for the disadvantaged group, but it may not pick the best feature for the minority group.
                                Thus, while the lower bound in \cref{thm:main_result} is large if any auxiliary feature $\ell\in [d^\prime]\backslash Q(t)$ has
                                large $\beta_\ell^\sexp{t}$ and small $\delta_\ell^\sexp{t}$, \cref{thm:result_for_advantaged_group}
                                requires $\beta_i^\sexp{t}$ to be large and $\delta_i^\sexp{t}$ to be small for the particular feature $i\in [d^\prime]\backslash Q(t)$ selected by \fairAUC{}.

                                \subsubsection{Preliminaries}\label{section:preliminaries}
                                In this section, we present three lemmas which will be used in the proof of \cref{thm:main_result}.
                                \begin{lemma}[\bf Expression for optimal AUC \protect{\cite[Corollary 3.1]{su1993linear}}]\label{fact:optimal_auc_expression}
                                  Consider two random variables $X\in \R^d$ and $Y\in \zo$, which are distributed according to a joint distribution $\cD$, such that for all $y\in \zo$, conditioned on $Y=y$, $X$ follows a multivariate normal distribution with mean $\mu_y\in \R^d$ and covariance matrix $\Sigma_y\in \R^{d\times d}$:
                                  \begin{align*}
                                    \text{for all $y\in \zo$,}\quad X\mid Y=y &\ \ \sim\ \ \cN(\mu_y,\Sigma_y).
                                  \end{align*}
                                  Let $\Delta\mu\coloneqq \abs{\mu_1-\mu_0}$ and $\Sigma\coloneqq \Sigma_0+\Sigma_1.$
                                  Then, the maximum AUC of a {generalized linear model} that takes $X$ as input and predicts $Y$ is $\Phi\inparen{\sqrt{\Delta\mu \Sigma^{-1} \Delta\mu}}$.
                                \end{lemma}
                                \begin{proof}{Proof of \cref{fact:optimal_auc_expression}}
                                  Corollary 3.1 in \cite{su1993linear} proves that the maximum  AUC of any linear classifier that takes $X$ as input and predicts $Y$ is $\Phi\inparen{\sqrt{\Delta\mu \Sigma^{-1} \Delta\mu}}$.

                                  This also extends to generalized linear models because the AUC of a generalized linear model is independent of its link function.
                                  In particular, given a generalized linear model, we can choose its link function as the identity function without changing its AUC.
                                  This converts the generalized linear model to a linear classifier, for which the result follows by Corollary 3.1 in \cite{su1993linear}.
                                  Let $G_{w,\psi}$ be the generalized linear model with weights $w\in \R^{d}$ and increasing and invertible link function $\psi\colon \R\to \R$.
                                  \begin{claim}[\textbf{AUC of a generalized linear model is independent of its link function}]\label{claim:auc_indep_of_link}
                                    For any $w\in \R^d$, there is a value $v\in [0,1]$ such that, for any increasing and invertible link function $\psi\colon \R\to \R$, the AUC of $G_{w,\psi}$ is $v$.
                                  \end{claim}
                                  \begin{proof}{Proof of \cref{claim:auc_indep_of_link}}
                                    Let $I\colon \R\to \R$ be the identity function (i.e., for all $x\in \R$, $I(x)\coloneqq x$).
                                    (For simplicity, we assume that the domain and range of $\psi$ are $\R$. The same proof holds for general case by limiting the threshold $t$ to the domain or range of $\psi$ as appropriate.)
                                    For any value $t\in \R$, $G_{w,\psi}$ with threshold $t$ makes the same predictions as $G_{w,I}$ with threshold $\psi(t)$:
                                    This holds because for any increasing and invertible function $\psi$
                                    \begin{align*}
                                      w^\top X > \psi\inparen{\tau}
                                      \quad\text{if and only if}\quad
                                      \psi^{-1}\inparen{w^\top X} > \tau.
                                      \yesnum\label{eq:transf_of_thresh}
                                    \end{align*}
                                    (The ``if'' follows by applying $\psi$ to both sides of $\psi^{-1}\inparen{w^\top X} > \tau$ and using the fact that $\psi$ is increasing.
                                    ``Only if'' follows by applying $\psi^{-1}$ to both sides of $w^\top X > \psi\inparen{\tau}$ and using that $\psi^{-1}$ is increasing.)

                                    Let $\text{TPR}_\psi(t)$ and $\text{TPR}_\psi(t)$ be the true positive rate and the false positive rate of $G_{w,\psi}$ at threshold $t$.
                                    Similarly, let $\text{TPR}_I(t)$ and $\text{TPR}_I(t)$ be the true positive rate and the false positive rate of $G_{w,I}$ at threshold $t$.
                                    From the above observation (\cref{eq:transf_of_thresh}) it follows that for all thresholds $t\in\R$
                                    \begin{align*}
                                      \text{TPR}_\psi(t) = \text{TPR}_I(\psi(t))
                                      \quad\text{and}\quad
                                      \text{FPR}_\psi^{-1}(t) = \psi^{-1}\inparen{\text{FPR}_I^{-1}(\psi(t))}.
                                    \end{align*}
                                    Now the claim follows by using the definition of AUC in \cref{eq:AUC_general}:
                                    \begin{align*}
                                      \text{AUC of $G_{w,\psi}$}
                                      &= \int_{0}^1 \text{TPR}_\psi\inparen{\text{FPR}_\psi^{-1}(z)}dz\\
                                      &= \int_{0}^1 \text{TPR}_I\inparen{\psi\inparen{\psi^{-1}\inparen{\text{FPR}_I^{-1}(z)}}}dz\\
                                      &= \int_{0}^1 \text{TPR}_I\inparen{\text{FPR}_I^{-1}(z)}dz\\
                                      &=\text{AUC of $G_{w,I}$}.
                                    \end{align*}
                                  \end{proof}
                                  \noindent  This completes the proof of \cref{claim:auc_indep_of_link} and also of \cref{fact:optimal_auc_expression}.

                                \end{proof}
                                \begin{lemma}[\bf ``Inverting'' $\Phi(\sqrt{\cdot})$]\label{fact:upper_bound_alpha}
                                  For all $\alpha\geq 0$ and $\gamma>0$, if $\Phi\inparen{\sqrt{\alpha}}<1-\gamma$ then it holds that {$$\alpha< -2\cdot  \ln\inparen{2\gamma}.$$}
                                \end{lemma}
                                \begin{proof}{Proof of \cref{fact:upper_bound_alpha}}
                                  We use the fact that for all $x\in\R$, the inequality $\Phi(x)\geq 1-{\frac{1}{2}\cdot} e^{-x^2/2}$ holds ({see e.g., \cite[Equation 2.24]{conc_inequality_notes}}).
                                  Applying this, we get
                                  \begin{align*}
                                    \Phi\inparen{\sqrt{\alpha}} \geq 1-{\frac{1}{2}\cdot} e^{-\sfrac{\alpha}{2}}.
                                  \end{align*}
                                  Chaining the above inequality with $1-\gamma > \Phi\inparen{\sqrt{\alpha}}$ and rearranging, we get {$\alpha< 2\ln\inparen{\frac{1}{2\gamma}}=-2\cdot\ln\inparen{2\gamma}.$}
                                \end{proof}
                                \begin{lemma}[\bf Lower bound on change in $\Phi(\sqrt{\cdot})$]\label{fact:lower_bound_normal}
                                  For all $0<\gamma\leq 1$, $\Delta_0>0$,
                                  {$\alpha\in (0,-2\cdot\ln\inparen{2\gamma})$,} and $\Delta\geq \Delta_0$, it holds that
                                  \begin{align*}
                                    \Phi(\sqrt{\alpha+\Delta}) - \Phi(\sqrt{\alpha}) \geq
                                    {\frac{2}{\sqrt{\pi}}\cdot
                                    \frac{\gamma^{\frac{3}2} \cdot \Delta_0 }{  \sqrt{\frac{2}{e}+\Delta_0}\cdot (2+\Delta_0)}.}
                                    \yesnum\label{eq:lowerbound_in_fact_3}
                                  \end{align*}
                                \end{lemma}
                                \noindent We note that the bound in \cref{fact:lower_bound_normal} weakens as $\Delta_0$ (and so, $\Delta$) increases.
                                To see this, observe that the LHS in Equation~\eqref{eq:lowerbound_in_fact_3} is an increasing function of $\Delta$.
                                In contrast, if $\Delta_0$ is large enough, the RHS in Equation~\eqref{eq:lowerbound_in_fact_3} is a decreasing function of $\Delta_0$.
                                Nevertheless, \cref{fact:lower_bound_normal} suffices to prove \cref{thm:main_result}.

                                The proof of \cref{fact:lower_bound_normal} appears in \cref{section:proofof:fact:lower_bound_normal}.

                                \subsubsection{Proof of \cref{thm:main_result}}
                                In this section, we present a proof of \cref{thm:main_result}.
                                We begin with the necessary notation and the lemmas (\cref{fact:information_inequality,fact:lower_bound_sz}) in \cref{section:supp_lemmas}.
                                Next, in \cref{section:proofof:thm:main_result_main}, we complete the proof of \cref{thm:main_result} assuming \cref{fact:information_inequality,fact:lower_bound_sz}.
                                Finally, in \cref{section:proofof:supporting_lemmas}, we present the proofs of \cref{fact:information_inequality,fact:lower_bound_sz} respectively.

                                Recall that we are given distribution $\cD$ which satisfies \cref{def:binormal_framework}.
                                We assume that the statistics returned by the two summary statistic subroutines are exact (\cref{asmp:ssr}).

                                Fix any iteration $t\in [d^\prime]$.
                                Our goals are to prove that for each auxiliary feature $\ell\in [d^\prime]\backslash Q(t)$,
                                $\AUC_{g(t)}(S(t),Z^{\ell}) - \AUC_{g(t)}(X)\geq
                                {\frac{1}{4} \cdot \inparen{\gamma^\sexp{t}}^{\sfrac{3}{2}}} \cdot \inparen{{\beta_\ell}^\sexp{t}\cdot(1-\delta_\ell)^\sexp{t}}^2,$
                                and that, in this iteration, \fairAUC{} improves the AUC for the current disadvantaged group $g(t)$, by at least
                                \begin{align*}
                                  \AUC_{g(t)}(X,Z^{i}) - \AUC_{g(t)}(X)
                                  &\ \ >\ \  \max\nolimits_{\ell\in [d^\prime]}
                                  {\frac{1}{4} \cdot \inparen{\gamma^\sexp{t}}^{\sfrac{3}{2}}} \cdot \inparen{{\beta_\ell}^\sexp{t}\cdot(1-\delta_\ell)^\sexp{t}}^2.
                                \end{align*}

                                \noindent Fix any auxiliary feature $\ell\in [d^\prime]$. %
                                Then the proof proceeds in two broad steps.
                                First, we show that $\AUC(X,Z^\ell)$ is lower bounded by $\AUC(S(t),Z^\ell)$ (\cref{fact:information_inequality}).
                                Then, we derive an explicit formula and lower bound for $\AUC(S(t),Z^\ell)$ (\cref{fact:lower_bound_sz}).
                                This formula is the same as the formula used to compute $\AUC_{g(t)}(S(t),Z^\ell)$ in \fairAUC{}.
                                Thus, \fairAUC{} selects the auxiliary feature $i$ where $$i\in \argmax\nolimits_{\ell\in [d^\prime]}\AUC_{g(t)}(S(t),Z^\ell).$$
                                Combining this with a lower bound $\AUC(S(t),Z^\ell)$ for any $\ell\in [d^\prime]$, we get that the auxiliary feature $i$ selected by \fairAUC{}, satisfies
                                $$\AUC_{g(t)}(S(t),Z^i) > \max\nolimits_{\ell\in [d^\prime]}
                                {\frac{1}{4} \cdot \inparen{\gamma^\sexp{t}}^{\sfrac{3}{2}}} \cdot \inparen{{\beta_\ell}^\sexp{t}\cdot(1-\delta_\ell)^\sexp{t}}^2.
                                $$
                                Then, using \cref{fact:information_inequality},
                                we get that the auxiliary feature selected by \fairAUC{} improves the AUC for $g(t)$ by at least $\max\nolimits_{\ell\in [d^\prime]}{\frac{1}{4} \cdot \inparen{\gamma^\sexp{t}}^{\sfrac{3}{2}}} \cdot \inparen{{\beta_\ell}^\sexp{t}\cdot(1-\delta_\ell)^\sexp{t}}^2$.
                                Finally since the choice of $t$ was arbitrary, we get that the result holds for all $t\in [d^\prime]$.

                                \subsubsection{Additional Notation and Supporting Lemmas}
                                \label{section:supp_lemmas}
                                We begin by presenting the two lemmas used in the proof of \cref{thm:main_result}. %

                                \begin{lemma}[\bf Projection does not increase AUC]\label{fact:information_inequality}
                                  Consider three random variables $X\in \R^d$, $Y\in \zo$, $Z\in \R$, which follow some joint distribution $\cD$.
                                  Given $w\in \R^d$, define $S\coloneqq \inangle{w,X},$ it holds that
                                  \begin{align*}
                                    \AUC_\cD(X,Z)\geq \AUC_\cD(S,Z).
                                  \end{align*}
                                \end{lemma}
                                \noindent The proof of \cref{fact:information_inequality} appears in \cref{section:proofof:supporting_lemmas}.

                                We require some additional notation to present \cref{fact:lower_bound_sz}.
                                Fix any iteration $t\in [d^\prime]$.
                                Since $t$ will be fixed for the remainder of the proof, we drop the superscripts from $\gamma^\sexp{t}$, $\Delta v_\ell^\sexp{t}$, $\beta_\ell^\sexp{t}$, and $\delta_\ell^\sexp{t}$.
                                From \cref{def:binormal_framework}, we know that conditioned on $A$ and $Y$, $X\cup Z$ follows an $m$-variate normal distribution.
                                It follows that $X(t)$ also has a Gaussian distribution conditioned on $Y$ and $A$ (see e.g., \cite[Theorem 5, Section 8.4]{mult_variate_Gaussian}).
                                Suppose for all $y\in \zo$
                                \begin{align*}
                                  \qquad\qquad\quad X(t)\mid Y=y,A=g(t) &\ \ \sim\ \ \cN(\mu_y,\Sigma_y),\tagnum{Binormality of acquired features}\customlabel{eq:dist_condition_1}{\theequation}
                                \end{align*}
                                and for all $y\in \zo$ and $\ell\in[d^\prime]\backslash Q(t)$,
                                \begin{align*}
                                  \qquad\qquad\quad\ \ \ \ Z^\ell \mid Y=y,A=g(t) &\ \ \sim\ \ \cN(v_{y\ell},\sigma_{y\ell}^2),\tagnum{Binormality of auxiliary features}\customlabel{eq:dist_condition_2}{\theequation}
                                \end{align*}
                                where $\mu_0,\mu_1\in \R^{d+t}$, $\Sigma_0,\Sigma_1\in \R^{(d+t)\times (d+t)}$, and for all $\ell\in[d^\prime]\backslash Q(t)$, $v_{0\ell},v_{1\ell}\in\R$ and $\sigma^2_{0\ell},\sigma^2_{1\ell}\geq 0.$
                                Note that $\sinbrace{Z^\ell}_{\ell\in[d^\prime]\backslash Q(t)}$ can be correlated with each other (and with $X(t)$).

                                Next, we show that $\Sigma_0+\Sigma_1$ is invertible.
                                Towards this, notice that \cref{def:binormal_framework} requires that for any $y\in \zo$ and $g\in \ab$, conditioned on $Y=y$ and $A=g$ the covariance matrix of all $m$ features, $M_{yg}$, is invertible.
                                Since covariance matrices are positive semi-definite (PSD) and any invertible PSD matrix is positive definite (PD), it follows that $M_{yg}$ is PD.
                                Notice that $\Sigma_0$ and $\Sigma_1$ are submatrices of $M_{0g}$ and $M_{1g}$.
                                Since submatrices of PD matrices are also PD, it follows that $\Sigma_0$ and $\Sigma_1$ are PD.
                                Then, $\Sigma_0+\Sigma_1$ is PD.
                                Thus, $\Sigma_0+\Sigma_1$ is invertible.

                                Define $\Delta \mu$ and $\Delta \mu_S$ to be the difference in class-conditional means of $X(t)$ and $S$ respectively
                                \begin{align*}
                                  \Delta \mu &\coloneqq \abs{\mu_1-\mu_0},\yesnum\label{def:delta_mu}\\
                                  \Delta\mu_S &\coloneqq \abs{\Ex[S(t)\mid Y=1,A=g(t)] - \Ex[S(t) \mid Y=0,A=g(t)]}.
                                \end{align*}
                                Similarly, for all $\ell \in [d^\prime]$, define $\Delta v_\ell$ to be the difference in class-conditional means of $Z^\ell$
                                \begin{align*}
                                  \Delta v_\ell &\coloneqq \abs{v_{1\ell} - v_{0\ell}}.\yesnum\label{eq:def_of_delta_v}
                                \end{align*}
                                For all $y\in \zo$, define $\sigma_{yS}^2$ to be the variance of $S(t)$ conditioned on $Y=y$:
                                \begin{align*}
                                  \sigma_{yS}^2 \coloneqq \Var[S(t) \mid Y=y, A=g(t)].
                                \end{align*}
                                Finally, for all $y\in \zo$ and $\ell\in [d^\prime]$, define $\rho_{y\ell}$ as the covariance of $S(t)$ and $Z^\ell$,
                                \begin{align*}
                                  \rho_{y\ell} \coloneqq \Cov[S(t), Z^\ell\mid Y=y,A=g(t)].
                                \end{align*}
                                Using the definition of $S(t)$ (Equation~\eqref{eq:def_s_star}), we can compute $\Delta\mu_S$ in terms of $\Delta\mu$.
                                \begin{align*}
                                  \Delta\mu_S \ \ &\coloneqq\ \  \abs{\Ex[S(t)\mid Y=1,A=g(t)] - \Ex[S(t)\mid Y=0,A=g(t)]}\\
                                  \ \ &\Stackrel{\eqref{eq:def_s_star}}{=}\ \  \abs{\Ex\insquare{\Delta\mu^\top (\Sigma_0+\Sigma_1)^{-1} X(t)\mid Y=1,A=g(t)} - \Ex\insquare{\Delta\mu^\top (\Sigma_0+\Sigma_1)^{-1} X(t)\mid Y=0,A=g(t)}}\\
                                  \ \ &=\ \  \abs{\Delta\mu^\top (\Sigma_0+\Sigma_1)^{-1} \inparen{\Ex\insquare{X(t) \mid Y=1,A=g(t)} - \Ex\insquare{X(t) \mid Y=1,A=g(t)}}}\\
                                  \ \ &\Stackrel{\eqref{def:delta_mu}}{=}\ \  \abs{\Delta\mu^\top (\Sigma_0+\Sigma_1)^{-1} \Delta\mu}\\
                                  \ \ &\Stackrel{}{=}\ \  {\Delta\mu^\top (\Sigma_0+\Sigma_1)^{-1} \Delta\mu}. \tagnum{Using that $(\Sigma_0+\Sigma_1)^{-1}$ is a PD matrix}\customlabel{eq:value_of_mus}{\theequation}
                                \end{align*}
                                Similarly, for all $y\in \zo$, we can also compute $\sigma_{yS}^2$.
                                For $y=1$, we have
                                \begin{align*}
                                  \sigma_{1S}^2 \ \ &\coloneqq\ \ \Var[S(t) \mid Y=1,A=g(t)]\\
                                  \ \ &\Stackrel{\eqref{eq:def_s_star}}{=}\ \ \Var\insquare{\Delta\mu^\top (\Sigma_0+\Sigma_1)^{-1} X(t)\mid Y=1,A=g(t)}\\
                                  \ \ &=\ \ \Delta\mu^\top (\Sigma_0+\Sigma_1)^{-1} \Sigma_1 (\Sigma_0+\Sigma_1) \Delta\mu.\yesnum\label{eq:value_of_ss1}
                                \end{align*}
                                In the last equality, we use the fact that for any vector $w\in \R^{d+t}$ and random variable $X(t)\in \R^{d+t}$ with covariance matrix $\Sigma\in \R^{(d+t)\times (d+t)}$, it holds that $\Var[\inangle{w,X(t)}]=w^\top \Sigma w$.
                                Similarly, for $y=0$ we have that
                                \begin{align*}
                                  \sigma_{0S}^2 \quad&\coloneqq\quad \Delta\mu^\top (\Sigma_0+\Sigma_1)^{-1} \Sigma_0 (\Sigma_0+\Sigma_1) \Delta\mu.\yesnum\label{eq:value_of_ss2}
                                \end{align*}
                                Combining Equations~\eqref{eq:value_of_ss1} and \eqref{eq:value_of_ss2}, we get
                                \begin{align*}
                                  \sigma_{0S}^2+\sigma_{1S}^2\qquad     &
                                  \Stackrel{\eqref{eq:value_of_ss1},\eqref{eq:value_of_ss2}}{=}
                                  \qquad \Delta\mu^\top (\Sigma_0+\Sigma_1)^{-1}\Delta\mu.\yesnum\label{eq:value_of_ss}
                                \end{align*}

                                \begin{lemma}\label{fact:lower_bound_sz}
                                  If $\cD$ satisfies Equations~\eqref{eq:dist_condition_1} and \eqref{eq:dist_condition_2}, then it holds that
                                  \begin{align*}
                                    \AUC_{g(t)}(S(t),Z)\coloneqq \Phi\inparen{\sqrt{
                                    \begin{bmatrix}
                                      \Delta\mu_S & \Delta v_\ell
                                    \end{bmatrix}
                                    \begin{bmatrix}
                                      \sigma^2_{0S} + \sigma^2_{1S} & \rho_{0\ell}+\rho_{1\ell} \\
                                      \rho_{0\ell}+\rho_{1\ell} & \sigma^2_{0\ell} + \sigma^2_{1\ell}
                                    \end{bmatrix}^{-1}
                                    \begin{bmatrix}
                                      \Delta\mu_S\\ \Delta v_\ell
                                    \end{bmatrix}
                                    }}.\yesnum\label{eq:expression_for_sz}
                                  \end{align*}
                                  Further, let $\alpha\geq 0$, be such that $\AUC_{g(t)}(X)=\Phi\inparen{\sqrt{\alpha}},$
                                  then
                                  Equation~\eqref{eq:expression_for_sz} implies that
                                  \begin{align*}
                                    \AUC_{g(t)}(S(t),Z)\geq \Phi\inparen{\sqrt{ \alpha + \frac{\inparen{\Delta v_\ell - (\rho_{0\ell}+\rho_{1\ell})}^2}{\sigma^2_{0\ell} + \sigma^2_{1\ell}}  }}.\yesnum\label{eq:lower_bound}
                                  \end{align*}
                                \end{lemma}
                                \noindent The proof of \cref{fact:lower_bound_sz} appears in \cref{section:proofof:supporting_lemmas}.
                                Define $\alpha\geq 0$ to be a constant, such that
                                \begin{align*}
                                  \AUC_{g(t)}(X(t))=\Phi\inparen{\sqrt{\alpha}}.\yesnum\label{def:alpha}
                                \end{align*}
                                ($\alpha$ is uniquely defined since $\Phi(\sqrt{\cdot})$ is a strictly increasing function.)
                                Further, define $\Delta^\prime\in \R$ to be the term added to $\alpha$ in Equation~\eqref{eq:lower_bound}:  %
                                \begin{align*}
                                  \Delta^\prime &\coloneqq \frac{\inparen{\Delta v - (\rho_{0\ell}+\rho_{1\ell})}^2}{\sigma^2_{0\ell} + \sigma^2_{1\ell}}.\yesnum\label{eq:lb_delta}
                                \end{align*}

                                \subsubsection{Proof of \cref{thm:main_result}}
                                \label{section:proofof:thm:main_result_main}
                                Now we are ready to complete the proof of \cref{thm:main_result}.

                                \begin{proof}{Proof of \cref{thm:main_result}}
                                  Fix any auxiliary feature $\ell\in [d^\prime]$. %
                                  Consider two cases:\\ %

                                  \noindent {\bf Case A ($\abs{\rho_{0\ell}+\rho_{1\ell}} < \Delta v_\ell$):}
                                  In this case, from Equation~\eqref{eq:cond_on_aux_2}, we have that $\delta_\ell = \Delta v_\ell^{-1}\cdot \abs{\rho_{0\ell}+\rho_{1\ell}}\in (0,1)$.
                                  Thus,
                                  \begin{align*}
                                    \Delta^\prime
                                    = \frac{\inparen{\Delta v_\ell- (\rho_{0\ell}+\rho_{1\ell})}^2}{\sigma^2_{0\ell} + \sigma^2_{1\ell}}
                                    \Stackrel{}{\geq}  \frac{\Delta v_\ell^2\cdot(1-\delta_\ell)^2}{\sigma^2_{0\ell} + \sigma^2_{1\ell}}.
                                  \end{align*}

                                  \noindent {\bf Case B ($\abs{\rho_{0\ell}+\rho_{1\ell}} \geq \Delta v_\ell$):}
                                  In this case, from Equation~\eqref{eq:cond_on_aux_2}, we have that $\delta_\ell = 1$.
                                  Thus,
                                  \begin{align*}
                                    \Delta^\prime
                                    = \frac{\inparen{\Delta v_\ell- (\rho_{0\ell}+\rho_{1\ell})}^2}{\sigma^2_{0\ell} + \sigma^2_{1\ell}}
                                    \Stackrel{}{\geq } 0  =  \frac{\Delta v_\ell^2\cdot(1-\delta_\ell)^2}{\sigma^2_{0\ell} + \sigma^2_{1\ell}}.
                                  \end{align*}

                                  \noindent Combining both cases and using Equation~\eqref{eq:cond_on_aux_1}, we can lower bound $\Delta^\prime$ (defined in Equation~\eqref{eq:lb_delta}) as follows
                                  \begin{align*}
                                    \Delta^\prime
                                    = \frac{\inparen{\Delta v_\ell- (\rho_{0\ell}+\rho_{1\ell})}^2}{\sigma^2_{0\ell} + \sigma^2_{1\ell}}
                                    \qquad\quad \Stackrel{\text{\footnotesize(Cases A and B)}}{\geq}\qquad\quad  \frac{\Delta v_\ell^2\cdot(1-\delta_\ell)^2}{\sigma^2_{0\ell} + \sigma^2_{1\ell}}
                                    \ \ \Stackrel{\eqref{eq:cond_on_aux_1}}{\geq}\ \  {{\beta_\ell}^2\cdot (1-\delta_\ell)^2}.\yesnum\label{eq:bound_1}
                                  \end{align*}
                                  Define $\Delta_0$ as the RHS of the above equation, i.e., $\Delta_0\coloneqq  {{\beta_\ell}^2 (1-\delta_\ell)^2}.$
                                  {Then, we can rewrite Inequality~\eqref{eq:bound_1} as}
                                  \begin{align*}
                                    \Delta^\prime \geq \Delta_0.\yesnum\label{eq:lb_delta_2}
                                  \end{align*}
                                  Using \cref{fact:lower_bound_sz}, we can show a lower bound on the improvement in the AUC
                                  \begin{align*}
                                    \AUC_{g(t)}(S(t),Z^\ell) - \AUC_{g(t)}(X(t))
                                    &\qquad\ \ \Stackrel{\eqref{def:alpha}}{=}\qquad\ \ \AUC_{g(t)}(S(t),Z) - \Phi\inparen{\sqrt{\alpha}}\\ %
                                    &\qquad\ \ \Stackrel{\rm\eqref{eq:lb_delta},\ \cref{fact:lower_bound_sz}}{\geq}\qquad\ \ \Phi\inparen{\sqrt{\alpha + \Delta^\prime}} - \Phi\inparen{\sqrt{\alpha}}\\
                                    &\qquad\ \ \geq
                                    {\frac{2}{\sqrt{\pi}}\cdot
                                    \frac{\gamma^{\frac{3}2} \cdot \Delta_0 }{  \sqrt{(\sfrac{2}{e})+\Delta_0}\cdot (2+\Delta_0)}}
                                    \tagnum{Using Equation~\eqref{eq:lb_delta_2}, \cref{fact:upper_bound_alpha}, and  \cref{fact:lower_bound_normal}}
                                    \customlabel{eq:intermediate_step}{\theequation}\\
                                    &\qquad\ \ \geq\ \
                                    {\frac{2}{3\sqrt{\pi}}\cdot
                                    \frac{\gamma^{\frac{3}2} \cdot \Delta_0 }{  \sqrt{(\sfrac{2}{e})+1}} }
                                    \tag{Using that $0 \leq \delta_\ell,\beta_\ell\leq 1$ and, hence, $\Delta_0\coloneqq {\beta_\ell}^2\cdot (1-\delta_\ell)^2\leq 1$}\\
                                    &\qquad\ \ \geq\ \  {\frac{1}{3.51}\cdot \gamma^{\frac{3}2} \cdot \Delta_0}
                                    \tag{Using that $\frac{3\sqrt{\pi(2+e)}}{2\sqrt{e}}\leq 3.51$}\\
                                    &\qquad\ \ \geq\ \  {\frac{1}{3.51}\cdot \gamma^{\frac{3}2} \cdot \beta_\ell^2\cdot (1-\delta_\ell)^2}
                                    \tag{Substituting $\Delta_0\coloneqq {\beta_\ell}^2\cdot (1-\delta_\ell)^2$}\\
                                    &\qquad\ \ \geq\ \  {\frac{1}{4}\cdot \gamma^{\frac{3}2} \cdot \beta_\ell^2\cdot (1-\delta_\ell)^2.}
                                    \yesnum\label{eq:final_bound_2}
                                  \end{align*}
                                  Recall that \fairAUC{} selects an auxiliary feature $i$, satisfying
                                  \begin{align*}
                                    i\in \argmax\nolimits_{\ell\in [d^\prime]}\AUC_{g(t)}(S(t),Z^\ell).\yesnum\label{eq:fairAUC_selects_argmax}
                                  \end{align*}
                                  Using Equations~\eqref{eq:final_bound_2} and \eqref{eq:fairAUC_selects_argmax}, we get that
                                  \begin{align*}
                                    \AUC_{g(t)}(S(t),Z^i) - \AUC_{g(t)}(X(t))
                                    \ \  \ &\Stackrel{\eqref{eq:fairAUC_selects_argmax}}{=}\ \ \
                                    \max_{\ell\in [d^\prime]}\AUC_{g(t)}(S(t),Z^\ell) - \AUC_{g(t)}(X(t))\\
                                    \ \ \ &\Stackrel{\eqref{eq:final_bound_2} }{\geq }\ \ \
                                    \max_{\ell\in [d^\prime]}
                                    {\frac{1}{4} \cdot \gamma^{\sfrac{3}{2}}} {\beta_\ell}^{2}(1-\delta_\ell)^2.
                                    \yesnum\label{eq:final_bound_3}
                                  \end{align*}
                                  Finally, using \cref{fact:information_inequality}, we get that
                                  \begin{align*}
                                    \AUC_{g(t)}(X(t),Z^i) - \AUC_{g(t)}(X(t))
                                    \qquad &\Stackrel{\rm\cref{fact:information_inequality}}{\geq} \qquad
                                    \AUC_{g(t)}(S(t),Z^i) - \AUC_{g(t)}(X)\\
                                    \ \ &\Stackrel{\eqref{eq:final_bound_3} }{\geq }\qquad
                                    \max_{\ell\in [d^\prime]}
                                    {\frac{1}{4} \cdot \gamma^{\sfrac{3}{2}}} {\beta_\ell}^{2}(1-\delta_\ell)^2.
                                  \end{align*}
                                \end{proof}

                                \subsubsection{Proof of Supporting Lemmas}\label{section:proofof:supporting_lemmas}

                                \begin{proof}{Proof of \cref{fact:lower_bound_sz}}
                                  Equation~\eqref{eq:expression_for_sz} follows by using \cref{fact:optimal_auc_expression} with $d=2$. %
                                  To see this, note that by Equation~\eqref{eq:def_s_star}, $S(t)$ is a fixed projection of the random variable $X$.
                                  Since conditioned on $Y$ and $A$, $X\cup Z$ is distributed according to a multivariate Gaussian distribution (\cref{def:binormal_framework}), it follows that $X(t)$, and so $S(t)$, also has a Gaussian distribution conditioned on $Y$ and $A$ (see e.g., \cite[Theorem 5, Section 8.4]{mult_variate_Gaussian}).
                                  Now Equation~\eqref{eq:expression_for_sz} follows from \cref{fact:optimal_auc_expression} by substituting appropriate values for the covariance matrix between $S(t)$ and $Z$, and the means of $S(t)$ and $Z$.

                                  Equation~\eqref{eq:lower_bound} follows by expanding Equation~\eqref{eq:expression_for_sz}.
                                  Consider the expression inside $\Phi(\sqrt{\cdot})$ in Equation~\eqref{eq:expression_for_sz}.
                                  We have
                                  \begin{align*}
                                    \begin{bmatrix}
                                      \sigma^2_{0S} + \sigma^2_{1S} & \rho_{0\ell}+\rho_{1\ell} \\
                                      (\rho_{0\ell}+\rho_{1\ell})& \sigma^2_{0\ell} + \sigma^2_{1\ell}
                                    \end{bmatrix}^{-1}
                                    = \frac1{\inparen{\sigma^2_{0S} + \sigma^2_{1S} }\cdot \inparen{\sigma^2_{0\ell} + \sigma^2_{1\ell} } - (\rho_{0\ell}+\rho_{1\ell})^2 }
                                    \begin{bmatrix}
                                      \sigma^2_{0\ell} + \sigma^2_{1\ell} & -(\rho_{0\ell}+\rho_{1\ell}) \\
                                      -(\rho_{0\ell}+\rho_{1\ell})&  \sigma^2_{0S} + \sigma^2_{1S}
                                      \end{bmatrix}.
                                    \end{align*}
                                    Evaluating the rest of the expression, we have
                                    \begin{align*}
                                      &\frac1{\inparen{\sigma^2_{0S} + \sigma^2_{1S} }\cdot \inparen{\sigma^2_{0\ell} + \sigma^2_{1\ell} } - (\rho_{0\ell}+\rho_{1\ell})^2 }\cdot
                                      \begin{bmatrix}
                                        \Delta\mu_S & \Delta v_\ell
                                      \end{bmatrix}
                                      \begin{bmatrix}
                                        \sigma^2_{0\ell} + \sigma^2_{1\ell} & -(\rho_{0\ell}+\rho_{1\ell}) \\
                                        -(\rho_{0\ell}+\rho_{1\ell})&  \sigma^2_{0S} + \sigma^2_{1S}
                                      \end{bmatrix}
                                      \begin{bmatrix}
                                        \Delta\mu_S\\ \Delta v_\ell
                                        \end{bmatrix}\\
                                        &=\quad  \frac1{\inparen{\sigma^2_{0S} + \sigma^2_{1S} }\cdot \inparen{\sigma^2_{0\ell} + \sigma^2_{1\ell} } - (\rho_{0\ell}+\rho_{1\ell})^2 }\cdot
                                        \begin{bmatrix}
                                          \Delta\mu_S\cdot \inparen{\sigma^2_{0\ell} + \sigma^2_{1\ell}} -\Delta v_\ell\cdot (\rho_{0\ell}+\rho_{1\ell})\\
                                          -\Delta\mu_S\cdot \inparen{\rho_{0\ell}+\rho_{1\ell}} +\Delta v_\ell\cdot (\sigma^2_{0S} + \sigma^2_{1S})\\
                                          \end{bmatrix}^\top
                                          \begin{bmatrix}
                                            \Delta\mu_S\\ \Delta v_\ell
                                            \end{bmatrix}\\
                                            &=\quad  \frac1{\inparen{\sigma^2_{0S} + \sigma^2_{1S} }\cdot \inparen{\sigma^2_{0\ell} + \sigma^2_{1\ell} } - (\rho_{0\ell}+\rho_{1\ell})^2 }\cdot
                                            \inparen{
                                            \Delta\mu_S^2\cdot \inparen{\sigma^2_{0\ell} + \sigma^2_{1\ell}}
                                            -2\Delta\mu_S \Delta v_\ell\cdot (\rho_{0\ell}+\rho_{1\ell})
                                            + \Delta v_\ell^2 \cdot (\sigma^2_{0S} + \sigma^2_{1S})
                                            }\\
                                            &\Stackrel{\eqref{eq:value_of_mus},\eqref{eq:value_of_ss}}{=} \quad \frac{1}{\Delta\mu_S\cdot \inparen{\sigma^2_{0\ell} + \sigma^2_{1\ell} } - (\rho_{0\ell}+\rho_{1\ell})^2 }\cdot
                                            \inparen{
                                            \Delta\mu_S^2\cdot \inparen{\sigma^2_{0\ell} + \sigma^2_{1\ell}}
                                            -2 \Delta\mu_S\cdot\Delta v_\ell\cdot (\rho_{0\ell}+\rho_{1\ell})
                                            + \Delta v_\ell^2\cdot \Delta\mu_S
                                            }\\
                                            &\Stackrel{(\Delta\mu_S>0)}{=} \quad \frac{1}{\inparen{\sigma^2_{0\ell} + \sigma^2_{1\ell} } -
                                            \frac{(\rho_{0\ell}+\rho_{1\ell})^2}{\Delta\mu_S } }\cdot
                                            \inparen{
                                            \Delta\mu_S\cdot \inparen{\sigma^2_{0\ell} + \sigma^2_{1\ell}}
                                            -2 \Delta v_\ell\cdot (\rho_{0\ell}+\rho_{1\ell})
                                            + \Delta v_\ell^2
                                            }\\
                                            &\Stackrel{}{=}
                                            \quad \Delta\mu_S + \frac{1}{\inparen{\sigma^2_{0\ell} + \sigma^2_{1\ell} }  -
                                            \frac{(\rho_{0\ell}+\rho_{1\ell})^2}{\Delta\mu_S}
                                            }
                                            \inparen{
                                            (\rho_{0\ell}+\rho_{1\ell})^2
                                            -2 \Delta v_\ell\cdot (\rho_{0\ell}+\rho_{1\ell})
                                            + \Delta v_\ell^2
                                            }\\
                                            &\Stackrel{}{=}
                                            \quad \Delta\mu_S + \frac{1}{\inparen{\sigma^2_{0\ell} + \sigma^2_{1\ell} }  -
                                            \frac{(\rho_{0\ell}+\rho_{1\ell})^2}{\Delta\mu_S}
                                            }\cdot
                                            (\Delta v_\ell - (\rho_{0\ell}+\rho_{1\ell}))^2
                                          \end{align*}
                                          \begin{align*}
                                            &\Stackrel{}{\geq}
                                            \quad \Delta\mu_S
                                            + \frac{\inparen{\Delta v_\ell - (\rho_{0\ell}+\rho_{1\ell})}^2}{\sigma^2_{0\ell} + \sigma^2_{1\ell}}
                                            \hspace{77mm}
                                            \tag{Using $\frac{(\rho_{0\ell}+\rho_{1\ell})^2}{\Delta\mu_S}\geq 0$}\\
                                            &\Stackrel{\eqref{eq:value_of_mus}}{=}
                                            \quad \Delta\mu^\top (\Sigma_0+\Sigma_1)^{-1} \Delta\mu
                                            + \frac{\inparen{\Delta v_\ell - (\rho_{0\ell}+\rho_{1\ell})}^2}{\sigma^2_{0\ell} + \sigma^2_{1\ell}}. \yesnum\label{eq:final_bound}
                                          \end{align*}
                                          Substituting Equation~\eqref{eq:final_bound} in \cref{fact:optimal_auc_expression}, and using the fact that $\Phi(\sqrt{\cdot})$ is an increasing function, we have
                                          \begin{align*}
                                            \AUC_{g(t), \cD}(S(t),Z) \geq  \Phi\inparen{
                                            \sqrt{
                                            \Delta\mu^\top (\Sigma_0+\Sigma_1)^{-1} \Delta\mu +
                                            \frac{\inparen{\Delta v_\ell - (\rho_{0\ell}+\rho_{1\ell})}^2}{\sigma^2_{0\ell} + \sigma^2_{1\ell}}
                                            }
                                            }.\yesnum\label{eq:expression_of_auc}
                                          \end{align*}
                                          From \cref{fact:optimal_auc_expression}, we also have that
                                          \begin{align*}
                                            \AUC_{g(t), \cD}(X(t)) = \Phi\inparen{
                                            \sqrt{\Delta\mu^\top (\Sigma_0+\Sigma_1)^{-1} \Delta\mu}
                                            }.
                                          \end{align*}
                                          Thus, $\alpha = \Delta\mu^\top (\Sigma_0+\Sigma_1)^{-1} \Delta\mu$.
                                          Combining this with Equation~\eqref{eq:expression_of_auc}, we get
                                          \begin{align*}
                                            \AUC_{g(t)}(S(t),Z)\geq \Phi\inparen{\sqrt{ \alpha + \frac{\inparen{\Delta v_\ell - (\rho_{0\ell}+\rho_{1\ell})}^2}{\sigma^2_{0\ell} + \sigma^2_{1\ell}}  }}.
                                          \end{align*}
                                        \end{proof}

                                        \begin{proof}{Proof of \cref{fact:information_inequality}}
                                          {By definition of the AUC (\cref{def:auc_2}), we have
                                          \begin{align*}
                                            \AUC(X,Z)&\coloneqq \max_{v\in \R^{d+1},\psi} \AUC(v,\psi,X,Z)
                                            \quad\text{and}\quad
                                            \AUC(S,Z) \coloneqq \max_{v\in \R^{2},\psi} \AUC(v,\psi,S,Z).
                                            \yesnum\label{eq:def_auc_1}
                                          \end{align*}
                                          As shown in the proof of \cref{fact:optimal_auc_expression}, a generalized linear model $G$ defined by weight $w\in \R^{d}$, increasing and invertible link function $\psi\colon \R\to \R$, and threshold $t\in \R$, makes the same predictions as the linear classifier $C$ defined with weights $w$ and threshold $\psi(t)$.
                                          In particular, $G$ and $C$ have the same AUC.
                                          Combining this with \cref{eq:def_auc_1}, it follows that:
                                          \begin{align*}
                                            \AUC(X,Z)&\coloneqq \max_{v\in \R^{d+1}} \AUC(v,I,\psi,X,Z)
                                            \quad\text{and}\quad
                                            \AUC(S,Z) \coloneqq \max_{v\in \R^{2}} \AUC(v,I,\psi,S,Z).
                                            \yesnum\label{eq:def_auc_12}
                                          \end{align*}
                                          where $I\colon \R\to \R$ is the identity function.
                                          Define vectors
                                          \begin{align*}
                                            v_2\coloneqq \argmax_{v\in \R^{2}} \AUC(v,I,S,Z),
                                            \quad\text{and}\quad
                                            v_1\coloneqq \begin{bmatrix}
                                            w & 0\\
                                            0 & 1
                                          \end{bmatrix}v_2\in \R^{d+1}.\yesnum\label{eq:def_v1}
                                        \end{align*}
                                        \noindent} Notice that the linear classifier using $v_1$ on $X$ and $Z$, is identical to the linear classifier using $v_2$ on $S$ and $Z$:
                                        \begin{align*}
                                          \inangle{v_1, (X^1,\dots,X^d,Z)}\ \  \Stackrel{\eqref{eq:def_v1}}{=}\ \  v_2^\top\begin{bmatrix}
                                          w^\top & 0\\
                                          0 & 1
                                        \end{bmatrix}\begin{bmatrix}
                                        X\\
                                        Z
                                      \end{bmatrix} \qquad\Stackrel{(S\coloneqq \inangle{w,X})}{=}\qquad v_2^\top\begin{bmatrix}
                                      S\\
                                      Z
                                    \end{bmatrix}.\yesnum\label{eq:equiv}
                                  \end{align*}
                                  Using this, we have
                                  \begin{align*}
                                    \AUC(X,Z)&\ \ \stackrel{\eqref{eq:def_auc_12}}{\geq}\ \  \AUC(v_1,I,X,Z)\ \ \Stackrel{\eqref{eq:equiv}}{=}\ \ \AUC(v_2,I,S,Z) \ \ \Stackrel{\eqref{eq:def_v1}}{=}\ \  \AUC(S,Z).
                                  \end{align*}
                                \end{proof}

                                \subsubsection{Proof of \cref{thm:result_for_advantaged_group}}
                                \begin{proof}{Proof of \cref{thm:result_for_advantaged_group}}
                                  The proof of \cref{thm:result_for_advantaged_group} follows from Equation~\eqref{eq:final_bound_2} and \cref{fact:information_inequality} in the proof of \cref{thm:main_result}.
                                  In the proof of \cref{thm:main_result} only Equation~\eqref{eq:fairAUC_selects_argmax} uses the fact that $g(t)$ is the disadvantaged group in iteration $t$.
                                  In particular, the proof of Equation~\eqref{eq:final_bound_2} (which occurs before Equation~\eqref{eq:fairAUC_selects_argmax}) does not use the fact that $g(t)$ is the disadvantaged group in iteration $t$.
                                  Thus, we can repeat the proof of Equation~\eqref{eq:final_bound_2} by substituting $g(t)$ with $\hat{g}(t)$.
                                  Then for all $\ell\in [d^\prime]\backslash Q(t)$, it holds that\footnote{Recall that in the proof of \cref{thm:main_result}, we dropped the superscript on $\hat{\gamma}^\sexp{t}$, $\hat{\beta}_\ell^\sexp{t}$, and $\hat{\delta}_\ell^\sexp{t}$. Here, we return to specifying the superscripts.}
                                  \begin{align*}
                                    \AUC_{\hat{g}(t)}(S(t),Z^\ell) - \AUC_{\hat{g}(t)}(X(t))
                                    &\Stackrel{}{\geq}
                                    {\frac{1}{4} \cdot \inparen{\hat{\gamma}^\sexp{t} }^{\sfrac{3}{2}}} \cdot \inparen{{\hat{\beta}_\ell}^\sexp{t} (1-\hat{\delta}_\ell^\sexp{t}) }^2.
                                    \yesnum\label{eq:end_bound}
                                  \end{align*}
                                  The proof of \cref{fact:information_inequality} does not refer to $g(t)$. Thus, we can use it directly.
                                  This gives us that for all $\ell\in [d^\prime]\backslash Q(t)$
                                  \begin{align*}
                                    \AUC_{\hat{g}(t)}(X(t),Z^\ell) - \AUC_{\hat{g}(t)}(X(t))
                                    \quad\ \ \  &\Stackrel{\rm\cref{fact:information_inequality}}{\geq} \qquad
                                    \AUC_{\hat{g}(t)}(S(t),Z^\ell) - \AUC_{\hat{g}(t)}(X) \
                                    \\
                                    &\Stackrel{\eqref{eq:end_bound}}{\geq} \qquad
                                    {\frac{1}{4} \cdot \inparen{\hat{\gamma}^\sexp{t} }^{\sfrac{3}{2}}} \cdot \inparen{{\hat{\beta}_\ell}^\sexp{t} (1-\hat{\delta}_\ell^\sexp{t}) }^2.
                                  \end{align*}
                                  In particular, this holds for the feature $i\in [d^\prime]\backslash Q(t)$, selected by \fairAUC{}.
                                \end{proof}

                                \subsubsection{Proof of \cref{fact:lower_bound_normal}}\label{section:proofof:fact:lower_bound_normal}

                                \begin{proof}{Proof of \cref{fact:lower_bound_normal}.}
                                  \begin{align*}
                                    \Phi(\sqrt{\alpha+\Delta}) - \Phi(\sqrt{\alpha})
                                    &\qquad =\qquad \int_{\sqrt{\alpha}}^{\sqrt{\alpha+\Delta}}  \frac1{\sqrt{2\pi}} e^{-\sfrac{y^2}{2}}dy\\
                                    &\qquad \Stackrel{}{\geq} \qquad \int_{\sqrt{\alpha}}^{\sqrt{\alpha+\Delta_0}}  \frac1{\sqrt{2\pi}} e^{-\sfrac{y^2}{2}}dy\tag{Using that $\Delta\geq \Delta_0$ and that the RHS is an increasing function of $\Delta$}\\
                                    &\qquad \geq\qquad  \int_{\sqrt{\alpha}}^{\sqrt{\alpha+\Delta_0}} \frac{y}{\sqrt{\alpha+\Delta_0}}  \frac1{\sqrt{2\pi}} e^{-\sfrac{y^2}{2}}dy \tag{Using the fact that for all $y\in [\sqrt{\alpha}, \sqrt{\alpha+\Delta_0}]$, $\frac{y}{\sqrt{\alpha+\Delta_0}}\leq 1$}\\
                                    &\qquad =\qquad  \frac{-e^{-\sfrac{y^2}{2}}}{\sqrt{\alpha+\Delta_0}}\cdot \frac1{\sqrt{2\pi}}
                                    \bigg|_{\sqrt{\alpha}}^{\sqrt{\alpha+\Delta_0}}\\
                                    &\qquad =\qquad  \frac{1}{\sqrt{2\pi}}\cdot \frac{e^{-\frac{\alpha}2} \cdot \inparen{1 -e^{-\sfrac{\Delta_0}2}} }{  \sqrt{\alpha+\Delta_0}}\\
                                    &\qquad {\geq\qquad  \frac{1}{\sqrt{2\pi}}\cdot
                                    \frac{e^{-\frac{3\alpha}2} \cdot \inparen{1 -e^{-\sfrac{\Delta_0}2}} }{  \sqrt{\frac{2}{e}+\Delta_0}}}
                                    \tagnum{Using that for all $\alpha,\Delta_0>0$, $\frac{e^{-\sfrac{\alpha}{2}}}{\sqrt{\alpha+\Delta_0}}\geq \frac{e^{-\sfrac{3\alpha}{2}}}{\sqrt{\frac{2}{e}+\Delta_0}}$; see Equation~\eqref{eq:new_inequality}}
                                    \customlabel{eq:use_of_new_inequality}{\theequation}
                                    \\
                                    &\qquad {\geq\qquad  \frac{1}{\sqrt{2\pi}}\cdot
                                    \frac{e^{-\frac{3\alpha}2} \cdot \Delta_0 }{  \sqrt{\frac{2}{e}+\Delta_0}\cdot (2+\Delta_0)}}
                                    \tag{Using the fact that for all $x\in \R$, $1 -e^{-\sfrac{x}2}\geq \frac{x}{2+x}$}\\
                                    &\qquad {\geq\qquad  \frac{1}{\sqrt{2\pi}}\cdot
                                    \frac{(2\gamma)^{\frac{3}2} \cdot \Delta_0 }{  \sqrt{\frac{2}{e}+\Delta_0}\cdot (2+\Delta_0)}}
                                    \tag{Using that $e^{-\frac{3\alpha}2}$ is a decreasing function of $\alpha$ and $\alpha<-2\cdot\ln\inparen{2\gamma}$}\\
                                    &\qquad {=\qquad  \frac{2}{\sqrt{\pi}}\cdot
                                    \frac{\gamma^{\frac{3}2} \cdot \Delta_0 }{  \sqrt{\frac{2}{e}+\Delta_0}\cdot (2+\Delta_0)}.}
                                  \end{align*}

                                  \noindent {It remains to prove the following inequality, which was used in Equation~\eqref{eq:use_of_new_inequality},
                                  \begin{align*}
                                    \forall x,y>0,\quad
                                    \frac{e^{-\frac{x}{2}}}{\sqrt{x+y}}\geq \frac{e^{-\frac{3x}{2}}}{\sqrt{\frac{2}{e}+y}}.
                                    \yesnum\label{eq:new_inequality}
                                  \end{align*}
                                  Using the fact that for all $x\geq 0$, $e^{-\frac{x}{2}}x\leq \frac{2}{e}$, we have that for any $x,y>0$
                                  \begin{align*}
                                    e^{-\frac{x}{2}}\cdot \inparen{x + y} \leq \frac{2}{e}+e^{-\frac{x}{2}}y \ \  \Stackrel{x>0}{<}\ \  \frac{2}{e}+y.
                                    \yesnum\label{eq:interm_ineq}
                                  \end{align*}
                                  Hence, we have
                                  \begin{align*}
                                    \frac{e^{-\frac{x}{2}}}{\sqrt{x+y}}
                                    = \frac{e^{-\frac{x}{2}}\cdot \sqrt{e^{-\frac{x}{2}}}}{\sqrt{e^{-\frac{x}{2}}\cdot \inparen{x + y}}}
                                    \ \ \Stackrel{\eqref{eq:interm_ineq}}{\geq} \ \
                                    \frac{e^{-\frac{3x}{2}}}{\sqrt{\frac{2}{e}+y}}.
                                  \end{align*}}

                                \end{proof}

                                \subsection{Proof of Theoretical Guarantees of noisy \fairAUC{}}\label{section:noisy_auc_proof}
                                In this section, we formally describe the noisy \fairAUC{} procedure, state its theoretical guarantees and prove them.

                                \fairAUC{}'s objective is to increase the AUC of the lowest AUC group. However, fairAUC\ can increase bias with the acquisition of a new feature. Suppose we have data ${{X}}$ and we acquire feature ${{Z}}$ so that our new data is ${{X}}' \coloneqq ({X},{Z})$.
                                We can then have $\text{Bias}({X}')>\text{Bias}({X})$, where $\text{Bias}({X})$ and $\text{Bias}({X}')$ are the values of bias obtained using features ${X}$ and ${X}'$ respectively.
                                In this section, we provide an approach that ensures that, in each iteration of \fairAUC{}, the bias is smaller than or equal to the bias in the previous iteration. Thus, overall we can guarantee that the bias will only decrease over rounds as we acquire more features. At a high level, the strategy involves adding a noisy version of ${{Z}}$ to the group with higher AUC rather than ${{Z}}$ itself. This strategy, however, trades off predictive accuracy to ensure bias does not increase.

                                To formalize this, consider the $t$-th iteration of the \fairAUC{} procedure.
                                Let $0\leq \bias(t)\leq 1$ be the bias at the {\em start} of the $t$-th iteration
                                \begin{align*}
                                  \bias(t)\coloneqq 1-\max\inbrace{\frac{\AUC_a(X(t))}{\AUC_b(X(t))}, \frac{\AUC_b(X(t))}{\AUC_a(X(t))}}.
                                \end{align*}
                                Here, $\AUC_g(X)$ denotes the largest AUC, on group $g\in \ab$, of a generalized linear model that uses $X(t)$ to predict $Y$ (as defined in \cref{def:auc_2}).
                                Suppose that in this iteration, \fairAUC{} acquires an auxiliary feature $Z$.
                                Let $\bias(t+1)$ be the bias after acquiring $Z$
                                \begin{align*}
                                  \bias(t+1)\coloneqq 1-\max\inbrace{\frac{\AUC_a(X(t),Z)}{\AUC_b(X(t),Z)}, \frac{\AUC_b(X(t),Z)}{\AUC_a(X(t),Z)}}.
                                \end{align*}
                                While \fairAUC{} does not decrease the AUC of either group (\cref{thm:main_result,thm:result_for_advantaged_group}), it can increase the bias.
                                We can then have $\bias(t+1) > \bias(t)$.
                                To avoid this, one can replace $Z$ with its ``noisy version,'' which introduces noise to the group with the higher AUC (where the AUC refers to the measure \textit{after} $Z$ is added).
                                We will show that this ensures that the bias never increases.

                                Let $g(t)$ be the advantaged group at the {\em end} of the $t$-th iteration, i.e., the group with the higher AUC. %
                                Let $0\leq \lambda\leq 1$ be a parameter controlling the extent of the noise.
                                The noisy version of $Z$ is
                                \begin{align*}
                                  Z_\lambda \coloneqq \begin{cases}
                                  \lambda\cdot Z + (1-\lambda)\cdot N & \text{if } A = g(t),\\
                                  Z & \text{if } A \neq g(t).
                                \end{cases}
                                \yesnum\label{def:z_gamma}
                              \end{align*}
                              where $N$ is a standard normal random variable, independent of all $d$ features, the class label $Y$, and protected group $A$.
                              In other words, $Z_\lambda$ is constructed from $Z$ by scaling $Z$ by a factor of $\lambda$ for all samples in the advantaged group and then adding standard normal noise (scaled by $1-\lambda$) to these samples.
                              Let $\bias_\lambda(t+1)$ be the bias obtained by using features $X(t)$ and $Z_\lambda$
                              \begin{align*}
                                \bias_\lambda(t+1) \coloneqq 1-\max\inbrace{\frac{\AUC_a(X(t),Z_\lambda)}{\AUC_b(X(t),Z_\lambda)}, \frac{\AUC_b(X(t),Z_\lambda)}{\AUC_a(X(t),Z_\lambda)}}.
                              \end{align*}
                              We show that one can choose a value of $\lambda$ such that $\bias_\lambda(t+1)\leq \bias(t)$ and the AUC of neither group decreases compared to their value {\em before} the iteration.
                              Formally, we prove the following theorem.
                              \begin{theorem}\label{thm:main_thm_adding_noise}
                                Suppose that the $m$ features $X\cup Z$, class label $Y$, and protected group $A$ follow the binormal framework (\cref{def:binormal_framework}).
                                Further, assume that two summary statistics subroutines satisfy \cref{asmp:ssr}.
                                Then, for each iteration $t\in[d^\prime]$ and each auxiliary feature $Z$, there is a value of $0\leq \lambda\leq 1$ such that the bias {obtained by using features $(X(t),Z_\lambda)$ is at most $\bias(t)$,}
                                \begin{align*}
                                  \bias_\lambda(t+1) \leq \bias(t),
                                \end{align*}
                                and the AUC of neither group decrease compared to their values before the $t$-th iteration,
                                \begin{align*}
                                  \AUC_a(X(t),Z_\lambda) &\geq \AUC_a(X(t))\quad \text{and}\quad
                                  \AUC_b(X(t),Z_\lambda) \geq \AUC_b(X(t)).
                                  \yesnum\label{eq:guarantee_noise_AUC}
                                \end{align*}
                              \end{theorem}
                              \noindent  {Thus, \cref{thm:main_thm_adding_noise} shows that adding a suitable amount of noise to the samples in the advantaged group ensures that the selected feature does not increase the amount of bias.
                              At the same time, the AUC improvement for the disadvantaged group is unaffected because we do not add noise to the samples in the disadvantaged group.
                              The AUC improvement for the advantaged group can decrease but we can lower bound the amount of decrease as explained in the following remark.}

                              \smallskip

                              \begin{remark}\label{remark:noisyproof}
                                \red{\cref{thm:result_for_advantaged_group}'s proof lower bounds the increase in the AUC of the advantaged group due to the acquisition of any feature that follows the normal distribution conditioned on the class and the protected group labels.
                                (This, then, directly implies the lower bound in \cref{thm:result_for_advantaged_group}.)
                                The same lower bound also applies to the ``noisy'' feature $Z_\lambda$ because it follows the normal distribution for any class and group.
                                The latter is true because $Z_\lambda$ is a sum of two normally-distributed variables (the selected feature $Z$ and the standard-normal noise) and because the sum of any two normally-distributed variables is another normally-distributed variable.
                                The lower bound depends on $Z_\lambda$'s moments and correlation with the scores $S$ via parameters that are analogous to those in \cref{thm:result_for_advantaged_group}.
                                If the class-wise variances of $Z$ sum to at least $\alpha$, then the AUC-improvement with the noisy feature is at least $\frac{1}{1+\frac{2(1-\lambda)^2}{\alpha\lambda^2}}$ times the lower bound in \cref{thm:result_for_advantaged_group}.}
                              \end{remark}

                              \smallskip

                              \noindent  \red{Moreover, even when we add noise to the acquired feature in this iteration. The lower bound in \cref{thm:main_result_main_body} (and that in \cref{thm:result_for_advantaged_group}) continues to hold in subsequent iterations.
                              This is because their proofs hold whenever all features (acquired and auxiliary) are normally-distributed, conditioned on class and group and, as we saw in the above remark, if $Z$ is normally-distributed (for each class and group), then so is $Z_\lambda$.}

                              We also prove two corollaries of \cref{thm:main_thm_adding_noise}, which generalize it (\cref{coro:coro_adding_noise}) and give a stronger result if when an additional condition is satisfied (\cref{coro:coro_adding_noiseb}).

                              \begin{corollary}\label{coro:coro_adding_noise}
                                Under the same assumptions as \cref{thm:main_thm_adding_noise},
                                for each iteration $t\in[d^\prime]$ and each auxiliary feature $Z$,
                                there is a $\Delta\geq 0$ such that for any value\footnote{Note the result is vacuously true if $\bias(t)-\Delta > \bias(t+1)$.} $$\bias(t)-\Delta\leq v\leq \bias(t+1),$$
                                there exists a $0\leq \lambda\leq 1$ such that
                                $\bias_\lambda(t+1)=v$ and Equation \eqref{eq:guarantee_noise_AUC} holds.
                              \end{corollary}

                              \noindent In the next corollary, note that we can swap groups $a$ and $b$.

                              \begin{corollary}\label{coro:coro_adding_noiseb}
                                Under the same assumptions as \cref{thm:main_thm_adding_noise},
                                for each iteration $t\in[d^\prime]$ and each auxiliary feature $Z$,
                                if group $a$ is the advantaged group at the end of the $t$-th iteration (i.e., $\AUC_a(X(t),Z) \geq \AUC_b(X(t),Z)$) and  $\AUC_a(X(t))\leq \AUC_b(X(t),Z),$
                                then
                                there exists a value $0\leq\lambda\leq 1$ such that $\bias_\lambda(t+1)=0$ and \cref{eq:guarantee_noise_AUC} holds.
                              \end{corollary}
                              \noindent Finally, we show that the above method of adding noise to only one group dominates a method that adds noise to both groups.
                              Given parameters $0\leq \alpha,\beta\leq 1$, determining the extent of noise on groups $a$ and $b$, the new method uses
                              \begin{align*}
                                Z_{\alpha,\beta} \coloneqq \begin{cases}
                                \alpha\cdot Z + (1-\alpha)\cdot N & \text{if } A = a,\\
                                \beta\cdot Z + (1-\beta)\cdot N & \text{if } A = b.
                              \end{cases}
                              \yesnum\label{def:z_gamma_alpha}
                            \end{align*}
                            Let $\bias_{\alpha,\beta}(t+1)$ be the bias obtained by using the features $\inparen{X(t), Z_{\alpha,\beta}}$.
                            We show that for any fixed value of bias, i.e., $\bias_{\alpha,\beta}(t+1)=\bias_{\lambda}(t+1)$, the group-wise AUCs obtained by adding noise to only the advantaged group (i.e., using $Z_\lambda$) weakly Pareto dominates the group-wise AUCs obtained by adding noise to both groups (i.e., using $Z_{\alpha,\beta}$).

                            \begin{theorem}\label{thm:pareto_dom}
                              Under the same assumptions as \cref{thm:main_thm_adding_noise},
                              for any $0\leq\alpha,\beta\leq 1$ such that $\bias_{\alpha,\beta}(t+1)\leq \bias(t+1)$, there exists  $0\leq\lambda\leq 1$ such that
                              $\bias_{\alpha,\beta}(t+1)=\bias_\lambda(t+1)$ and
                              \begin{align*}
                                \AUC_{a}(X(t), Z_\lambda)\geq \AUC_{a}(X(t), Z_{\alpha,\beta})
                                \quad\text{and}\quad
                                \AUC_{b}(X(t), Z_\lambda)\geq \AUC_{b}(X(t), Z_{\alpha,\beta}).
                              \end{align*}
                            \end{theorem}

                            \subsubsection{Proof of \cref{thm:main_thm_adding_noise}}
                            In this section, we prove \cref{thm:main_thm_adding_noise}.
                            We need to show that there exists a $0\leq \lambda\leq 1$ such that $\bias_\lambda(t+1) \leq \bias(t)$.
                            Recall that $\bias_\lambda(t+1)$ is a function of $\AUC_a(X(t), Z_\lambda)$ and $\AUC_b(X(t), Z_\lambda)$.
                            We will express $\AUC_a(X(t), Z_\lambda)$ and $\AUC_b(X(t), Z_\lambda)$ as functions of $\lambda$.
                            The proof follows by analyzing these functions.
                            Without loss of generality assume that $g(t)=a$, i.e.,
                            $$\AUC_a(X(t),Z) \geq \AUC_b(X(t),Z).$$

                            \smallskip
                            \noindent{\bf Expression for $\mathbf{\AUC_b(X(t), Z_\lambda)}$.}
                            Since $(Z_\lambda=Z)\mid A=b$ (\cref{def:z_gamma}), we get that
                            \begin{align*}
                              \forall \lambda\in [0,1],\quad
                              \AUC_b(X(t), Z_\lambda) = \AUC_b(X(t), Z).
                              \yesnum\label{eq:invariance}
                            \end{align*}
                            Hence, $\AUC_b(X(t), Z_\lambda)$ is invariant of $\lambda$.

                            \smallskip
                            \noindent{\bf Expression for $\mathbf{\AUC_a(X(t), Z_\lambda)}$.}
                            Using \cref{fact:optimal_auc_expression}, we can express $\AUC_a(X(t), Z_\lambda)$ as a function of the conditional means and covariances of $\inparen{X(t), Z_\lambda}$; conditioned on the events $(A=a,Y=0)$ or $(A=a,Y=1)$.
                            We need some additional notation to state this expression.
                            For each $y\in \zo$, let the mean and covariance matrix of $\inparen{X(t), Z}$ conditioned on $(Y=y,A=a)$ be:
                            \begin{align*}
                              \begin{bmatrix}
                                \mu_{X,y}\\ \mu_{Z,y}
                              \end{bmatrix}
                              \quad\text{and}\quad
                              \begin{bmatrix}
                                \Sigma_{y} & \rho_y \\
                                \rho_y^\top & \sigma_y
                                \end{bmatrix},
                                \yesnum\label{eq:mean_and_cov_1}
                              \end{align*}
                              where for each $y\in \zo$
                              \begin{itemize}
                                \item $\mu_{X,y}\in \R^{d+t}$ is the mean of $X(t)|Y=y, A=a$,
                                \item $\mu_{Z,y}\in \R$ is the mean of $Z|Y=y, A=a$,
                                \item $\Sigma_y\in \R^{(d+t)\times (d+t)}$ is the covariance matrix of $X(t)|Y=y, A=a$,
                                \item $\sigma_y\in \R$ is the variance of $Z|Y=y, A=a$, and
                                \item $\rho_y\in \R^{d+t}$ is the covariance of $X(t)|Y=y, A=a$ and $Z|Y=y, A=a$.
                              \end{itemize}
                              From the definition of $Z_\lambda$ (\cref{def:z_gamma}), \cref{eq:mean_and_cov_1}, and the fact that $N$ is a standard normal variable independent of $(X(t),Z,A,Y)$,
                              we get that the mean and covariance matrix of $\inparen{X(t), Z_\lambda}$ conditioned on $(Y=y,A=a)$ to be:
                              \begin{align*}
                                \begin{bmatrix}
                                  \mu_{X,y}\\ \lambda\cdot \mu_{Z,y}
                                \end{bmatrix}
                                \quad\text{and}\quad
                                \begin{bmatrix}
                                  \Sigma_{y} & \lambda\cdot \rho_y \\
                                  \lambda \cdot \rho_y^\top & \lambda^2\cdot \sigma_y+ (1-\lambda)^2
                                  \end{bmatrix}.
                                  \yesnum\label{eq:mean_and_cov_gamma}
                                \end{align*}
                                To simplify the notation, define
                                \begin{align*}
                                  \Sigma &\coloneqq \Sigma_0 + \Sigma_1,\qquad \ \ \
                                  \sigma\coloneqq \sigma_0 + \sigma_1,\qquad
                                  \rho\coloneqq \rho_0 + \rho_1,\\
                                  \Delta\mu_{X} &\coloneqq \abs{\mu_{X,1} - \mu_{X,0}},\quad
                                  \Delta\mu_{Z} \coloneqq \abs{\mu_{Z,1} - \mu_{Z,0}}.
                                \end{align*}
                                Substituting the above definitions and Equation~\eqref{eq:mean_and_cov_gamma} in \cref{fact:optimal_auc_expression}, we get that
                                \begin{align*}
                                  \AUC_a(X(t), Z_\lambda)
                                  =
                                  \Phi\inparen{
                                  \begin{bmatrix}
                                    \Delta\mu_{X}\\ \lambda\cdot \Delta\mu_{Z}
                                    \end{bmatrix}^\top
                                    \begin{bmatrix}
                                      \Sigma & \lambda\cdot \rho \\
                                      \lambda \cdot \rho^\top & \lambda^2\cdot \sigma+ (1-\lambda)^2
                                    \end{bmatrix}^{-1}
                                    \begin{bmatrix}
                                      \Delta\mu_{X}\\ \lambda\cdot \Delta\mu_{Z}
                                    \end{bmatrix}
                                    }.
                                    \yesnum\label{eq:exp_auc_gamma}
                                  \end{align*}
                                  Substituting $\lambda=0$ and $\lambda=1$  in \cref{eq:exp_auc_gamma}, we recover that
                                  \begin{align*}
                                    \AUC_a(X(t), Z_0)=\AUC_a(X(t))
                                    \quad\text{and}\quad
                                    \AUC_a(X(t), Z_1)=\AUC_a(X(t), Z).
                                    \yesnum\label{eq:auc_exp_for_0_and_1}
                                  \end{align*}
                                  This is expected because $Z_0|g=a$ is an independent standard normal variable, and therefore does not provide ``any information'' about $Y$, and because $Z_1=Z$, respectively.
                                  Using \cref{eq:exp_auc_gamma} we can also prove the following result.
                                  \begin{lemma}\label{lem:continuity}
                                    $\AUC_a(X(t),Z_\lambda)$ is a continuous function of $\lambda$ over $[0,1]$.
                                  \end{lemma}

                                  \paragraph{\em Proof of \cref{thm:main_thm_adding_noise}.}

                                  If $\AUC_a(X(t), Z)=\AUC_b(X(t), Z)$, then $\bias(t+1)=0$.
                                  Since $\bias(t)\geq 0$ and $\bias(t+1)=\bias_1(t+1)$, we are done.
                                  Henceforth, assume that
                                  \begin{align*}
                                    \AUC_a(X(t), Z) > \AUC_b(X(t), Z).
                                    \yesnum\label{eq:non_trivial_case}
                                  \end{align*}

                                  \paragraph{Case A ($\AUC_b(X(t), Z)\geq \AUC_a(X(t))$:}
                                  In this case
                                  \begin{align*}
                                    \AUC_a(X(t), Z_0)
                                    \quad &\Stackrel{\eqref{eq:auc_exp_for_0_and_1}}{=} \quad
                                    \AUC_a(X(t))\\
                                    &\leq\quad \AUC_b(X(t), Z)\\
                                    \quad&\Stackrel{\eqref{eq:non_trivial_case}}{<} \quad\AUC_a(X(t), Z)\\
                                    \quad &\Stackrel{\eqref{eq:auc_exp_for_0_and_1}}{=} \quad \AUC_a(X(t), Z_1).
                                  \end{align*}
                                  Let $f(\lambda)\coloneqq \AUC_a(X(t), Z_\lambda)$.
                                  $f$ is continuous over $[0,1]$ by \cref{lem:continuity}.
                                  Since $f$ is continuous over $[0,1]$ and $f(0)\leq \AUC_b(X(t), Z) < f(1)$,
                                  by the intermediate value theorem there exists an $\alpha\in [0,1]$ such that $f(\alpha)=\AUC_b(X(t), Z)$.
                                  Using $f(\alpha)=\AUC_b(X(t), Z)$ and \cref{eq:invariance}, we get
                                  $$\bias_\alpha(t+1)=1-\frac{\AUC_b(X(t), Z_\alpha)}{\AUC_a(X(t), Z_\alpha)}=1-\frac{\AUC_b(X(t), Z)}{\AUC_b(X(t), Z)}=0.$$
                                  Hence, in this case \cref{thm:main_thm_adding_noise} follows for $\lambda=\alpha$.
                                  This also proves \cref{coro:coro_adding_noiseb}.

                                  \smallskip

                                  \paragraph{Case B ($\AUC_b(X(t), Z) < \AUC_a(X(t))$:}
                                  Define
                                  \begin{align*}
                                    \zeta\coloneqq \AUC_b(X(t),Z)-\AUC_b(X(t)).
                                    \yesnum\label{def:zeta_noise}
                                  \end{align*}
                                  We have that
                                  \begin{align*}
                                    \AUC_a(X(t), Z_0)
                                    \quad \Stackrel{\eqref{eq:auc_exp_for_0_and_1}}{=} \quad \AUC_a(X(t))
                                    > \AUC_b(X(t), Z)
                                    \quad\Stackrel{\eqref{eq:invariance}}{=}\quad
                                    \AUC_b(X(t), Z_0).
                                  \end{align*}
                                  Using this inequality, we can express $\bias_0(t+1)$ as
                                  \begin{align*}
                                    \bias_0(t+1)
                                    \quad
                                    &=\quad  1-\frac{\AUC_b(X(t),Z_0)}{\AUC_a(X(t), Z_0)}\\
                                    &\Stackrel{\eqref{def:zeta_noise}}{=}\quad  1-\frac{\AUC_b(X(t))+\zeta}{\AUC_a(X(t), Z_0)}\\
                                    &\Stackrel{\eqref{eq:auc_exp_for_0_and_1}}{=}\quad  1-\frac{\AUC_b(X(t))+\zeta}{\AUC_a(X(t))}.
                                    \intertext{Using \cref{fact:information_inequality}, we have that $\AUC_b(X(t),Z)\geq \AUC_b(X(t))$.
                                    Combining this inequality with the fact that in this case $\AUC_b(X(t), Z) < \AUC_a(X(t))$, it follows that  $\AUC_a(X(t)) > \AUC_b(X(t))$.
                                    Consequently}
                                    \bias_0(t+1) \quad&\Stackrel{}{=}\quad  1-\max\inbrace{\frac{\AUC_b(X(t))}{\AUC_a(X(t))}, \frac{\AUC_a(X(t))}{\AUC_b(X(t))}  }   -   \frac{\zeta}{\AUC_a(X(t))}\\
                                    &=\quad \bias(t)-\frac{\zeta}{\AUC_a(X(t))}.
                                  \end{align*}
                                  By \cref{thm:main_result,thm:result_for_advantaged_group} and the definition of $\zeta$, we have that $\zeta\geq 0$ and, hence, $\frac{\zeta}{\AUC_a(X(t))}\geq 0$.
                                  Therefore, in this case, \cref{thm:main_thm_adding_noise} follows for $\lambda=0$.

                                  \paragraph{\em Proof of \cref{coro:coro_adding_noise}.}

                                  \begin{proof}{Proof of \cref{coro:coro_adding_noise}}
                                    Let $\Delta \coloneqq \min\inbrace{\frac{\zeta}{\AUC_a(X(t))}, \bias(t)}$, where $\zeta$ is as defined in \cref{def:zeta_noise}.
                                    If $\AUC_a(X(t), Z)=\AUC_b(X(t), Z)$, then $\bias(t+1)=0$.
                                    Since $\bias(t)-\Delta\geq 0$, we are done.
                                    Henceforth, assume that $\AUC_a(X(t), Z) > \AUC_b(X(t), Z).$
                                    Using \cref{lem:continuity}, \cref{eq:invariance}, and the definition of $\bias_\lambda(t+1)$, it follows that
                                    $\bias_\lambda(t+1)$ is a continuous function of $\lambda$ over $[0,1]$.

                                    \smallskip
                                    \paragraph{Case A ($\AUC_b(X(t), Z)\geq \AUC_a(X(t))$:}
                                    In the proof of \cref{thm:main_thm_adding_noise} we showed that there exists a $\alpha\in [0,1]$
                                    such that $\bias_\alpha(t+1)=0$.
                                    Since
                                    $\bias_\alpha(t+1) = 0 \leq \bias(t)-\Delta$ and $\bias_1(t+1)= \bias(t+1)$, using the intermediate value theorem it follows that
                                    for any $\bias(t)-\Delta<v\leq \bias(t+1)$,
                                    there exists a $w\in [\alpha,1]$ such that
                                    $\bias_w(t+1)=v.$

                                    \smallskip
                                    \paragraph{Case B ($\AUC_b(X(t), Z) < \AUC_a(X(t))$:}
                                    In the proof of \cref{thm:main_thm_adding_noise} we showed that
                                    $\bias_0(t+1)\leq \bias(t)-\Delta$.
                                    Combining this with the fact that $\bias_1(t+1)= \bias(t+1)$, and using the intermediate value theorem it follows that
                                    for any $\bias(t)-\Delta<v\leq \bias(t+1)$,
                                    there exists a $w\in [0,1]$ such that
                                    $\bias_w(t+1)=v.$

                                  \end{proof}

                                  \paragraph{\em Proof of \cref{lem:continuity}}

                                  \begin{proof}{Proof of \cref{lem:continuity}}
                                    Consider the matrix
                                    $$\Lambda(\lambda)\coloneqq \begin{bmatrix}
                                    \Sigma & \lambda\cdot \rho \\
                                    \lambda \cdot \rho^\top & \lambda^2\cdot \sigma+ (1-\lambda)^2
                                    \end{bmatrix}.$$
                                    We claim that $\Lambda(\lambda)$ is invertible for all $0\leq \lambda\leq 1$.
                                    Suppose this claim is true.
                                    Then the lemma follows from \cref{eq:exp_auc_gamma}, because
                                    \begin{itemize}
                                      \item $\Phi\colon \R\to \R$ is a continuous function,
                                      \item $A\to A^{-1}$ is a continuous function over the set of all invertible matrices.
                                    \end{itemize}
                                    One can show this, e.g., using \cite[Theorem 18.2]{munkres2000topology}.

                                    It remains to prove the claim.
                                    Recall that in the binormal framework (\cref{def:binormal_framework}), we assume that for each $y\in \zo$ and $A=a$,
                                    $\Sigma_{y}$ and
                                    $\left[\begin{smallmatrix}
                                    \Sigma_{y} & \rho_y \\
                                    \rho_y^\top & \sigma_y
                                    \end{smallmatrix}\right]$ are invertible.
                                    Since these matrices are also positive semi-definite for each $y\in \zo$, it implies that $\Sigma$ and $\left[\begin{smallmatrix}
                                    \Sigma & \rho\\
                                    \rho^\top & \sigma
                                    \end{smallmatrix}\right]$ are invertible.
                                    (This is proved in the paragraph below Equation~\eqref{eq:dist_condition_2}).
                                    This proves that $\Lambda(\lambda)$ is invertible for $\lambda\in \zo$.
                                    Further, using that $\Lambda(1)$ and $\Sigma$ are invertible and \cite[Theorem 2.1(i)]{lu2002inverses}, we also get that
                                    $\frac{1}{\sigma-\rho^\top \Sigma^{-1}\rho} \neq 0$.
                                    Further, since $\Lambda(1)$ is positive definite, it must be that $\frac{1}{\sigma-\rho^\top \Sigma^{-1}\rho} > 0$ or equivalently
                                    \begin{align*}
                                      \sigma-\rho^\top \Sigma^{-1}\rho > 0.
                                      \yesnum\label{eq:tmp}
                                    \end{align*}
                                    This shows that for any $0<\lambda < 1$
                                    \begin{align*}
                                      \lambda^2\cdot \inparen{\sigma+\rho^\top \Sigma^{-1}\rho } + (1-\lambda)^2
                                      &> (1-\lambda)^2\tag{Using \cref{eq:tmp} and $\lambda\neq 0$}\\
                                      &> 0.\tag{Using $0<\lambda<1$}
                                    \end{align*}
                                    Now \cite[Theorem 2.1(i)]{lu2002inverses} implies that $\Lambda(\lambda)$ is invertible for all $0<\lambda < 1$.
                                  \end{proof}

                                  \subsubsection{Proof of \cref{thm:pareto_dom}}

                                  \begin{proof}{Proof of \cref{thm:pareto_dom}}
                                    To show the existence of the claimed $\lambda$, it suffices to show that $\bias_{\alpha,\beta}(t+1)\geq \bias(t)-\Delta$ and use \cref{coro:coro_adding_noise}.
                                    In the proof of \cref{coro:coro_adding_noise}, we set
                                    \begin{align*}
                                      \Delta \coloneqq \min\inbrace{\frac{\AUC_b(X(t),Z)-\AUC_b(X(t))}{\AUC_a(X(t))}, \bias(t)}.
                                      \yesnum\label{def:delta_2}
                                    \end{align*}

                                    \paragraph{Case A ($\AUC_b(X(t),Z) \geq \AUC_a(X(t))$):}
                                    In this case, $\bias(t)=1-\frac{\AUC_b(X(t))}{\AUC_a(X(t))}$.
                                    Using this and the definition of $\Delta$, one can verify that in this case $\Delta=\bias(t)$.
                                    Hence, $\bias_{\alpha,\beta}(t+1)\geq 0=\bias(t)-\Delta$.

                                    \paragraph{Case B ($\AUC_b(X(t),Z) < \AUC_a(X(t))$):}
                                    We will use the following lemma:
                                    \begin{lemma}\label{lem:range_of_AUC}
                                      For all $0\leq\alpha,\beta\leq 1$ and both groups $g\in \ab$,
                                      $\AUC_g(X(t))\leq \AUC_g(X(t),Z_{\alpha,\beta})\leq \AUC_g(X(t),Z)$.
                                    \end{lemma}
                                    \begin{proof}{Proof of \cref{lem:range_of_AUC}}
                                      Note that conditioned on $A=g$, $Z_{\alpha,\beta}$ is a linear combination of $Z$ and $N$.
                                      Hence, $\AUC_g(X,Z_{\alpha,\beta})\leq \AUC_g(X,Z,N)$ (see \cref{fact:information_inequality}).
                                      Further, since $N$ is independent of $Y$, $X$, $Z$, and $A$, using \cref{fact:optimal_auc_expression} it follows that $\AUC_g(X,Z,N)=\AUC_g(X,Z)$; this is intuitively true because $N$ does not give any information about $Y$.
                                      This proves the upper bound.
                                      The lower bound follows from \cref{fact:information_inequality} as $X(t)$ is a projection of $\inparen{X(t), Z_{\alpha,\beta}}$ that omits the last coordinate.
                                      This completes the proof of \cref{lem:range_of_AUC}.
                                    \end{proof}
                                    Using \cref{lem:range_of_AUC}, it follows that in this case
                                    $$\AUC_a(X(t),Z_{\alpha,\beta})\geq \AUC_a(X(t))>\AUC_b(X(t),Z)\geq \AUC_b(X(t),Z_{\alpha,\beta}).$$
                                    Hence,
                                    \begin{align*}
                                      \bias_{\alpha,\beta}(t+1)
                                      &=1-\frac{\AUC_b(X(t),Z_{\alpha,\beta})}{\AUC_a(X(t),Z_{\alpha,\beta})}\\
                                      &\geq 1-\frac{\AUC_b(X(t),Z)}{\AUC_a(X(t))}\tag{Using \cref{lem:range_of_AUC} to lower bound the denominator and upper bound the numerator}\\
                                      &=\bias(t) - \frac{\AUC_b(X(t),Z)-\AUC_b(X(t))}{\AUC_a(X(t))}.
                                      \yesnum\label{eq:ineq_case_2_pareto}
                                    \end{align*}
                                    Further, $\bias_{\alpha,\beta}(t+1)\geq 0=\bias(t)-\bias(t)$.
                                    Combining this inequality with \cref{eq:ineq_case_2_pareto} it follows that
                                    $\bias_{\alpha,\beta}(t+1)\geq 0=\bias(t)-\Delta$.

                                    This completes the proof of the existence of $\lambda$ claimed in \cref{thm:pareto_dom}.
                                    It remains to prove the weak Pareto-optimality condition.
                                    Toward this observe that
                                    \begin{align*}
                                      \AUC_b(X(t), Z_{\alpha,\beta})
                                      \qquad\Stackrel{\rm\cref{lem:range_of_AUC}}{\leq}\qquad \AUC_b(X(t), Z)
                                      \quad \Stackrel{\eqref{eq:invariance}}{=}\quad \AUC_b(X(t), Z_\lambda).
                                      \yesnum\label{eq:pareto_pt_1}
                                    \end{align*}
                                    It remains to prove that $\AUC_a(X(t), Z_{\alpha,\beta})\leq \AUC_a(X(t), Z_\lambda).$

                                    To prove this, we will use the fact that the $\lambda$ constructed in the proof of \cref{coro:coro_adding_noise} satisfies
                                    \begin{align*}
                                      \AUC_{a}(X(t), Z_\lambda)\geq  \AUC_{b}(X(t), Z_\lambda).
                                      \yesnum\label{eq:pareto_condition}
                                    \end{align*}

                                    \paragraph{Case A ($\AUC_a(X(t),Z_{\alpha,\beta}) \geq \AUC_b(X(t),Z_{\alpha,\beta})$):}
                                    Note that if $\bias_{\alpha,\beta}(t+1)=\bias_{\lambda}(t+1)=1$, then
                                    $$\AUC_a(X(t),Z_{\alpha,\beta})=\AUC_b(X(t),Z_{\alpha,\beta})\quad \Stackrel{\eqref{eq:pareto_pt_1}}{\leq} \quad \AUC_b(X(t),Z_{\lambda})=\AUC_a(X(t),Z_{\lambda}).$$
                                    Hence, we are done if $\bias_{\alpha,\beta}(t+1)=\bias_{\lambda}(t+1)=1$.
                                    Thus, assume that $\bias_{\alpha,\beta}(t+1)=\bias_{\lambda}(t+1)<1$.
                                    Because of the assumption in this case (i.e., Case A), we have that $\bias_{\alpha,\beta}(t+1)
                                    =1-\frac{\AUC_b(X(t),Z_{\alpha,\beta})}{\AUC_a(X(t),Z_{\alpha,\beta})}.$
                                    Hence,
                                    \begin{align*}
                                      \AUC_a(X(t),Z_{\alpha,\beta})
                                      \qquad
                                      &=\qquad \frac{\AUC_b(X(t),Z_{\alpha,\beta})}{1-\bias_{\alpha,\beta}(t+1)}\tag{Using that $\bias_{\alpha,\beta}(t+1)<1$}\\
                                      &\Stackrel{\rm\cref{lem:range_of_AUC}}{\leq}\qquad \frac{\AUC_b(X(t),Z)}{1-\bias_{\alpha,\beta}(t+1)}\\
                                      &\Stackrel{}{=}\qquad \frac{\AUC_b(X(t),Z)}{1-\bias_{\lambda}(t+1)}
                                      \tag{Using that, by construction, $\bias_{\lambda}(t+1)=\bias_{\alpha,\beta}(t+1)$}\\
                                      &\Stackrel{\eqref{eq:invariance}}{\leq}\qquad \frac{\AUC_b(X(t),Z_\lambda)}{1-\bias_{\lambda}(t+1)}\\
                                      &\Stackrel{}{=}\qquad \AUC_a(X(t),Z_\lambda).\tag{Using \cref{eq:pareto_condition} and definition of $\bias_{\lambda}(t+1)$}
                                    \end{align*}

                                    \paragraph{Case B ($\AUC_a(X(t),Z_{\alpha,\beta}) < \AUC_b(X(t),Z_{\alpha,\beta})$):}
                                    In this case, %
                                    \begin{align*}
                                      \AUC_a(X(t),Z_{\alpha,\beta}) < \AUC_b(X(t),Z_{\alpha,\beta})
                                      \quad \Stackrel{\eqref{eq:pareto_pt_1}}{\leq} \quad \AUC_b(X(t),Z_{\lambda})
                                      \quad \Stackrel{\eqref{eq:pareto_condition}}{\leq} \quad
                                      \AUC_a(X(t),Z_{\lambda}).
                                    \end{align*}
                                  \end{proof}

                                  \noindent While this noisy \fairAUC\ \algo\ prevents bias from increasing round to round, introducing noise reduces the AUC of the would-be advantaged group. Therefore the original \fairAUC\ \algo\ Pareto dominates the noisy \fairAUC\ \algo\ in terms of AUCs.

                                  \section{Additional Empirical Results}

                                  \subsection{fairAUC with a Bias Penalty Term}\label{appendix:penalty_bias}
                                  Another potential strategy to prevent bias from increasing is to add a penalty term for bias. However, we find that such a strategy is unable to ensure bias does not increase after acquiring a new feature.
                                  Figures \ref{figure:bias_penalty} (Left) and (Right) show the performance of \fairAUC\ with different levels of penalty placed on bias. In Figure \ref{figure:bias_penalty} (Left), equal weights are placed on AUC and the bias penalty and we can see that the bias still can increase after the acquisition of a feature with \fairAUC. In Figure \ref{figure:bias_penalty} (Right), greater weight is placed on the bias penalty term. Placing a large penalty on bias essentially reduces the \algo\ to the \minBias\ procedure.

                                  \begin{figure}[htbp]
                                    \caption{(Left) AUC and Bias Penalty Equally Weighted. (Right) 33\% Weight on AUC and 67\% Weight on Bias Penalty.}
                                    \begin{minipage}{0.5\textwidth}
                                      \begin{center}
                                        \includegraphics[height=5.5cm]{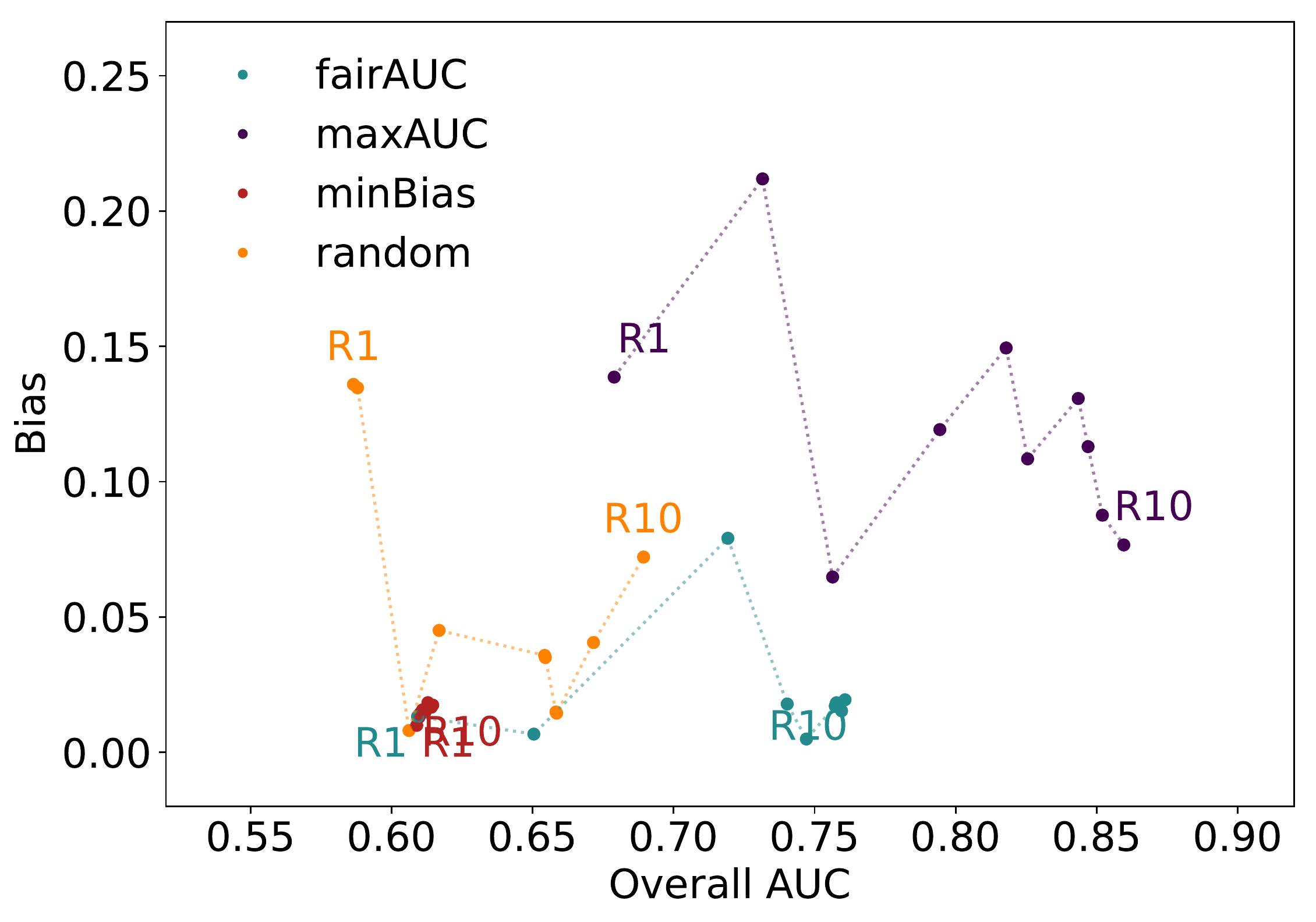}
                                      \end{center}
                                    \end{minipage}
                                    \begin{minipage}{0.5\textwidth}
                                      \begin{center}
                                        \includegraphics[height=5.5cm]{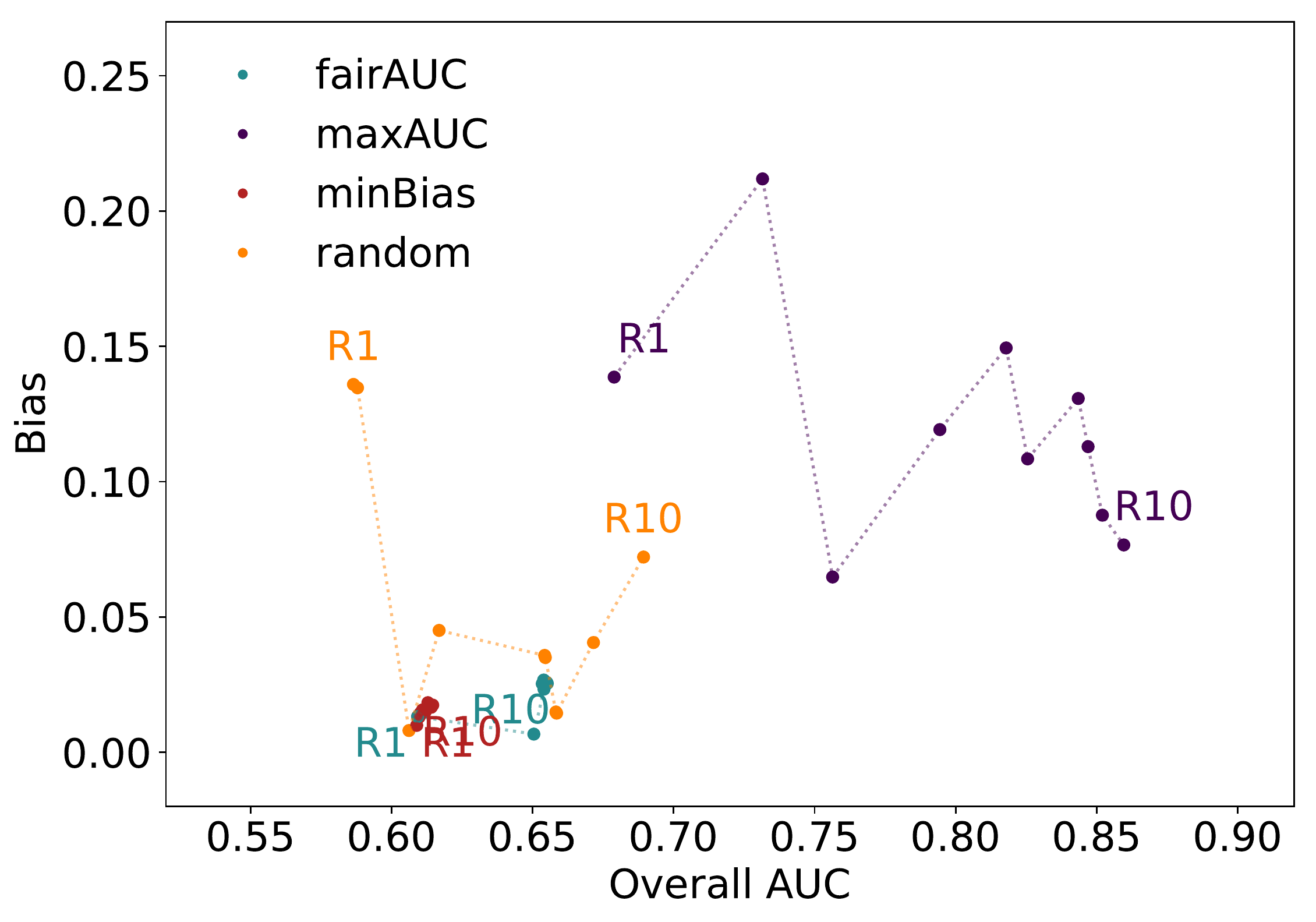}
                                      \end{center}
                                    \end{minipage}
                                    \footnotesize \textit{Note:} R1 represents Round 1 of feature acquisition and R10 represents Round 10.
                                    \label{figure:bias_penalty}
                                  \end{figure}

                                  \noindent Noisy \fairAUC\ prevents bias from increasing each iteration but a penalty-based strategy is not sufficient.

                                  \subsection{Empirical Results When the Protected Attribute is Not Used in Classification}\label{appendix:single_classifier}

                                  \fairAUC\ does not require the use of the protected attribute during classification. In the main body of the paper, the protected attribute was used during classification. Here we show how \fairAUC\ performs when the protected attribute is not used.

                                  \subsubsection{Synthetic Data Analysis}\label{appendix:synthetic_single}

                                  As can be seen in Figures \ref{figure:rounds_sim_single_eqsep}, \ref{figure:tradeoff_sim_single_eqsep}, and \ref{figure:pareto_sim_single_eqsep}, the same pattern of results continues to hold when the protected attribute is not used in classification. Compared to \maxAUC, the \fairAUC\ \algo\ greatly decreases the bias between the two groups and improves the AUC of the disadvantaged group. \minBias\ also greatly reduces bias but at the cost of learning about either group. While \fairAUC\ decreases bias relative to \maxAUC, it does trade off AUC in the process as shown in Figure \ref{figure:tradeoff_sim_single_eqsep}.

                                  \begin{figure}[htbp]
                                    \centering
                                    \caption{Group-wise AUCs over Feature Augmentation Rounds without Protected Attribute}
                                    \includegraphics[width=9cm]{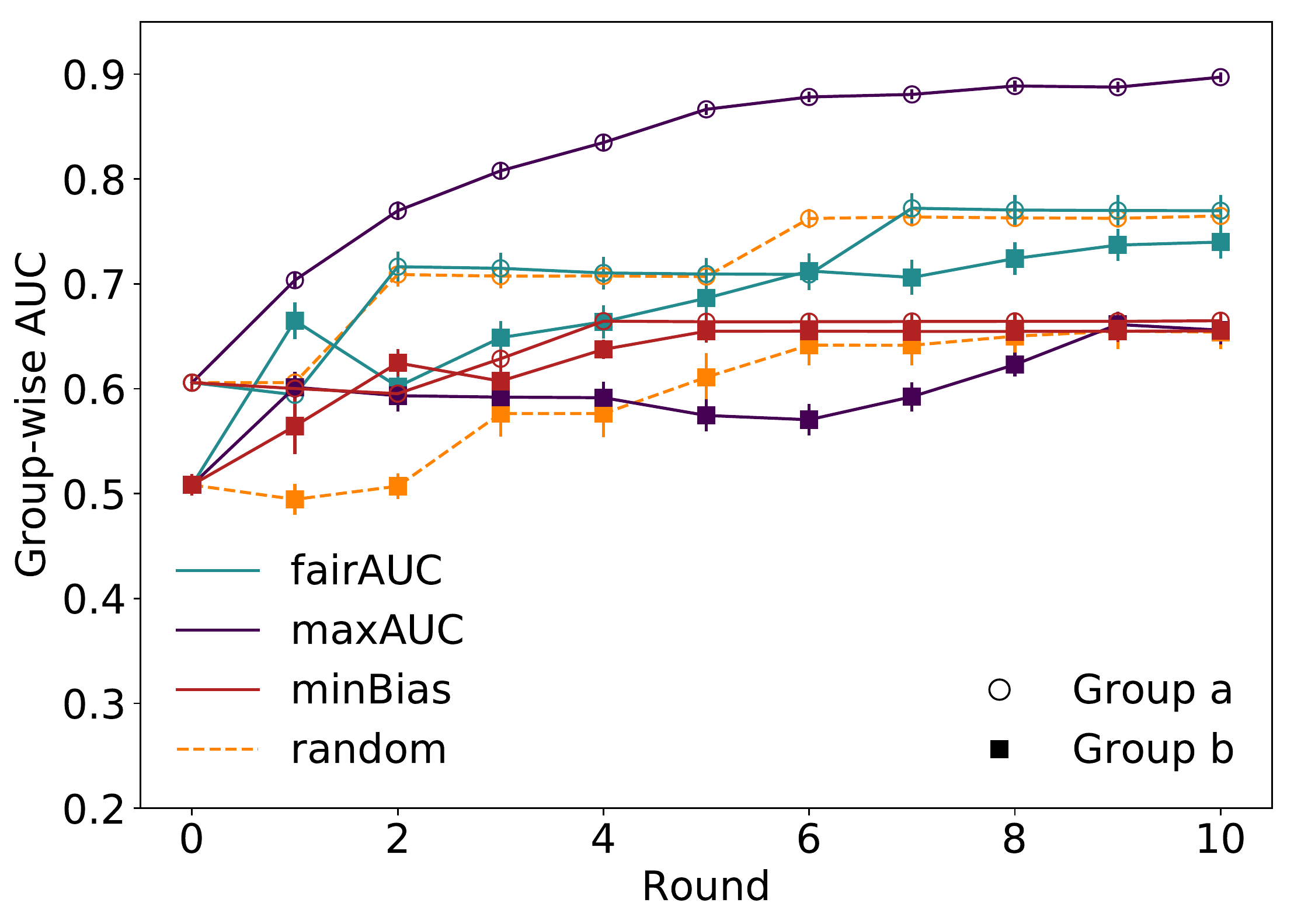}
                                    \label{figure:rounds_sim_single_eqsep}
                                  \end{figure}

                                  \begin{figure}[htbp]
                                    \centering
                                    \caption{Accuracy-Fairness Tradeoff without Protected Attribute}
                                    \includegraphics[width=9cm]{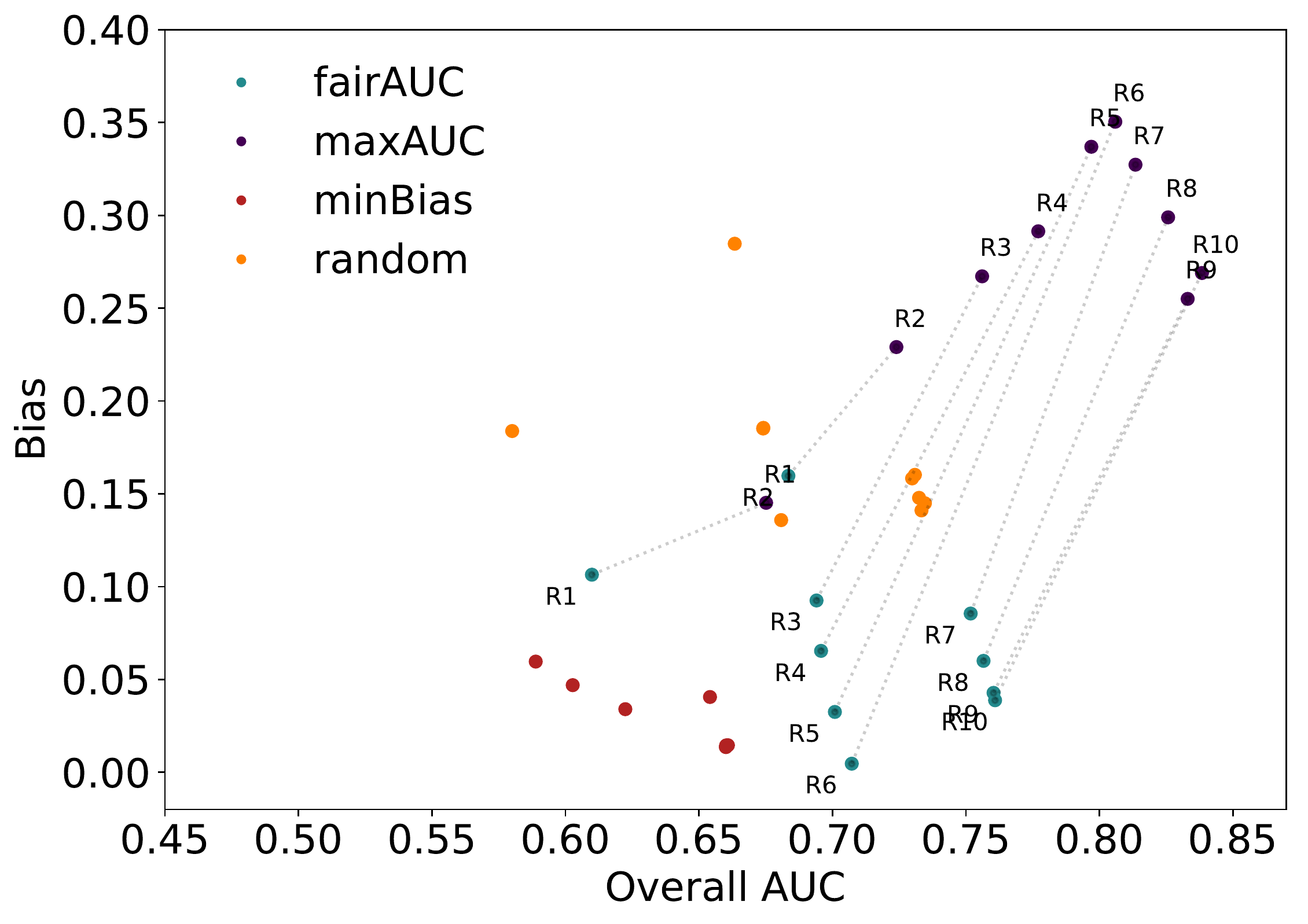}
                                    \label{figure:tradeoff_sim_single_eqsep}
                                  \end{figure}

                                  \begin{figure}[htbp]
                                    \centering
                                    \caption{Pareto Frontier without Protected Attribute}
                                    \includegraphics[width=9cm]{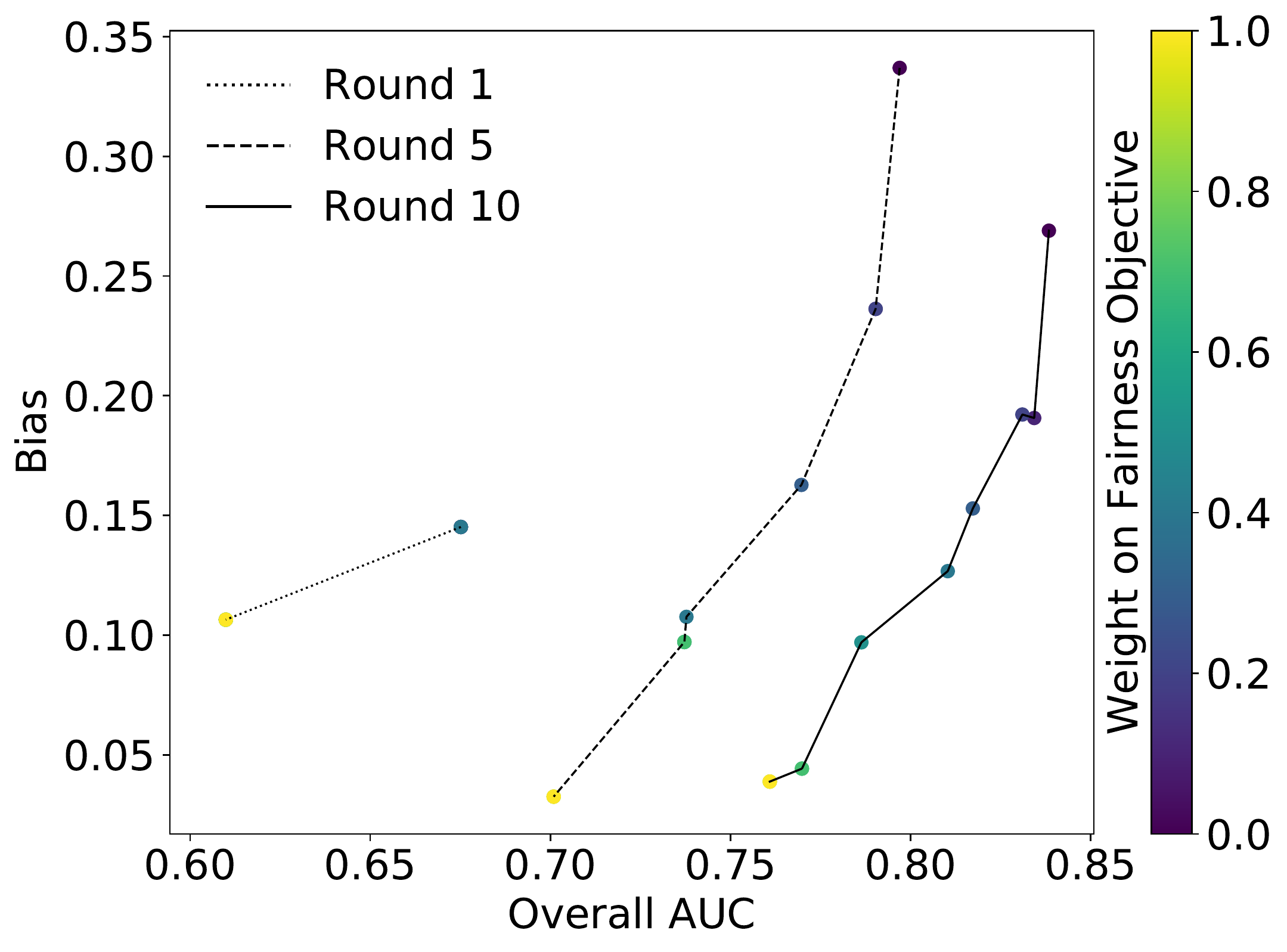}
                                    \label{figure:pareto_sim_single_eqsep}
                                  \end{figure}

                                  \newpage
                                  \subsubsection{COMPAS Data Analysis}\label{appendix:compas_single}

                                  We next show the performance of the various \algos\ when the protected attribute is not used in classification on the COMPAS dataset in Figure \ref{figure:compas_single}. The \algos\ follow the same pattern as when the protected attribute is used in classification (as seen in Figure \ref{figure:compas_separate} in the paper).

                                  \begin{figure}[ht]
                                    \centering
                                    \caption{Predicting Violent Recidivism without Protected Attribute (Age)}
                                    \includegraphics[width=9cm]{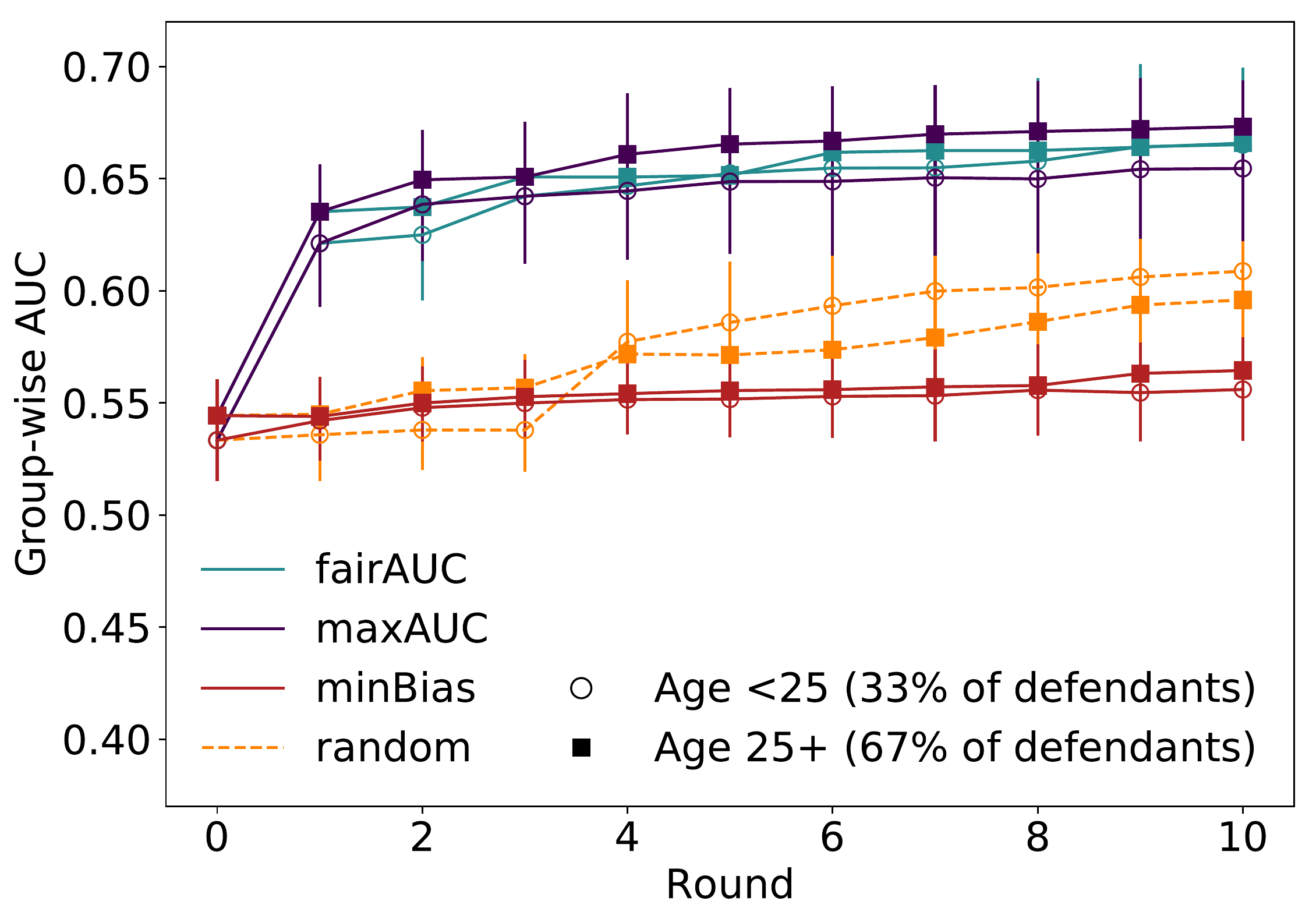}
                                    \label{figure:compas_single}
                                    \vspace{-5mm}
                                  \end{figure}

                                  \vspace{-3mm}

                                  \subsection{Synthetic Data with Multivariate Gamma Distribution}\label{appendix:other_data_distributions}
                                  In the main analysis, we generate data that follow multivariate normal distributions. To test the robustness of \fairAUC\ to other data generating processes, we generate data that follow different multivariate gamma distributions. We choose the gamma distribution because it can be flexibly parametrized using its shape and rate parameters. The objective here is to investigate whether our theoretical guarantees, which hold when the features are generated from a multivariate normal distribution, continue to hold when the data is instead generated from a gamma distribution.

                                  Gamma distributions are defined by the probability density function:
                                  \begin{equation*}
                                    f(x;\alpha,\beta) = \frac{x^{\alpha-1}e^{-\beta x}\beta^\alpha}{(\alpha-1)!} \text{  for } x,\alpha,\beta>0.
                                  \end{equation*}

                                  \noindent Parameterized by two parameters, $\alpha$ (shape) and $\beta$ (rate), gamma distributions can take on a wide variety of distribution shapes. The mean of a gamma distribution is $\alpha/\beta$ and the variance is $\alpha/\beta^2$. We generate data for two groups such that they have the same shape and rate parameters but one group captures a much larger fraction of observations, mirroring the procedure followed in the main paper for multivariate normal data. The following procedure is used to generate the data:

                                  \begin{enumerate}
                                    \item For each group, generate $p$ informative features that have the same class-conditional variance but different class-conditional means and a random correlation structure \citep{pourahmadi2015distribution}
                                    \item For each group, generate $q$ uninformative features that have no difference in class-conditional means or variances and which have a random correlation structure
                                  \end{enumerate}

                                  \noindent Below we show the results of the \fairAUC, \maxAUC, \minBias, and \random\ \algos\ on two datasets generated by the above procedure with different shape and rate combinations. Table \ref{table:gamma_parameters} lists the various parameters used and why we choose those parameters. We use different distribution shapes to test the robustness of \fairAUC. We train the four \algos\ using separate (logistic regression) classifiers for each group. Figures \ref{figure:gamma1} and \ref{figure:gamma2} compare the group-wise AUCs over feature acquisition rounds.

                                  \begin{table}
                                    \footnotesize
                                    \centering
                                    \caption{Parameters used for Synthetic Gamma Data}
                                    \begin{tabular}{lll}
                                      \hline \hline
                                      Parameter & Value & Reason for Choice \\
                                      \hline
                                      Fraction of observations in group $a$ & 0.95 & Large group imbalance as is often seen in real data \\
                                      Number of observations & 20,000 & Larger datasets give us more precise AUC results \\
                                      Dataset 1: \\
                                      \,\,\,\,Negative class $\alpha$ and $\beta$ & 1, $\sqrt{2}$ & Hump-shaped with long right tail so distributions \\
                                      \,\,\,\,Positive class $\alpha$ and $\beta$ & 2, 2 & different from normal but still hump-shaped \\
                                      Dataset 2: \\
                                      \,\,\,\,Negative class $\alpha$ and $\beta$ & 1, 0.5 & Negative class exponential distribution-shaped so \\
                                      \,\,\,\,Positive class $\alpha$ and $\beta$ & 2, $\sqrt{0.5}$ & deviates more greatly from normal distribution;\\
                                      & & positive class hump-shaped with long right tail \\
                                      \hline \hline
                                    \end{tabular}
                                    \label{table:gamma_parameters}
                                  \end{table}

                                  \begin{figure}
                                    \centering
                                    \caption{Performance for Dataset 1}
                                    \includegraphics[width=9cm]{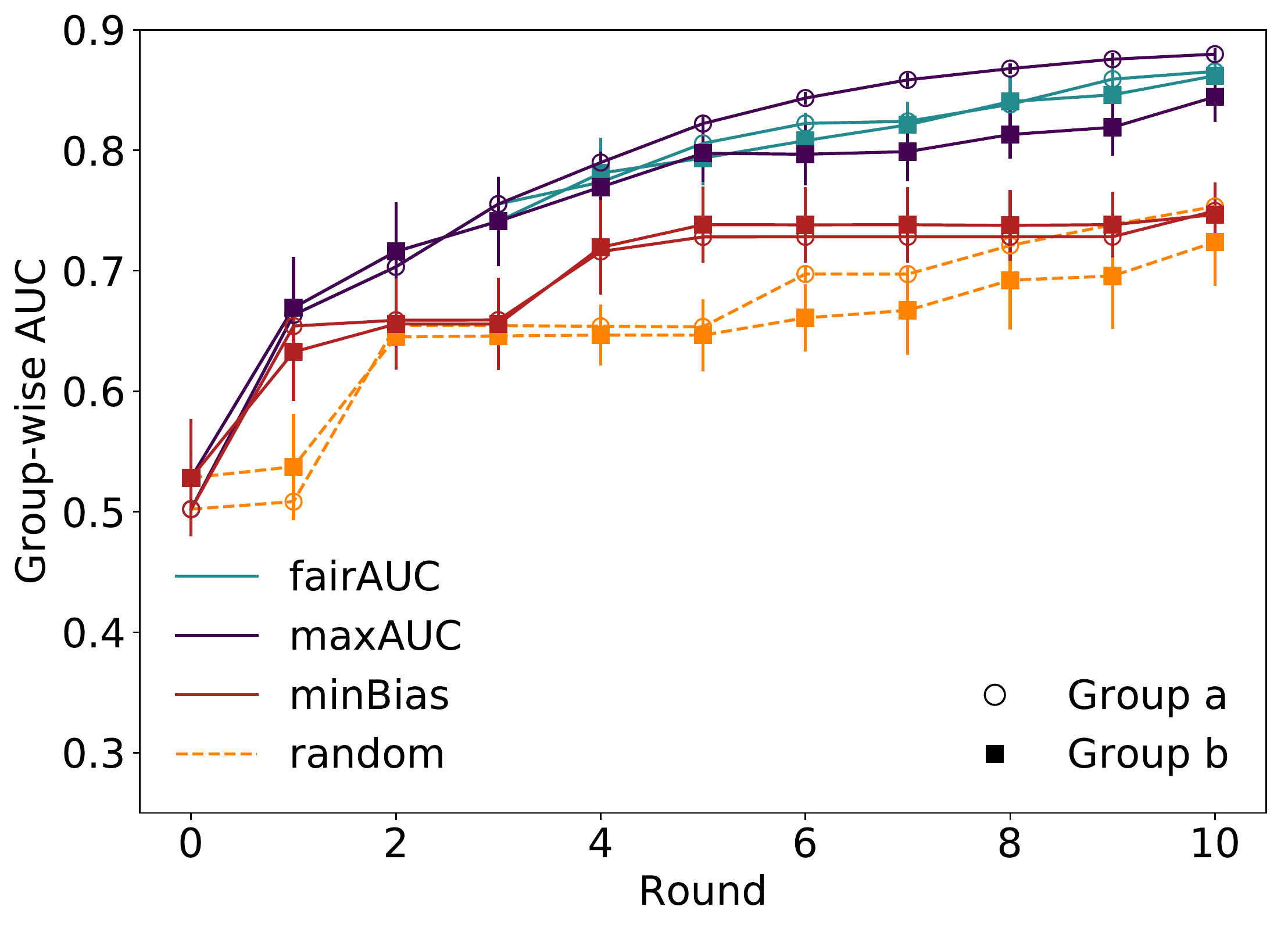}
                                    \label{figure:gamma1}
                                    \vspace{-3mm}
                                  \end{figure}

                                  \begin{figure}
                                    \centering
                                    \caption{Performance for Dataset 2}
                                    \includegraphics[width=9cm]{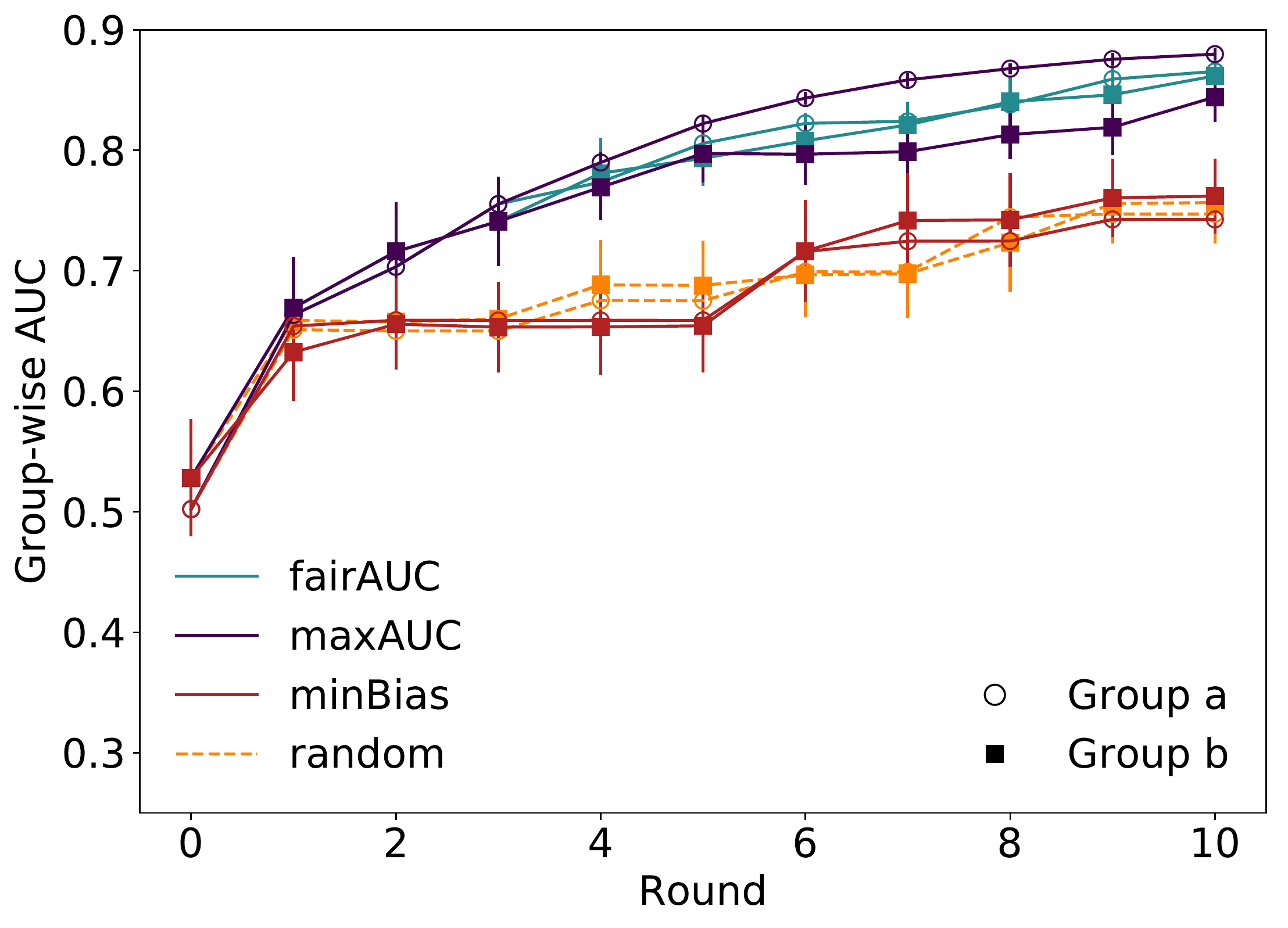}
                                    \label{figure:gamma2}
                                    \vspace{-3mm}
                                  \end{figure}

                                  \noindent The different procedures follow the same pattern observed with the multivariate normal distributions. Specifically, we observe that fairAUC reduces the bias to a very small degree, where the two groups $a$ and $b$ have similar AUCs, compared to maxAUC, where group $a$ has a higher AUC, and the overall bias is high. \fairAUC\ selects features that improve the AUC for the disadvantaged group relative to \maxAUC{}. The minBias algorithm obtains low bias, but obtains a really low AUC for each group.

                                  \newpage
                                  \subsection{Nonlinear Classifiers}\label{section:nonlinear_classification}
                                  The theoretical guarantees obtained in the paper hold for generalized linear models (GLMs). The experiments with synthetic data and empirical application with real-world data in the main body of the paper use logistic regression as the classifier. Here, we examine the robustness of the results to using nonlinear classifiers, specifically \textit{random forest} and \textit{support vector machine} (SVM) with a nonlinear (rbf) kernel. We use the protected attribute during classification.

                                  \paragraph{Random Forest:} Figure \ref{figure:normal_separate_forest} shows the performance of the four procedures on multivariate normal data \citep{guyon2003design} using a random forest classifier. We limit the maximum depth of the forest to three layers to prevent model overfitting. The pattern of results is consistent with the results from the logistic regression model, which are detailed in Figure \ref{figure:synthetic_results}. For example, \fairAUC\ selects features that improve the AUC of the disadvantaged group and in doing so reduces bias relative to \maxAUC. \minBias\ fails to acquire informative features.

                                  \begin{figure}[ht]
                                    \centering
                                    \caption{Performance using Random Forest as Classifier}
                                    \includegraphics[width=9.5cm]{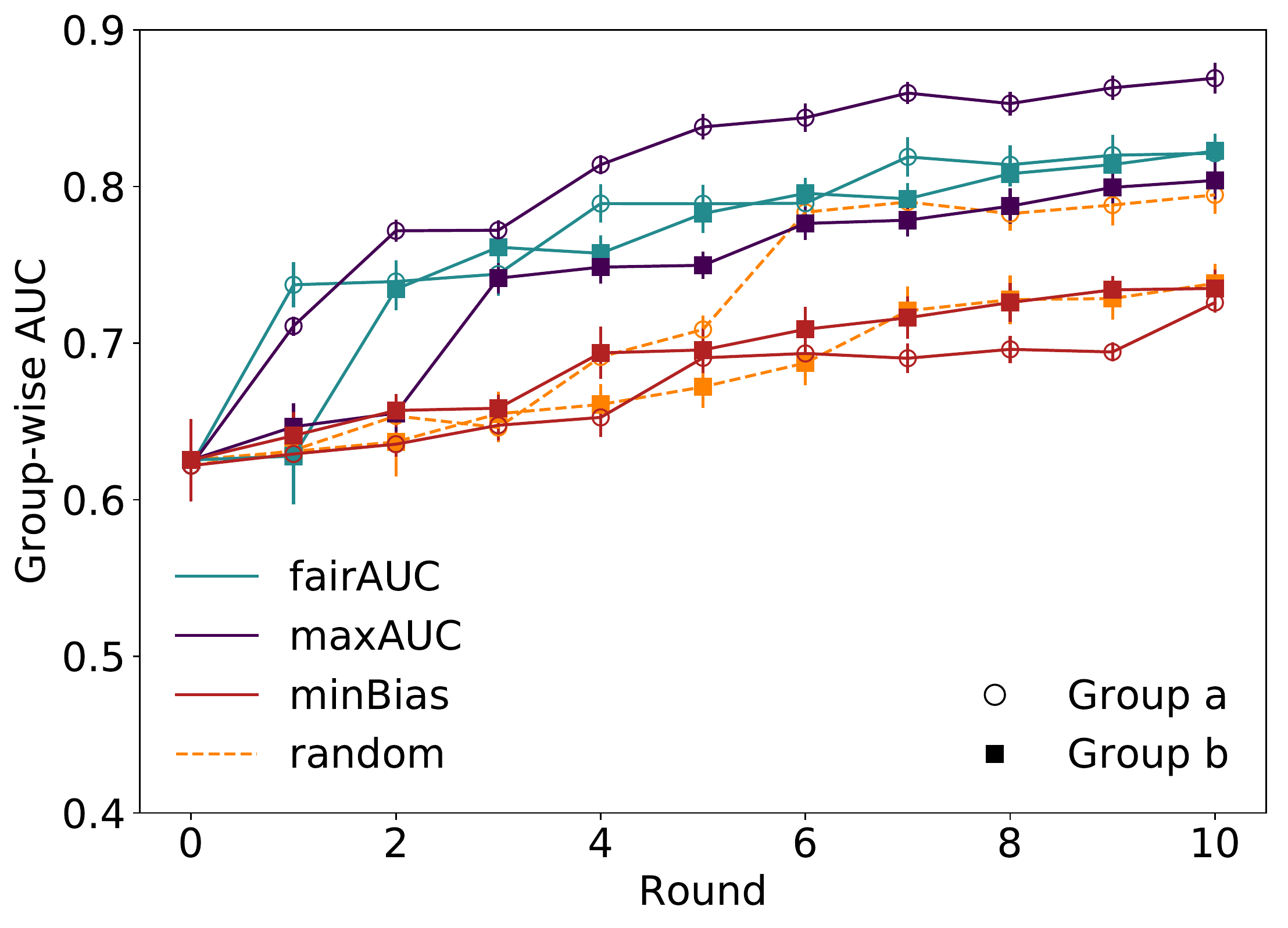}
                                    \label{figure:normal_separate_forest}
                                  \end{figure}

                                  \noindent Next, we compare which features are acquired depending on whether logistic regression or random forest is used in Table \ref{table:logistic_vs_forest}. Note that logistic regression and FLD select the same features in the same order on this data set over the first ten rounds. For nine out of the first ten rounds, logistic regression and random forest acquire the same features. We highlight in bold which features differ between the two classifiers. Random forest generates slightly higher AUCs.

                                  \begin{table}[htbp]
                                    \footnotesize
                                    \centering
                                    \caption{Using Logistic Regression vs. Random Forest for Classification}
                                    \begin{tabular}{c|cccccc|cccccc}
                                      \hline \hline
                                      & \multicolumn{6}{c}{Logistic Regression} &  \multicolumn{6}{c}{Random Forest} \\
                                      Round & Feature & $\AUC_a$ & $\AUC_b$ & $\AUC_{\rm All}$ & Bias & Disadv & Feature & $\AUC_a$ & $\AUC_b$ & $\AUC_{\rm All}$ & Bias & Disadv \\
                                      &&&&& & Group & & &&& & Group \\
                                      \hline
                                      0 &    & 0.6058 & 0.5083 & 0.5846 & 0.1610 & b &
                                      & 0.6217 & 0.6253 & 0.6127 & 0.0057& a \\ %
                                      1 & 47 & 0.6063 & 0.6946 & 0.6368 & 0.1271 & a &
                                      18 & 0.7372 &  0.6278 & 0.7056 & 0.1484 & b \\
                                      2 & 18 & 0.7297 & 0.6976 & 0.7205 & 0.0439 & b &
                                      47 & 0.7393 & 0.7344 & 0.7357 & 0.0066 & b \\
                                      3 & 13 & 0.7299 & 0.7266 & 0.7291 & 0.0045 & b &
                                      13 & 0.7441 & 0.7612 & 0.7482 & 0.0225 & a \\
                                      4 & 48 & 0.7300 & 0.7486 & 0.7359 & 0.0249 & a &
                                      26 & 0.7891 & 0.7573 & 0.7797 & 0.0403 & b \\
                                      5 & 26 & 0.7878 &  0.7494 & 0.7771 & 0.0488 & b &
                                      48 & 0.7890 & 0.7827 & 0.7868 & 0.0079 & b \\
                                      6 & 30 & 0.7879 & 0.7736 & 0.7840 &  0.0181 & b &
                                      30 & 0.7893 & 0.7956 & 0.7909 & 0.0079 & a \\
                                      7 & 15 & 0.7879 & 0.7861 & 0.7877 & 0.0023 & b &
                                      20 & 0.8189 & 0.7921 & 0.8095 & 0.0327 & b \\
                                      8 & \textbf{49} & 0.7880 & 0.8048 & 0.7934 & 0.0209 & a &
                                      15 & 0.8139 & 0.8082 & 0.8119 & 0.0070 & b \\
                                      9 & 20 &  0.8209 & 0.8049 & 0.8163 & 0.0194 & b &
                                      12 & 0.8199 & 0.8140 & 0.8175 & 0.0073 & b \\
                                      10 & 12 & 0.8209 & 0.8130 & 0.8186 & 0.0095 & b &
                                      \textbf{38} & 0.8212 & 0.8227 & 0.8211 & 0.0018 & a \\
                                      \hline \hline
                                    \end{tabular}
                                    \label{table:logistic_vs_forest}
                                  \end{table}

                                  \paragraph{SVM with Nonlinear Kernel:} Figure \ref{figure:normal_separate_svm} shows the performance of the four procedures on multivariate normal data \citep{guyon2003design} using SVM with a nonlinear kernel function. Again, the pattern of results is consistent with the logistic regression results. Table \ref{table:logistic_vs_svm} compares which features are acquired depending on whether logistic regression or nonlinear SVM is used. Nonlinear SVM results in much higher AUC values than logistic regression but both end up acquiring roughly the same set of features in the first ten rounds (eight out of ten features overlap).

                                  \begin{figure}[ht]
                                    \centering
                                    \caption{Performance using Nonlinear SVM as Classifier}
                                    \includegraphics[width=9.5cm]{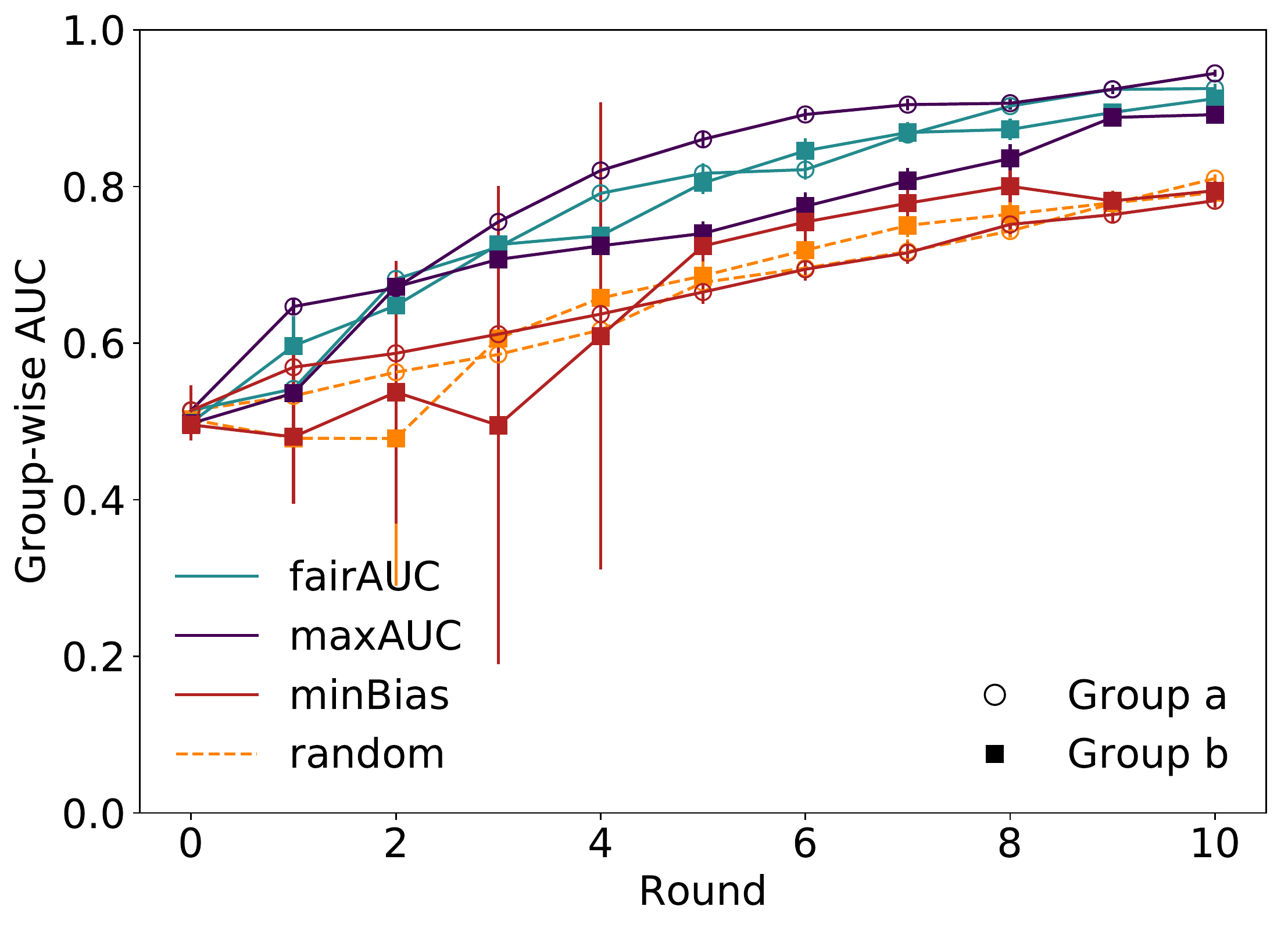}
                                    \label{figure:normal_separate_svm}
                                  \end{figure}

                                  \begin{table}[htbp]
                                    \footnotesize
                                    \centering
                                    \caption{Using Logistic Regression vs. Nonlinear SVM for Classification}
                                    \begin{tabular}{c|cccccc|cccccc}
                                      \hline \hline
                                      & \multicolumn{6}{c}{Logistic Regression} &   \multicolumn{6}{c}{Nonlinear SVM} \\
                                      Round & Feature & $\AUC_a$ & $\AUC_b$ & $\AUC_{\rm All}$ & Bias & Disadv & Feature & $\AUC_a$ & $\AUC_b$ & $\AUC_{\rm All}$ & Bias & Disadv \\
                                      &&&&& & Group & & &&& & Group \\
                                      \hline
                                      0 &    & 0.6058 & 0.5083 & 0.5846 & 0.1610 & b &
                                      & 0.5140 & 0.4983 & 0.5122 & 0.0305 & b \\
                                      1 & 47 & 0.6063 & 0.6946 & 0.6368 & 0.1271 & a &
                                      47 & 0.5415 & 0.5965 & 0.5610 & 0.0922 & a \\
                                      2 & 18 & 0.7297 & 0.6976 & 0.7205 & 0.0439 & b &
                                      18 & 0.6819 & 0.6479 & 0.6657 & 0.0498 & b \\
                                      3 & 13 & 0.7299 & 0.7266 & 0.7291 & 0.0045 & b &
                                      13 & 0.7225 & 0.7257 & 0.7221 & 0.0045 & a \\
                                      4 & \textbf{48} & 0.7300 & 0.7486 & 0.7359 & 0.0249 & a &
                                      26 & 0.7912 & 0.7372 & 0.7713 & 0.0683 & b \\
                                      5 & 26 & 0.7878 &  0.7494 & 0.7771 & 0.0488 & b &
                                      49 & 0.8168 & 0.8045 & 0.8137 & 0.0151 & b \\
                                      6 & \textbf{30} & 0.7879 & 0.7736 & 0.7840 &  0.0181 & b &
                                      \textbf{38} & 0.8215 & 0.8456 & 0.8290 & 0.0285 & a \\
                                      7 & 15 & 0.7879 & 0.7861 & 0.7877 & 0.0023 & b &
                                      \textbf{17} & 0.8663 & 0.8690 & 0.8671 & 0.0031 & a \\
                                      8 & 49 & 0.7880 & 0.8048 & 0.7934 & 0.0209 & a &
                                      20 & 0.9026 &  0.8727 & 0.8938 & 0.0331 & b \\
                                      9 & 20 &  0.8209 & 0.8049 & 0.8163 & 0.0194 & b &
                                      12 & 0.9239 & 0.8946 & 0.9157 & 0.0317 & b \\
                                      10 & 12 & 0.8209 & 0.8130 & 0.8186 & 0.0095 & b
                                      & 15 & 0.9251 & 0.9122 & 0.9214 & 0.0139 & b \\
                                      \hline \hline
                                    \end{tabular}
                                    \label{table:logistic_vs_svm}
                                  \end{table}

                                  The results suggest that the FLD heuristic for feature acquisition is robust to using different classifiers.

                                  \subsection{Acquiring Multiple Features at a Time}\label{appendix:multiple_variables}
                                  The \fairAUC\ algorithm can certainly be altered to acquire more than one feature at a time. The benefit to acquiring more than one feature at a time would be that the algorithm is then able to incorporate the covariance between the auxiliary features. However, acquiring multiple features also increases the complexity of the problem and decreases the flexibility in determining which group is the disadvantaged group. If there are $n$ auxiliary features and we plan to acquire $k$ features at a time, there are $\binom{n}{k}$ possible combinations of features to acquire, which grows faster than $2^k$ for $k<\frac{n}{2}$.

                                  We show the AUC and bias when \fairAUC\ is used to collect two features each feature acquisition round. We use Equation \ref{eq:fld_auc} to calculate the AUC associated with the acquisition of each possible pair of auxiliary features. Table \ref{table:multiple_vars} compares the AUCs and bias that result from acquiring one feature at a time versus acquiring two features at a time on the synthetic dataset generated using \cite{guyon2003design}. Group $a$ begins as the disadvantaged group. First, we observe that nine out of the ten features first acquired are the same but in slightly different orders. Second, acquiring two features at a time results in slightly higher overall AUCs across both groups. Third, acquiring one feature at a time results in lower bias on average (0.0395 for one feature vs. 0.0704 for two features). This lower bias is likely because choosing one feature at a time is less likely to overshoot on AUC since it allows flexibility in determining which is the disadvantaged group in each round. Increasing the number of features to acquire $k$ would further decrease flexibility and allow for greater overshooting of AUC.

                                  \begin{table}[htbp]
                                    \footnotesize
                                    \centering
                                    \caption{Selecting One Feature at a Time vs. Two at a Time using fairAUC}
                                    \begin{tabular}{c|cccccc|cccccc}
                                      \hline \hline
                                      & \multicolumn{6}{c}{One Feature} & \multicolumn{6}{c}{Two Features} \\
                                      Round & Feature & $\AUC_a$ & $\AUC_b$ & $\AUC_{\rm All}$ & Bias & Disadv & Features & $\AUC_a$ & $\AUC_b$ & $\AUC_{\rm All}$ & Bias & Disadv \\
                                      &&&&& & Group & & &&& & Group \\
                                      \hline
                                      0 &    & 0.5057 & 0.5810 & 0.5332 & 0.1296 & a &
                                      & 0.5057 & 0.5810 & 0.5332 & 0.1296 & a \\
                                      1 & 17 & 0.6825 & 0.6298 & 0.6679 & 0.0772 & b &
                                      \\
                                      2 & 47 & 0.6832 & 0.7247 & 0.6965 & 0.0572 & a &
                                      17,20 & 0.7547 & 0.6309 & 0.7237 & 0.1641 & b \\
                                      3 & 20 & 0.7549 & 0.7251 & 0.7463 & 0.0395 & b &
                                      \\
                                      4 & 12 & 0.7550 & 0.7599 & 0.7566 & 0.0065 & a &
                                      13,47 & 0.7551 & 0.7599 & 0.7566 & 0.0064 & a \\
                                      5 & 26 & 0.7932 & 0.7601 & 0.7839 & 0.0417 & b &
                                      \\
                                      6 & 13 & 0.7946 & 0.7831 & 0.7913 & 0.0144 & b &
                                      18,26 & 0.8302 & 0.7614 & 0.8114 & 0.0829 & b \\
                                      7 & 15 & 0.7946 & 0.7992 & 0.7962 & 0.0057 & a &
                                      \\
                                      8 & 18 & 0.8303 & 0.8000 & 0.8217 & 0.0366 & b &
                                      12,15 & 0.8303 & 0.8000 & 0.8217 & 0.0366 & b \\
                                      9 & 49 & 0.8336 & 0.8181 &  0.8292 & 0.0186 & b & \\
                                      10 & \textbf{38} & 0.8337 & 0.8278 & 0.8322 & 0.0071 & b & \textbf{7},49 & 0.8336 & 0.8312 & 0.8333 & 0.0029 & b \\
                                      \hline \hline
                                    \end{tabular}
                                    \label{table:multiple_vars}
                                  \end{table}

                                  \noindent One strategy to decrease the complexity of the problem but still allow for more than one feature to be acquired at a time is to first acquire the ``best" feature and then acquire the feature least correlated with the ``best" feature. The ``best" feature is the feature which increases the AUC of the disadvantaged group the most. We use the unconditional correlation to determine which second auxiliary feature should be acquired given the first. Once the first feature is determined, there are only $n-1$ combinations of two features.

                                  Table \ref{table:best_leastcorr_vars} shows the AUCs and bias of selecting two features simultaneously versus sequentially according to the strategy proposed above (i.e., best feature $v$ and least correlated with $v$ feature). Selecting two features simultaneously results in higher AUCs. %

                                  \begin{table}[htbp]
                                    \footnotesize
                                    \centering
                                    \caption{Selecting Two Features Simultaneously vs. Sequentially}
                                    \begin{tabular}{c|ccccc|ccccc}
                                      \hline \hline
                                      & \multicolumn{5}{c}{Simultaneous} & \multicolumn{5}{c}{Sequential} \\
                                      Round & Features & $\AUC_a$ & $\AUC_b$ & $\AUC_{\rm All}$ & Bias & Features & $\AUC_a$ & $\AUC_b$ & $\AUC_{\rm All}$ & Bias\\
                                      \hline
                                      0 & & 0.5057 & 0.5810 & 0.5332 & 0.1296 &
                                      & 0.5057 & 0.5810 & 0.5332 & 0.1296 \\
                                      1 & 17,20 & 0.7547 & 0.6309 & 0.7237 & 0.1641 &
                                      17,37 & 0.6827 & 0.6317 & 0.6685 & 0.0747 \\
                                      2 & 13,47 & 0.7551 & 0.7599 & 0.7566 & 0.0064 &
                                      47,10 & 0.7034 & 0.7256 & 0.7104 & 0.0305 \\
                                      3 & 18,26 & 0.8302 & 0.7614 & 0.8114 & 0.0829 &
                                      20,24 & 0.7717 & 0.7267 & 0.7591 & 0.0584 \\
                                      4 & 12,15 & 0.8303 & 0.8000 & 0.8217 & 0.0366 &
                                      12,32 & 0.7721 & 0.7630 & 0.7697 & 0.0119 \\
                                      5 & 7,49 & 0.8336 & 0.8312 & 0.8333 & 0.0029 &
                                      13,8 & 0.7723 & 0.7865 & 0.7769 & 0.0180 \\
                                      \hline \hline
                                    \end{tabular}
                                    \label{table:best_leastcorr_vars}
                                  \end{table}

                                  \subsection{Data Vendor Features}\label{appendix:data_vendor}

                                  We acquire auxiliary features to augment the COMPAS dataset from Aspire North\footnote{\url{https://www.aspire-north.com/}}, a data vendor. Table \ref{table:aspire_north_features} lists some of the features that the firm offers. The firm sells a core bundle of features for \$80/1,000 individuals. Certain variables, like ethnicity, net worth, spending, cost more and purchasing all available variables exceeds \$15,000/1,000 individuals.

                                  To obtain data from the vendor, we must provide identifying information on the customers, which can include the following types of data to establish identity: name, date of birth, email address, and home address.

                                  \begin{table}[htbp]
                                    \small
                                    \centering
                                    \caption{Aspire North Feature Examples}
                                    \vspace{1mm}
                                    \begin{tabular}{lp{5.5cm}rr}
                                      \hline \hline
                                      Feature Groups & Feature Examples & Cost/1,000 individuals & \# of Features  \\
                                      & & & Available \\
                                      \hline
                                      Demographics & Age, gender, education,  marital status & Core bundle & 26 \\
                                      Housing & House type, renter, homeowner & Core bundle & 14 \\
                                      Material ownership & Vehicle ownership, computer ownership & Core bundle & 23 \\
                                      Interests & Interest in crafts, gourmet cooking & Core bundle & 55 \\
                                      Consumer Financial Insights & Estimate of deposit, balances & \$40 per feature & 5 \\
                                      Consumer Spend & Prediction of annual, spend on dining & \$50 per feature & 10\\
                                      Ethnic Insights & Prediction of language preference  & \$16 per feature & 6\\
                                      Financial Personalities & Prediction of preference for rewards credit cards & \$400 per feature & 37 \\
                                      Discretionary Spend & Prediction of household spend on education  & \$50 per feature & 15\\
                                      \hline \hline
                                    \end{tabular}
                                    \label{table:aspire_north_features}
                                  \end{table}

                                  \newpage
                                  \subsection{Using Only Class-Conditional Means and Variances}\label{appendix:zero_correlation}

                                  The \fairAUC\ \algo\ requires data sharing between the firm and data vendor because of the class-conditional correlations between the score (which the firm manager has access to) and the auxiliary features (which the data vendor has access to).
                                  One strategy to simplify the \algo\ even further is to assume independence between the auxiliary features and the firm's data.
                                  While this assumption is likely to be incorrect, we evaluate how well \fairAUC\ performs in such a case.
                                  If we are willing to make such an assumption, then the data vendor can just provide means and variances corresponding to specific individuals chosen by the firm, without transferring the score or any other data about the individuals except for class.

                                  \smallskip

                                  Table \ref{table:zero_corr} shows the differences in AUC between the \algo\ that includes the class-conditional correlations and that which assumes the correlations to be zero. Specifically, we subtract the AUC with zero correlations from the AUC with correlations and do the same for the bias values. We highlight in bold which features are acquired when the correlations are assumed to be zero but are not acquired when the correlations are accounted for. Only two features fall into this category. However, incorporating the class-conditional correlations generally results in higher overall AUC and lower bias, as we might expect, which illustrates the tradeoff. Therefore, broadly, it seems like ignoring the class-conditional correlations generally results in the acquisition of the same features but in perhaps a less efficient order.

                                  \begin{table}[htbp]
                                    \small
                                    \centering
                                    \caption{Effect of Ignoring Conditional Correlation between Acquired Data and Auxiliary Features}
                                    \begin{tabular}{cccrrrr}
                                      \hline \hline
                                      & Feature with & Feature without & Difference & Difference & Difference & Difference \\
                                      Round & Correlations & Correlations & in $\AUC_a$ & in $\AUC_b$ & in $\AUC_{\rm All}$ & in Bias \\
                                      \hline
                                      1 &  17 & 17 & 0 & 0 & 0 & 0 \\
                                      2 &  47 & 47 & 0 & 0 & 0 & 0 \\
                                      3 &  20 & 18 & 0.0081 & $-$0.0011 & 0.0055 & 0.0120 \\
                                      4 &  12 & 13 & 0.0079 & $-$0.0005 & 0.0054 & $-$0.0111 \\
                                      5 &  26 & 26 & $-$0.0017 & $-$0.0008 & $-$0.0012 & $-$0.0010 \\
                                      6 &  13 & 38 & $-$0.0005 & 0.0114 & 0.0031 & $-$0.0149 \\
                                      7 &  15 & \textbf{29} & $-$0.0006 & 0.0185 & 0.0053 & $-$0.0125 \\
                                      8 &  18 & 49 & 0.0313 & $-$0.0005 & 0.0220 & 0.0347\\
                                      9 &  49 & 20 & $-$0.0001 & 0.0174 & 0.0047 & $-$0.0209 \\
                                      10 & 38 & \textbf{28} & 0.0006 & 0.0196 & 0.0058 & $-$0.0228 \\
                                      \hline \hline
                                      \multicolumn{7}{l}{\textit{Note}: Difference $=$ (with correlation) - (without correlation)}
                                    \end{tabular}
                                    \label{table:zero_corr}
                                  \end{table}

                                  \nocite{bao2021s}
                                  \bibliographystyle{plainnat}
                                  \bibliography{reference-arxiv-v2.bib}

                                  \newpage

                                  \appendix

                                  \section{Measures of Fairness Used in the Literature}\label{Appendix A}

                                  Table \ref{table:fairness_definitions} provides advantages and disadvantages to several fairness metrics that have been suggested in the literature. Let $Y\in\{0,1\}$ represent the true outcome, $\hat{Y}\in\{0,1\}$ the predicted outcome, $A$ the protected attribute, $X$ the non-protected attributes, and $C$ the classifier.

                                  \label{appendix:fairness_measures}
                                  \begin{table}[htbp]
                                    \small
                                    \centering
                                    \caption{Measures of Fairness in the Literature}
                                    \vspace{3pt}
                                    \begin{tabular}{p{2.8cm}p{13cm}}
                                      \hline
                                      Measure & Definition, Advantages, \& Disadvantages \\
                                      \hline
                                      \\
                                      Unawareness & $C = C(X)$ \newline
                                      Advantage: Addresses disparate treatment and complies with existing laws (e.g., Civil Rights Act of 1964) by not using protected attribute as an explicit variable. \newline
                                      Disadvantage: If $X$ and $A$ are correlated then the protected attribute is still incorporated into the classifier. \\
                                      \\
                                      Statistical Parity	& $\Pr[\hat{Y}=1|A=i]=\Pr[\hat{Y}=1|A=j]$ for all groups $i$ and $j$\newline
                                      Advantage: Addresses disparate impact and is the foundation for some laws (e.g., four-fifths rule). \newline
                                      Disadvantages: Can be achieved simply by selecting $x$\% from all groups regardless of justification, potentially resulting in reverse-discrimination. If $Y$ and $A$ are correlated, the ideal predictor $\hat{Y}=Y$ cannot be obtained. \\
                                      \\
                                      Predictive Rate Parity	& $\Pr[Y=1|A=i,\hat{Y}=1]=\Pr[Y=1|A=j,\hat{Y}=1]$ for all groups $i$ and $j$ and $\Pr[Y=0|A=i,\hat{Y}=0]=\Pr[Y=0|A=j,\hat{Y}=0]$ for all groups $i$ and $j$ \newline
                                      Advantage: Optimality-compatible (i.e., allows $\hat{Y}=Y$), aligning fairness with accuracy, and avoids reverse-discrimination. \newline
                                      Disadvantage: May not close the gap between groups over time if $Y$ and $A$ are correlated. \\
                                      \\
                                      Equalized Odds	& $\Pr[\hat{Y}=1|A=i,Y=1]=\Pr[\hat{Y}=1|A=j,Y=1]$ for all groups $i$ and $j$ and
                                      $\Pr[\hat{Y}=1|A=i,Y=0]=\Pr[\hat{Y}=1|A=j,Y=0]$ for all groups $i$ and $j$ \newline
                                      Advantage: Optimality-compatible (i.e., allows $\hat{Y}=Y$) and avoids reverse-discrimination. \newline
                                      Disadvantage: May not close the gap between groups over time if $Y$ and $A$ are correlated. \\
                                      \\
                                      \hline
                                    \end{tabular}
                                    \label{table:fairness_definitions}
                                  \end{table}

                                  \newpage

                                  \section{maxAUC Algorithm}\label{appendix:maxAUC}
                                  \begin{algorithm}%
                                    \SetAlgorithmName{Procedure 2}{procedure}{}
                                    \caption{{\maxAUC} ($t$-th iteration) \label{Procedure_maxAUC}}
                                    \footnotesize
                                    \KwIn{data owned $(\bm{X_i},A_i,Y_i)_{i=1}^N$, scoring algorithm $r$, bias threshold $\epsilon$, set of acquired features $Q(t)$, data available for acquisition $(\bm{Z_i})_{i=1}^N$;}
                                    \KwOut{$Q(t+1)$;}
                                    \eIf{$\bm{A}$ cannot be used}{
                                    $\bm{S} \coloneqq r(\hat{\bm{X}},\bm{Y})$\;
                                    }{
                                    $\bm{S} \coloneqq r(\hat{\bm{X}},\bm{A},\bm{Y})$\;
                                    }
                                    \For {group $g \in \{a,b\}$}{
                                    compute $\AUC_g(S)$ (Definition \ref{AUC_def}) \;
                                    }
                                    $\mathrm{Bias} \coloneqq 1 - \frac{\min_g(\mathrm{AUC}_g(S))}{\max_g(\mathrm{AUC}_g(S))}$ (Definition \ref{bias_def}) \;
                                    \eIf{$\mathrm{Bias}>\epsilon$}{
                                    \For {feature $\bm{Z}^j \in \hat{\bm{Z}},j \notin Q(t)$}{
                                    \eIf{$\bm{A}$ cannot be used}{
                                    for feature $\bm{Z}^j$ and score $\bm{S}$, obtain class-conditional means, $\bm{\mu_{0}},\bm{\mu_{1}}$, and covariance matrices, $\bm{\Sigma_{0}},\bm{\Sigma_{1}}$ (
                                    Overall Summary Statistics Subroutine)\;
                                    $h(\bm{S},\bm{Z}^j) \coloneqq \Phi \left(\sqrt{(\bm{\mu_{1}}-\bm{\mu_{0}})^\top (\bm{\Sigma_0}+\bm{\Sigma_1})^{-1}(\bm{\mu_{1}}-\bm{\mu_{0}})} \right)$\;
                                    }{
                                    \For {$g \in \{a,b\}$}{
                                    for feature $\bm{Z}^j$, group $g$, and score $\bm{S}$, obtain class-conditional means, $\bm{\mu_{0}},\bm{\mu_{1}}$, and covariance matrices, $\bm{\Sigma_{0}},\bm{\Sigma_{1}}$ (Summary Statistics by Group Subroutine)\;
                                    $\phi_g$ represents the proportion of individuals from group $g$ \;
                                    $\omega_g \coloneqq (\bm{\mu_{1}}-\bm{\mu_{0}})^\top (\bm{\Sigma_{0}}+\bm{\Sigma_{1}})^{-1}(\bm{\mu_{1}}-\bm{\mu_{0}})$ \;
                                    }
                                    $h(\bm{S},\bm{Z}^j) \coloneqq \phi_a \Phi( \sqrt{\omega_a}) + \phi_b \Phi (\sqrt{\omega_b})$\;
                                    }
                                    }
                                    $j^\star \coloneqq \arg\max_j h(\bm{S},\bm{Z}^j)$\;
                                    acquire feature $\bm{Z}^{j^\star}$\;
                                    return $Q(t+1) \coloneqq Q(t) \cup \{j^\star\}$\;
                                    }{
                                    no intervention;
                                    }
                                  \end{algorithm}

                                  \begin{algorithm}%
                                    \SetAlgorithmName{Subroutine}{subroutine}{list of subroutines}
                                    \caption{{Overall Summary Statistics} \label{Subroutine_maxAUC}}
                                    \footnotesize
                                    \KwIn{feature available for acquisition $\bm{Z}$, existing score $\bm{S}$;}
                                    \KwOut{class-conditional mean vectors $\bm{\mu_{0}},\bm{\mu_{1}}$, class-conditional covariance matrices $\bm{\Sigma_{0}},\bm{\Sigma_{1}}$;}
                                    \For {class $y \in \{0,1\}$}{
                                    $n \coloneqq n_{Y=y}$ \;
                                    $\bm{\mu_y} \coloneqq
                                    \begin{bmatrix}
                                      \Bar{S}_y\\
                                      \Bar{Z}_y
                                      \end{bmatrix} =
                                      \begin{bmatrix}
                                        \frac{1}{n}\sum_{i:
                                        \begin{smallmatrix}

                                          Y_i=y
                                          \end{smallmatrix}}S_i \\
                                          \frac{1}{n}\sum_{i:
                                          \begin{smallmatrix}
                                            Y_i=y
                                            \end{smallmatrix}}Z_i
                                            \end{bmatrix}$ ;
                                            $\bm{\Sigma_y} \coloneqq
                                            \begin{bmatrix}
                                              \sigma_{S,y}^2 & \rho_y \sigma_{S,y} \sigma_{Z,y} \\
                                              \rho_y \sigma_{S,y} \sigma_{Z,y} & \sigma_{Z,y}^2
                                              \end{bmatrix}$ \\
                                              \noindent where
                                              $\sigma_{S,y}^2=
                                              \frac{1}{n-1} \sum_{i:
                                              \begin{smallmatrix}
                                                Y_i=y
                                                \end{smallmatrix}}(S_i-\Bar{S}_y)^2$,
                                                $\sigma_{Z,y}^2=
                                                \frac{1}{n-1}  \sum_{i:
                                                \begin{smallmatrix}
                                                  Y_i=y
                                                  \end{smallmatrix}}(Z_i-\Bar{Z}_y)^2$,
                                                  and $\rho_y=
                                                  \frac{1}{(n-1)\sigma_{s,y}\sigma_{Z,y}} \sum_{i:
                                                  \begin{smallmatrix}
                                                    Y_i=y
                                                    \end{smallmatrix}}(S_i-\Bar{S}_y)(Z_i-\Bar{Z}_y)$\;
                                                    }
                                                    return $\bm{\mu_{0}},\bm{\mu_{1}},\bm{\Sigma_{0}},\bm{\Sigma_{1}}$
                                                  \end{algorithm}

                                                  \end{document}